\documentclass[sigplan, nonacm, 10pt]{acmart}
\usepackage{amsmath}
\usepackage{layouts}

\setcopyright{none}

\usepackage{pgf}
\usepackage{pgfplots}
\pgfplotsset{compat=1.9}
\usepackage{multirow}

\usepackage[utf8]{inputenc} 
\usepackage[T1]{fontenc}    
\usepackage{hyperref}       
\usepackage{url}            
\usepackage{booktabs}       
\usepackage{amsfonts}       
\usepackage{nicefrac}       
\usepackage{microtype}      
\usepackage{xcolor}         
\usepackage{wrapfig}

\usepackage[ruled,vlined]{algorithm2e}

\usepackage{amsmath}

\usepackage{amssymb}
\usepackage{mathtools}
\usepackage{amsthm}
\usepackage{mathrsfs}
\usepackage{amsfonts}
\usepackage[utf8]{inputenc}

\usepackage{amsmath}
\usepackage{amssymb}
\usepackage{mathtools}
\usepackage{amsthm}
\usepackage{mathrsfs}
\usepackage{amsfonts}
\usepackage[utf8]{inputenc}
\usepackage{hyperref}
\usepackage{xspace}
\usepackage{booktabs}
\usepackage{multirow}
\usepackage{soul}
\usepackage[export]{adjustbox}
\usepackage{enumitem}

\def\M{\mathcal{M}}
\def\X{\mathcal{X}}
\def\Y{\mathcal{Y}}
\def\S{\mathcal{S}}

\def\D{\mathcal{D}}

\def\U{\mathcal{U}}

\def\R{\mathbb{R}}

\def\de{\overset{\Delta}{=}}
\def\R{\mathbb{R}}

\SetKwInput{Input}{Input}
\SetKwProg{kwSystem}{System executes}{:}{}
\SetKwProg{kwClient}{ClientUpdate}{:}{}
\SetKwProg{ServerUpdate}{ServerUpdate}{:}{}
\SetKwProg{ServerDistribute}{ServerDistribute}{:}{}
\SetKwProg{kwServerAggregation}{ServerAggregation}{:}{}
\SetKwProg{kwDynaComm}{DynaComm}{:}{}

\newcommand{\mathdefault}[1][]{}

\theoremstyle{definition}
\newtheorem{remark}{Remark}
\newtheorem{theorem}{Theorem}

\graphicspath{{graphs/pgf/}}

\AtBeginDocument{%
  \providecommand\BibTeX{{%
    \normalfont B\kern-0.5em{\scshape i\kern-0.25em b}\kern-0.8em\TeX}}}

\begin{document}

\author{Qi Le}
\affiliation{%
  \institution{University of Minnesota-Twin Cities}
  \city{Minneapolis,  MN 55455}
  \country{USA}}
\email{le000288@umn.edu}

\author{Enmao Diao}
\affiliation{%
  \institution{Duke University}
  \city{Durham, NC 27705}
  \country{USA}}
\email{enmao.diao@duke.edu}

\author{Xinran Wang}
\affiliation{%
  \institution{University of Minnesota-Twin Cities}
  \city{Minneapolis,  MN 55455}
  \country{USA}}
\email{wang8740@umn.edu}

\author{Vahid Tarokh}
\affiliation{%
  \institution{Duke University}
  \city{Durham, NC 27705}
  \country{USA}}
\email{vahid.tarokh@duke.edu}

\author{Jie Ding}
\affiliation{%
  \institution{University of Minnesota-Twin Cities}
  \city{Minneapolis,  MN 55455}
  \country{USA}}
\email{dingj@umn.edu}

\author{Ali Anwar}
\affiliation{%
  \institution{University of Minnesota-Twin Cities}
  \city{Minneapolis,  MN 55455}
  \country{USA}}
\email{aanwar@umn.edu}

\title{DynamicFL: Federated Learning with Dynamic Communication Resource Allocation}
\setcopyright{none}

\begin{abstract}

Federated Learning (FL) is a collaborative machine learning framework that allows multiple users to train models utilizing their local data in a distributed manner. However, considerable statistical heterogeneity in local data across devices often leads to suboptimal model performance compared with independently and identically distributed (IID) data scenarios. In this paper, we introduce \textit{DynamicFL}, a new FL framework that investigates the trade-offs between global model performance and communication costs for two widely adopted FL methods: Federated Stochastic Gradient Descent (FedSGD) and Federated Averaging (FedAvg). Our approach allocates diverse communication resources to clients based on their data statistical heterogeneity, considering communication resource constraints, and attains substantial performance enhancements compared to uniform communication resource allocation. Notably, our method bridges the gap between FedSGD and FedAvg, providing a flexible framework leveraging communication heterogeneity to address statistical heterogeneity in FL. 
Through extensive experiments, we demonstrate that DynamicFL surpasses current state-of-the-art methods with up to a 10\% increase in model accuracy, demonstrating its adaptability and effectiveness in tackling data statistical heterogeneity challenges. Our code is available \href{https://github.com/Qi-Le1/DynamicFL-Federated-Learning-with-Dynamic-Communication-Resource-Allocation}{here}.
\end{abstract}

\setcopyright{none}
\maketitle
\setcopyright{none}
\section{Introduction}

\paragraph{\textbf{Motivation}}
Federated Learning (FL)~\cite{konevcny2016federated, mcmahan2017communication} is a collaborative distributed machine learning approach that facilitates model training on data distributed across multiple devices or locations~\cite{kaissis2020secure, bonawitz2019towards}. However, data on each device often exhibit high levels of statistical heterogeneity~\cite{mao2017survey, kairouz2021advances}, failing to represent the overall population accurately. This statistical heterogeneity in FL can lead to a federated model that poorly generalizes to the overall population. 

Federated Averaging (FedAvg), the widely-known method, performs multiple gradient updates before averaging, which can result in client drift issues~\cite{karimireddy2020scaffold, hsu2019measuring, khaled2020tighter}. Updating the local model using only local data may result in solutions optimized for local data, but misaligned with global optimal solutions in statistically heterogeneous FL scenarios~\cite{li2019convergence, malinovskiy2020local}. This misalignment can cause a considerable decline in overall model performance. Alternatively, Federated Stochastic Gradient Descent (FedSGD)\cite{mcmahan2017communication, li2019fair} addresses performance degradation in statistically heterogeneous situations by aggregating a single gradient update from selected clients per aggregation. Nonetheless, FedSGD demands a substantially higher number of client-server communications for convergence compared to FedAvg, trading off global model performance for increased communication costs\cite{mcmahan2017communication}. 
\begin{figure*}[htbp]
\centering
{\includegraphics[width=0.75\linewidth]{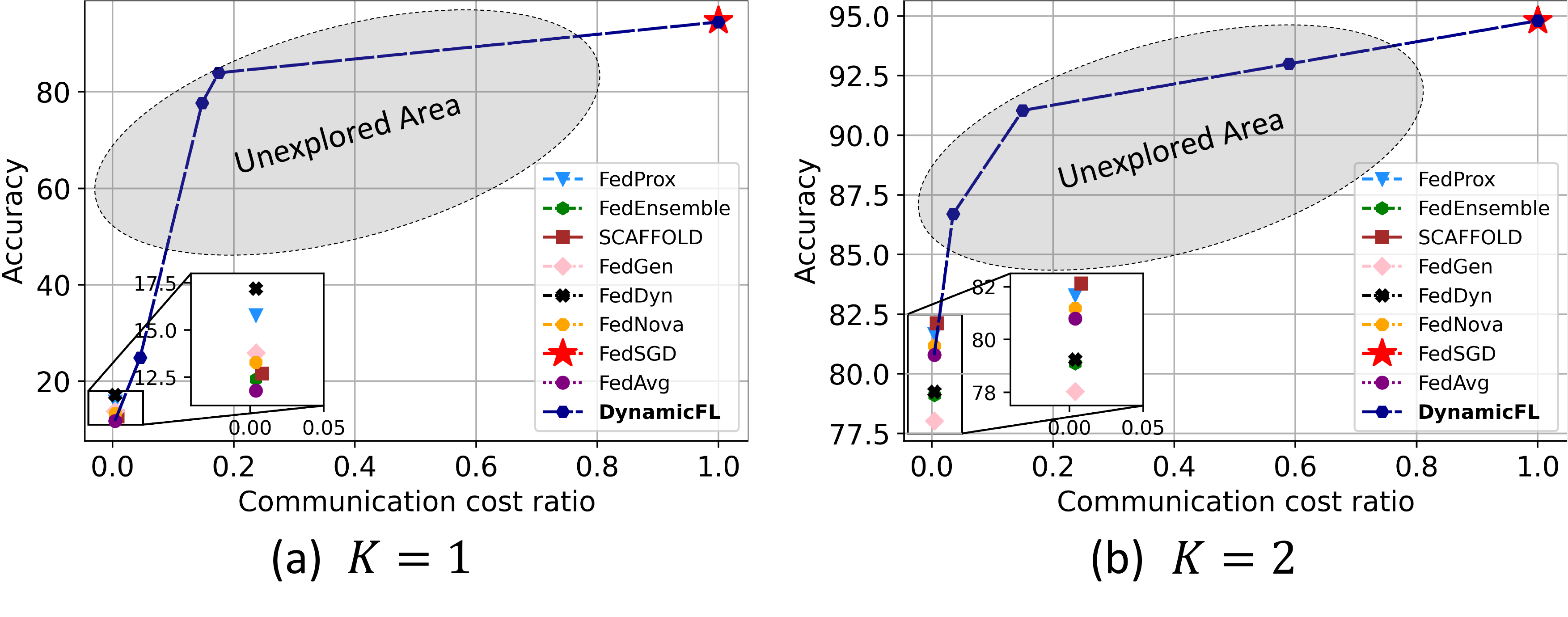}}
\vspace{-0.5cm}
\caption{Experiment with ResNet-18 on CIFAR-10 in a statistically heterogeneous scenario: (a) each client has one class label data, (b) each client has two classes label data. A significant, unexplored performance-communication cost gap exists between FedSGD and FedAvg.}
\label{fig:motivation}
\vspace{-0.5cm}
\end{figure*}

\paragraph{\textbf{Limitation of state-of-art approaches}}
Numerous advanced FedAvg-based methods, such as SCAFFOLD~\cite{karimireddy2020scaffold}, FedGen~\cite{zhuzhuangdi2021data}, and FedDyn~\cite{acar2021federated}, have been proposed to address the statistical heterogeneity in FL. While these methods claim to deliver high performance in statistically heterogeneous settings, their performance often remains relatively close to that of FedAvg. In some instances, they exhibit only a marginal improvement over FedAvg, and in other cases, their performance might be even worse~\cite{li2022federated}. This could be attributed to various factors, such as dependency on an accurate global model, the need for fine-tuning different hyper-parameters for specific scenarios, and computation constraints~\cite{diao2020heterofl,diao2023pruning}. As shown in Figure~\ref{fig:motivation}, despite these advancements, a significant performance gap between FedSGD and FedAvg global models still exists in statistically heterogeneous settings. Furthermore, the per-round communication costs of these methods are generally fixed by their respective algorithms. However, methods such as FedGen and SCAFFOLD require additional resources for generative models and gradient transmission, respectively.

\paragraph{\textbf{Key insights and contributions}}
The significant disparities in performance and communication costs between FedSGD and FedAvg motivate us to examine the trade-offs between global model performance and communication costs in statistically heterogeneous scenarios. In practical settings, communication resources for clients and servers can differ substantially, resulting in \textit{Communication Heterogeneity} of FL~\cite{li2020federated}. We define \textit{Communication Heterogeneity} as the variability in communication resources across different devices and networks. In real-world applications, it is natural for some clients to have a larger communication budget but also exhibit extreme statistical heterogeneity. We define the communication budget as a equal weighted average of the user's perceived communication budget and the current bandwidth. We leave the weights as tuneable parameters users can decide the weights based on whether they want the communication budget to reflect available network resources or their own budget priority. Effectively leveraging the diverse communication resources of clients and servers to address statistical heterogeneity is essential for successfully deploying FL in real-world situations~\cite{bonawitz2019towards, kairouz2021advances, ding2022federated}. We primarily focus on two scenarios~\cite{lim2020federated, dinh2020federated, brik2020federated, wang2022uav, ye2020federated, xu2021bandwidth}: client communication heterogeneity and server communication heterogeneity. The first scenario concerns disparate communication resources among clients; for instance, smartphones and tablets generally have substantial communication resources owing to high-speed internet access and advanced features, whereas wearable devices might be more constrained. The second scenario, server communication heterogeneity, arises when the server, such as an edge server or data relay station (e.g., Unmanned Aerial Vehicles functioning as mobile base stations), can only allocate minimal resources to FL due to bandwidth and network reliability constraints~\cite{rodrigues2010survey,rahman2010survey, wireless,diao2019restricted,diao2020drasic} or challenges in handling dynamically changing data volumes~\cite{han2017slants,huang2017distributed,xiang2019estimation}.
\vspace{-0.3cm}
\paragraph{\textbf{Experimental methodology and artifact availability}}
In this work, we present a novel FL framework, \textit{DynamicFL}, which bridges the gap between FedSGD and FedAvg by efficiently and adaptively utilizing communication resources according to clients’ data statistical heterogeneity and resource availability, ultimately improving performance in statistically heterogeneous scenarios. This flexibility allows for a more robust method that accommodates diverse communication resources and adapts to the dynamic nature of federated environments. However, several questions arise, such as: \textit{How can we determine the trade-off relationship between FedSGD and FedAvg, and could this relationship involve a non-linear association?} Furthermore, \textit{how can we develop strategies to utilize varying communication resources across devices effectively?} Our work addresses these challenges, summarized as follows:
\begin{itemize}
\item We propose \textit{DynamicFL}, a novel FL framework that flexibly bridges FedSGD and FedAvg, adaptively accommodating varying client-server communication capabilities and data statistical heterogeneity. This framework significantly enhances global model performance by strategically allocating diverse communication resources to clients based on their data statistical heterogeneity ranking, subject to their respective communication resource budgets.

\item We introduce an approach called \textit{Dynamic Communication Optimizer} for efficiently ranking data statistical heterogeneity and utilizing varying communication resources. DynamicFL boosts client collaboration with diverse communication resources, fostering a more resilient FL system.

\item Through comprehensive experiments, we demonstrate the effectiveness of the DynamicFL framework in various statistically heterogeneous scenarios. 
DynamicFL surpasses current state-of-the-art methods with up to a 10\% increase in accuracy, demonstrating its adaptability and effectiveness in tackling data statistical heterogeneity challenges.
Additionally, the experiments highlight the non-linear trade-off relationship between FedSGD and FedAvg, emphasizing DynamicFL's potential to enhance performance across diverse federated environments.
\end{itemize}

\subsubsection*{\textbf{ Limitations}}
In this paper, we specifically study the method under a fully synchronized scenario, allowing us to \textit{address statistical heterogeneity by utilizing communication heterogeneity} and explore the trade-offs between global model performance and communication costs when statistical heterogeneity is large. However, imbalances in synchronization can lead to certain challenges such as delayed model updates and skewed data representation in the learning process. We plan to tackle these challenges in the future work.

\section{Related works}
Federated Learning (FL) is a distributed machine learning paradigm enabling large-scale, collaborative model training across multiple clients. The primary challenges in FL include statistical heterogeneity and system heterogeneity~\cite{li2020federated,zhuhangyu2021federated, ding2022federated}. Strategies have been proposed to enhance global model performance under these scenarios. To address statistical heterogeneity, regularization of the local model is often employed. In contrast, Asynchronous FL addresses the straggler problem inherent to system heterogeneity by permitting more frequent aggregation from a subset of clients, which, in turn, positively influences global model performance. Other strategies include local fine-tuning~\cite{yu2021adaptive, chen2022self,le2022personalized}, personalized model layers~\cite{arivazhagan2019federated, ma2022layer,collins2022perfedsi}, neural architecture search~\cite{Fednas,mushtaq2021spider}, multi-task learning~\cite{smith2017federated, corinzia2019variational}, client clustering~\cite{chai2020tifl,ye2022meta}, and among others.

\paragraph{\textbf{Regularized FL}}Local model regularization is performed in both \textit{parameter space} and \textit{representation space}. In \textit{parameter space}, the global model serves to constrain the local model. FedProx~\cite{li2020federated} incorporates a proximal term based on the server model in local training to constrain the local model. SCAFFOLD~\cite{karimireddy2020scaffold} adopts a variance reduction approach to rectify locally drifted updates, and FedDyn~\cite{acar2021federated} incorporates linear and quadratic penalty terms to dynamically adjust local objectives, ensuring alignment between local and global objectives. In \textit{representation space}, Data-free knowledge distillation has emerged as a promising solution. FedGen~\cite{zhuzhuangdi2021data} develops a lightweight server-side generator leveraging clients' mapping from output to latent space, using this generator to refine local training in the hidden \textit{representation space}. FedFTG~\cite{zhang2022fine} leverages generative models to explore the input \textit{representation spaces} of local models, utilizing synthesized data to fine-tune the global model.

\paragraph{\textbf{Compression techniques}}Some works have also explored compression-based techniques such as Quantization~\cite{reisizadeh2020fedpaq} and Pruning~\cite{diao2023pruning} in FL to reduce the communication overhead. These studies are orthogonal to our work and can be incorporated into DynamicFL, however, these model compression techniques have an adverse effect on the training accuracy of the models~\cite{FLOAT} while adaptive communication frequency in DynamicFL does not directly reduce accuracy performance.

\begin{figure*}[htbp]
\centering
{\includegraphics[width=0.97\linewidth]{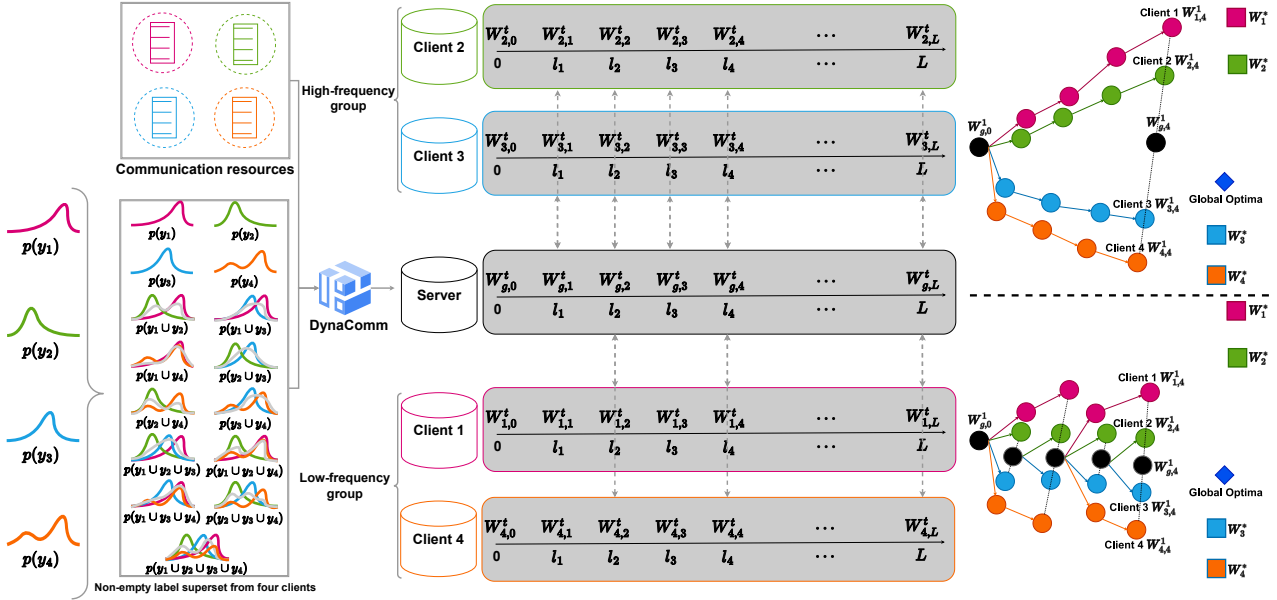}}
\vspace{-0.3cm}
\caption{Overall Design: DynamicFL's operation over one global communication round. DynaComm identifies clients $2$ and $3$ for the high-frequency group based on their statistical heterogeneity and communication resource budgets, emulating FedSGD to guide optimizations towards global optima. The figure's right side shows the gradient update correction process for all four clients over four local steps, where $W_{1, 0}^t$ signifies client 1's model at round $t$ after $0$ local updates.}
\label{fig:schema}
\end{figure*}

\paragraph{\textbf{Asynchronous FL}} Asynchronous FL approaches have been developed to tackle straggler issues in FL. For instance, FedAsync~\cite{xie2019asynchronous} enables asynchronous client updates, while FedAT~\cite{chai2021fedat} utilizes asynchronous tiers to boost convergence speed and accuracy. Further, HFL~\cite{li2021stragglers} attempts to accommodate delayed gradient updates. FedBuff~\cite{nguyen2022federated} employs a secure buffer for local updates, updating the global model once the buffer is full. These approaches generally result in fast clients contributing more frequently to global updates than slower devices.

However, existing advanced Regularized FL techniques aim to solve statistical heterogeneity by using uniform communication resource allocation, disregarding the potential communication heterogeneity. Conversely, Asynchronous FL methods, while attempting to utilize communication heterogeneity, mainly focus on the straggler issue and neglect statistical heterogeneity. To our best knowledge, no existing work has extensively investigated the relationship between handling statistical heterogeneity by leveraging naturally existing communication heterogeneity. Our work steps in to fill this gap. We significantly improve global model performance by bridging the gap between FedSGD and FedAvg, inspired by the Asynchronous FL's strategy of allowing more frequent server-client aggregation for certain clients. We propose the first method that effectively leverages communication heterogeneity to address statistical heterogeneity, thereby optimizing global model performance and boosting real-world applicability.

\section{Design}
\subsection{Problem setting}

\paragraph{Local data statistical heterogeneity.} We first introduce the problem setting and notation. A single batch is a mini-batch with a batch size of $B$, and one gradient update refers to a mini-gradient update. We define a set of clients $\M$ consisting of $M$ total clients, with each local model represented by a model parameter $W_{m}$ for $m\in \M$. The global dataset $\D$, consisting of feature component $\X$ and label component $\Y$, is distributed across the clients in $\M$. For a generic client $m$, its local data is represented by $\D_m$, which consists of $\left| \D_{m} \right|$ training instances with feature vectors $x_m$ and corresponding labels $y_{m}$. Here, we let $| \cdot |$ indicate the cardinality of a set. Therefore, $\D = \bigcup_{m \in \M} \D_{m}$ and $\Y = \bigcup_{m \in \M} y_{m}$. At communication round $t$, a random subset of clients, $\M^t$, is selected from $\M$. We define $\S^{t}$ as the non-empty power set of $\M^t$ at communication round $t$, having a cardinality of $2^{\left| \M^t \right| - 1}$. For each element $z$ in $S^{t}$, we refer to the joint labels $\bar{y}_{z} \de \bigcup_{m \in z} y_{m}$ as the combined set of labels from all clients in $z$. We use $p(y_m)$ to represent the probability distribution of labels in client $m$, and the probability distribution $p(\bar{y}_{z})$ of joint labels $\bar{y}_{z}$ is defined as:
\begin{equation}
p(\bar{y}_{z}) \de \sum_{m \in z}\frac{\left| y_{m} \right|}{\left| \bar{y}_{z} \right|}p(y_{m}) .
\label{prob_distribution_of_subset}
\end{equation}
\vspace{-0.3cm}

Although label sharing has been explored in various works in FL~\cite{luo2021fedsld}, if the clients prefer to keep labels private DynamicFL also includes other mechanisms to discern the probability distribution via similarity-based techniques such as cosine similarity used for clustering clients based on privacy-aware inferred data information~\cite{auxo}.
\paragraph{\textbf{Communication resource}} Next, we simplify the communication resource settings for a single round of analysis. Let $L$ be the number of gradient updates in a round, and $\nu_{\text{high}}$ and $\nu_{\text{low}}$ represent the high and low client-server aggregation frequencies, respectively. The communication intervals, $I_{\text{high}} = \lceil \frac{L}{\nu_{\text{high}}}\rceil$ and $I_{\text{low}} = \lceil \frac{L}{\nu_{\text{low}}}\rceil$, determine the number of local gradient updates executed by a client between server aggregations. Shorter intervals indicate more frequent communication with the server. Each of the $M$ clients is allocated a communication resource budget $[\tau_{1}, \tau_{2}, \cdots, \tau_{M}]$ per round, for each round, while the server receives a budget of $\tau_{g}$. This budget, $\tau$, represents the maximum amount of data that can be transferred in one round and as aforementioned is a weighted average of user's budget priority and network bandwidth to reflect both user's choice and available network resources. In round $t$, DynamicFL assigns a communication frequency $\nu_{m}^{t}$ to each client $m \in \M^{t}$. The communication cost $\kappa_{m}^{t}$ is then computed as twice the product of the model size (number of parameters, $|W_{m}|$) and the assigned communication frequency $\nu_{m}^{t}$, with the factor of two accounting for both upload and download operations. The server's communication cost $\kappa_{g}^{t}$ is the cumulative cost of all client communication costs in $\M^{t}$:
\vspace{-0.1cm}
\begin{equation}
\kappa_{g}^{t} \de \sum_{m \in \M^{t}} \kappa_{m}^{t} = \sum_{m \in \M^{t}} 2 \times |W_{m}| \times \nu_{m}^{t}.
\label{server_cost_of_communication_resources}
\end{equation}

\subsection{Dynamic Federated Learning}
\paragraph{\textbf{Overall design}} Dynamic Federated Learning (DynamicFL) improves global model performance by strategically allocating communication resources to clients, taking into account data statistical heterogeneity and communication resource budgets. In each communication round $t$, we partition the client set $\M^{t}$ into two groups: a high-frequency communication group with frequency $\nu_{\text{high}}$ and a low-frequency communication group with frequency $\nu_{\text{low}}$, as shown in Figure~\ref{fig:schema}. In the high-frequency group, our objective is to emulate the behavior of FedSGD. To this end, we choose a client subset $z$ from $\M^{t}$ that closely resembles the global label distribution $\Y$ while adhering to their communication resource constraints. The remaining active clients $\mathcal{M}^t \setminus z$ are assigned to the low-frequency group. Both groups perform intra-group averaging, and their resulting models are combined through inter-group averaging. This inter-group averaging occurs when the current number of local gradient updates is a multiple of communication intervals $I_{\text{high}}$ and $I_{\text{low}}$. Through this approach, as illustrated in Figure~\ref{fig:schema}, we facilitate the high-frequency group to function similarly to FedSGD, thereby correcting the direction and magnitude of divergent gradients, while the low-frequency group, mirroring FedAvg, aids in accelerating dataset iteration. As a result, DynamicFL enhances the convergence and accuracy of the FL process by effectively managing communication resources based on data statistical heterogeneity and the communication budgets of clients and the server. Theoretical analysis is provided in Appendix~\ref{sec:theory}.
\begin{remark}[Two distinct groups]
\textit{A high-frequency group, aligning with the FedSGD for gradient correction, and a low-frequency group, patterned after FedAvg for efficient dataset iteration.}
\end{remark}

\paragraph{\textbf{Bridging the gap between FedSGD and FedAvg}} We aim to develop an algorithm that bridges the gap between FedSGD and FedAvg. To ensure fairness, we establish a fixed total number of gradient updates, $L_{T}$, for both methods, as depicted in Equation~\ref{total_gradient_updates}:
\vspace{-0.1cm}
\begin{align}
L_{T} = \sum_{t=1}^{T}\sum_{m \in \M^{t}}L_{m}^{t} .
\label{total_gradient_updates}
\end{align}
Here, $L_{m}^{t}$ signifies the number of local gradient updates for client $m$ at communication round $t$, and $T$ denotes the total global communication rounds. For both FedSGD and FedAvg,  the batch size $B$ and the fraction $C$ of active clients per communication round are the same. For FedSGD, only one batch of local data is trained by the model at each communication round, and thus $L_{m}^{t}$ is invariably one. Conversely, for FedAvg, $L_{m}^{t}$ is calculated as $\frac{|\mathcal{D}_{m}| \cdot E}{B}$, with $\left|\mathcal{D}_{m}\right|$ being the size of client $m$'s data, and $E$ representing the local epoch. Generally, $\frac{|\D_{m}| \cdot E}{B} \gg 1$. Hence, with $L_{T}$ fixed, FedSGD requires a greater number of global communication rounds $T$ compared to FedAvg, complicating direct comparative analysis. To overcome this, our DynamicFL method fixes $L_{T}$ and $T$ and adjusts the number of local gradient updates based on the strategy proposed by Wang et al.\cite{wang2020tackling}. We ensure that the number of gradient updates is consistent across all clients and communication rounds, i.e., $L = L_{1} = L_{2} = \cdots = L_{M}$, represented in Figure~\ref{fig:schema}. This strategy mitigates the objective inconsistency due to varying numbers of local gradient updates across clients\cite{wang2020tackling}. Consequently, $L$ can be expressed as:
\begin{align}
L = \frac{\sum_{m=1}^{M}L_{m}}{M} = \frac{\sum_{m=1}^{M}|\D_{m}| \cdot E}{M \cdot B}.
\label{local_gradient_update_convert}
\end{align}
Under this configuration, by fixing $T$, we achieve equal total gradient updates for both FedSGD and FedAvg, effectively bridging the distinction between them. We refer to this modified version of FedSGD, which performs $L$ local gradient updates for client $m$ instead of one, as DynamicSGD. While DynamicSGD subtly diverges from traditional FedSGD, it maintains comparable performance, as shown in Figure~\ref{fig:motivation}. Furthermore, we define DynamicAvg as consisting of a group of clients mirror FedSGD's behavior, while the remaining clients adopt a behavior similar to FedAvg. It is important to note that Equation~\ref{local_gradient_update_convert} is used primarily in our experiments to ensure DynamicFL and competing methods based on FedAvg share the same total gradient updates. In practical applications, the number of gradient updates can be flexibly adjusted, providing a more adaptable approach compared to using local epochs as the iteration criteria.
\begin{remark}[Bridging DynamicSGD and FedAvg]
\textit{DynamicSGD and DynamicAvg are special cases of DynacmiFL.  When the communication interval equals one, DynamicFL becomes DynamicSGD. Conversely, when the interval equals $L$, it behaves like FedAvg. Additionally, when communication heterogeneity is present, DynamicFL evolves into DynamicAvg.}
\end{remark}

\paragraph{\textbf{Workflow}} We outline the execution flow of the DynamicFL system as depicted in Algorithm~\ref{alg:dynamicfl}. At each communication round $t$, a client subset $\mathcal{M}^{t}$ is selected from the total client set $\M$ based on a fraction $C$ (line 4). From $\M^{t}$, a further high-frequency client subset $z$ is selected using the DynaComm (lines 22-37). For every client $m$ in $\mathcal{M}^t$, a communication frequency $\nu_{m}^{t}$ and a communication interval $I_{m}$ are assigned, based on whether the client $m$ is part of $z$ (lines 7-8). The server model is subsequently distributed to these clients (line 9). Concurrently,  each client $m$ performs local gradient updates (lines 14-15). When the local gradient updates $l$ reach a multiple of the client's communication interval $I_{m}$, the client $m$ sends updated model parameters $W_{m, l}^t$ to the server, with its index added to $\U_{l}^{t}$(line 16). The server then updates the global model parameters $W_{g, l}^t$ by performing a weighted average of received models and distributes it to all clients listed in $\U_{l}^{t}$ (line 18). Upon completing $L$ gradient updates, all active clients transmit their models to the server. The server aggregates these models using a weighted method (line 20). This process is repeated across $T$ communication rounds.

\begin{algorithm}[htp]
\begin{footnotesize}
\SetAlgoLined
\DontPrintSemicolon
\Input{Data $\D$ and labels $\Y$ across $M$ clients, fraction $C$, minibatch size $B$, communication frequencies $\nu_{\text{high}}$ and $\nu_{\text{low}}$, learning rate $\eta$, total global communication rounds $T$, gradient updates $L$, global model $W_g$, client communication budgets $[\tau_{1}, \tau_{2}, \cdots, \tau_{M}]$, server communication budget $\tau_{g}$, DynaComm ensemble times $\eta_{\text{ens}}$.}

\kwSystem{}{

Initialize $W_g^0$\;\For{\textup{each communication round }$t = 1, \dots, T$}{
$\M^{t} \leftarrow \max(C \cdot M, 1)$ clients sampled from $\M$\;

Select high-frequency client subset $z$ using \textbf{DynaComm}($\M^{t}$, $[\tau_{1}, \cdots, \tau_{M}]$, $\tau_{g}$)

\For{\textup{each client }$m \in \M^{t}$\textup{ \textbf{in parallel}}}{
Communication frequency $\nu_{m}^{t} \leftarrow \nu_{\text{high}}$ if $m \in z$ else $\nu_{\text{low}}$

Communication interval $I_{m}^{t} \leftarrow \lceil \frac{L}{\nu_{m}^{t}} \rceil$

$W_{m, 0}^{t} \leftarrow W_{g}^{t-1}$
}
\For{\textup{each gradient update }$l = 1, \dots, L$}{
Server initializes client indices set  $\U_{l}^{t} \leftarrow  \varnothing$

\For{\textup{each client }$m \in \M^{t}$\textup{ \textbf{in parallel}}}{
\textup{Randomly select batch $b_{m} \in \D_{m}$}

$W_{m, l}^t \leftarrow W_{m, l-1}^t\eta \nabla \ell(W_{m, l-1}^t; b_m)$

\textbf{if} {$l \bmod I_{m}^{t} = 0$ \textup{or} $l = L$} \textbf{then} $\U_{l}^{t} \leftarrow \U_{l}^{t} \cup \{ m\}$ 

}
$W_{g, l}^{t} \leftarrow \sum_{m \in \U_{l}^{t}}\frac{\left| \D_{m}\right|}{\sum_{m \in \U_{l}^{t}}\left| \D_{m}\right|} W_{m, l}^{t}$, \quad
$W_{m, l}^{t} \leftarrow W_{g, l}^{t}$ for each client $m \in \U_{l}^{t}$\;
}
$W_{g}^{t} \leftarrow W_{g, L}^{t}$\;
}
}
\kwDynaComm{$(\M^{t}$, $[\tau_{1}, \tau_{2}, \cdots, \tau_{M}]$, $\tau_{g})$}
{
Minimum KL-divergence $KL_{\text{min}}$ $\leftarrow \infty$;~ High-frequency client subset $z$ $\leftarrow$ An empty set $\varnothing$

\For{$\eta \leftarrow 1~\textup{to}~\eta_{\textup{ens}}$}{
Shuffle $\M^{t}$; 

Initialize $(|\mathcal{M}^{t}|+1) \times (|\mathcal{M}^{t}|+1)$ matrix $O$, where each element $O[i,j]$ has two attributes: $O[i,j]$.KL with infinite KL divergence, and $O[i,j]$.clients with empty client subsets.

\For{$(i, j) \leftarrow (1, 1)$ \textup{to} $(O$\textup{.rows}, $O$\textup{.columns}$)$}{
    \textbf{if} $i < j$ \textbf{then continue;} 
    
    $O[i, j]$.KL $\leftarrow$ $O[i-1, j]$.KL;~ $O[i, j]$.clients $\leftarrow$ $O[i-1, j]$.clients

    Initialize a high-frequency client subset $z_{i, j} \leftarrow$ $O[i-1, j-1]$.clients $\cup$ $\{\M^{t}[i-1]\}$
    
    Calculate KL-divergence $d_{\textup{KL}}$ for client labels in $z$
    
    Compute server communication cost $\kappa_{g}^{t}$ (See Eq.~\ref{server_cost_of_communication_resources})
    
    \textbf{if} $\kappa_{\M^{t}[i-1]}^{t} > \tau_{\M^{t}[i-1]}$ \text{or} $\kappa_{g}^{t} > \tau_{g}$ \textbf{then continue}
    
    \textbf{if} $d_{\textup{KL}} < O[i, j]$\textup{.KL} \textup{and} $|c| = j$ \textbf{then} $O[i, j]$.KL $\leftarrow d_{\textup{KL}}$;~$O[i, j]$.clients $\leftarrow z_{i, j} $
    
    \textbf{if} $d_{\textup{KL}} < KL_{\text{min}}$ \textbf{then} $KL_{\text{min}} \leftarrow d_{\textup{KL}}$;~ $z \leftarrow z_{i, j} $\;
}
}
Return $z$ to server

\vspace{3pt}
}
\vspace{+0.1cm}
\caption{DynamicFL: Dynamic Federated Learning}
\label{alg:dynamicfl}
\end{footnotesize}
\end{algorithm}

\paragraph{\textbf{Dynamic Communication Optimizer (DynaComm)}} In each communication round $t$, DynamicFL optimally allocates communication resources to clients based on their data statistical heterogeneity, considering communication resource constraints. The first step in DynamicFL involves identifying the subset of clients from $\M^t$, denoted as $z \in \S^t$, whose sampling distribution $p(\bar{y})$ closely approximates the population distribution $p(\Y)$, thereby imitating FedSGD's behavior. It is worth noting that we are using only the sampling label distribution (i.e., the count of data points per class, easily retrievable in practice.) but not the population label distribution of each client. Addressing statistical heterogeneity often involves leveraging data distribution properties, such as FedGen~\cite{zhuzhuangdi2021data}, which utilizes clients’ label distribution, and FedFTG~\cite{zhang2022fine}, focusing on client input distribution. 

It is crucial that all clients within this subset possess sufficient communication resources to engage in high-frequency communication. To achieve this selection, we minimize the Kullback-Leibler (KL) divergence~\cite{kullback1951information, hinton2015distilling}, denoted as $D_{\text{KL}}$, between the two distributions, taking into consideration the communication resource constraints, as below:
\begin{align}
&\min D_{\text{KL}} \left[p(\bar{y_{z}})\left|\right|p(\Y)\right] \label{optimization_problem} \\
&\textrm{subject to:} \quad z \in \S^t, \nonumber \\
&\kappa_{m}^{t} \leq \tau_{m} \quad \forall m \in \M^{t}, \nonumber \\
&\kappa_{g}^{t} = \sum_{m \in \M^{t}} \kappa_{m}^{t} \leq \tau_g . \nonumber
\end{align}
A smaller KL-divergence indicates a higher similarity between the distributions. Different communication resource budgets are used to balance the trade-off between model performance and communication resources. Specifically, a larger communication resource budget emphasizes model performance, while a smaller budget aims to minimize communication resource consumption.

Addressing the optimization problem (Eq.\ref{optimization_problem}) poses a mixed-integer nonlinear programming challenge, particularly when considering communication resource budgets. To this end, we propose an ensemble dynamic programming-based method, DynaComm, shown in Algorithm~\ref{alg:dynamicfl} (lines 22-37). The algo computes the new KL-divergence value for the current client subset, and checks if adding the current client would exceed the communication resource budgets (lines 28-32). If the new client subset improves the KL-divergence and meets the communication constraints, DynaComm updates the current cell in the matrix (line 33). Throughout this process, DynaComm tracks the client subset with the smallest KL-divergence (lines 34) and returns the client subset to the server (line 37).

\begin{table}[htb]
\centering
\caption{DynaComm consistently outperforms random client selection in terms of accuracy, 
starting from a minimal selection of 0.1 portions of clients (akin to FedAvg) and extending to 
include all active clients (DynamicSGD). Two random experiments with standard error are conducted under the specified parameters: communication interval 'a-g', CIFAR10, and CNN model.}
\label{tab:dynacomm_random_ablation_l1l2}
\begin{small}
\addtolength{\tabcolsep}{-0.4em}
\begin{tabular}{@{}cccccc@{}}
\toprule
\multirow{2}{*}{Method}     & \multirow{2}{*}{Freq.} & \multicolumn{2}{c}{$K=1$}                   & \multicolumn{2}{c}{$K=2$}                   \\ \cmidrule(l){3-6} 
                            & group                                           & Random               & DynaComm             & Random               & DynaComm             \\ \midrule
DynSGD                  & 1.0                                         & \multicolumn{2}{c}{73.9 (1.0)}              & \multicolumn{2}{c}{77.2 (0.3)}              \\ \midrule
\multirow{9}{*}{DynAvg} & 0.9                                          & 69.6 (0.2)           & \textbf{72.1 (0.5)}  & 76.8 (0.5)           & \textbf{76.9 (0.3)}  \\
                            & 0.8                                          & 67.5 (0.1)           & 67.0 (0.1)           & 76.0 (0.2)           & 75.6 (0.2)           \\
                            & 0.7                                          & 64.7 (1.7)           & \textbf{68.0 (1.4)}  & 74.5 (0.4)           & \textbf{75.6 (0.2)}  \\
                            & 0.6                                          & 65.9 (1.5)           & 65.6 (0.7)           & 74.3 (0.3)           & \textbf{75.2 (0.3)}  \\
                            & 0.5                                          & 64.9 (0.2)           & \textbf{67.4 (0.8)}  & 73.7 (0.1)           & \textbf{74.2 (0.6)}  \\
                            & 0.4                                          & 63.7 (1.4)           & \textbf{64.9 (0.4)}  & 71.6 (0.8)           & \textbf{72.0 (0.1)}  \\
                            & 0.3                                          & 61.4 (0.7)           & \textbf{61.6 (0.0)}  & 69.2 (0.1)           & \textbf{70.8 (0.2)}  \\
                            & 0.2                                          & 44.4 (0.3)           & \textbf{49.8 (0.1)}  & 68.6 (0.0)           & 68.1 (0.4)           \\
                            & 0.1                                          & 20.45 (0.1) & 20.2 (0.7) & 65.5 (0.4) & 65.6 (0.5) \\ \midrule
FedAvg                      & 0.0                                          & \multicolumn{2}{c}{19.1 (1.4)}              & \multicolumn{2}{c}{66.3(0.0)}               \\ \bottomrule
\end{tabular}
\end{small}
\end{table}
		
Our DynaComm method provides a better trade-off between global model performance and communication cost than the random client selection does. In particular, our method optimizes the selection of the high-frequency group by making the joint label distribution of the high-frequency group close to the global label distribution. Because the joint label distribution of the high-frequency group is close to the global label distribution, it can rectify the gradient divergence~\cite{hsu2019measuring} caused by the low-frequency group. For instance, a client may only contain certain classes of labels, but the joint label distribution of multiple clients could result in a uniform distribution (balanced classification task). Our method aggregates the selected clients with high frequency to approximate FedSGD which does not have any gradient divergence issue. On the other hand, the low-frequency group, mirroring FedAvg’s behavior for efficient dataset iteration can indeed enhance model performance.

We conducted an ablation study to further validate DynaComm's performance against random client selection, focusing on different subset sizes. We start with FedAvg, then 0.1 portions of active clients $|\M^{t}|$ into the high-frequency group and extend to DynamicSGD. DynaComm outperforms random client selection in accuracy in most cases, as shown in Table~\ref{tab:dynacomm_random_ablation_l1l2}. For instance, under $K=1$, a 5.4\% accuracy boost is observed when the high-frequency group is 0.2, and under $K=2$, a 1.6\% accuracy boost is noted when the high-frequency group is 0.3. Notably, with 0.1 high-frequency group portion, DynamicAvg reduces to FedAvg, and thus both methods yield similar results to FedAvg.

Additionally, we observe that KL-divergence plateaus or even increases after the subset size reaches a certain threshold, suggesting a non-linear relationship between the sampling and global label distributions. Factors such as label class, client numbers, and data statistical heterogeneity contribute to this phenomenon, indicating a nonlinear relationship between FedSGD and FedAvg. 
\begin{remark}[Efficient DynaComm]
\textit{DynaComm efficiently resolves the optimization problem (Eq.~\ref{optimization_problem}), optimizing the selection of the high-frequency group, and thereby achieving superior performance compared to random client selection.}
\end{remark}

\section{Experiments}
\subsection{Experimental setup}
\paragraph{\textbf{Baselines}} We benchmark DynamicFL against FedAvg~\cite{mcmahan2017communication}, FedProx~\cite{li2020federated}, SCAFFOLD~\cite{karimireddy2020scaffold}, FedDyn~\cite{acar2021federated}, FedGen~\cite{zhuzhuangdi2021data}, FedEnsemble~\cite{zhuzhuangdi2021data}, FedNova~\cite{wang2020tackling}, utilizing DynamicSGD as the upper bound for comparison with DynamicAvg. FedFTG~\cite{zhang2022fine} is excluded due to unavailable code. All methods undergo equal total gradient updates.
\paragraph{\textbf{Datasets and models}} We test DynamicFL on CIFAR10 and CIFAR100 datasets~\cite{netzer2011reading, li2022federated} using CNN and ResNet-18~\cite{he2016deep}, and FEMNIST dataset~\cite{caldas2018leaf} exclusively using CNN due to size constraints. Our experiments encompass 100 clients with 10\% active per round. We use layer normalization~\cite{ba2016layer} to address batch normalization issues~\cite{acar2021federated}. 

We examine both balanced and unbalanced data partitions for CIFAR10 and CIFAR100. For balanced partitions, clients hold equal data volumes from $K$ classes, where $K = \{1, 2\}$. but exhibit a skewed label distribution~\cite{diao2020heterofl, diao2022semifl}. In unbalanced partitions, both data size and label distribution are skewed, determined by a Dirichlet distribution (Dir($\alpha$)), where $\alpha = \{0.1, 0.3\}$~\cite{acar2021federated, yu2022spatl, diao2022semifl}. A smaller $\alpha$ indicates higher data statistical heterogeneity. For FEMNIST, balanced partitions are unfeasible due to significant class size differences. Hence, we conduct our experiments to unbalanced partitions with $\alpha = \{0.01, 0.1, 0.3\}$.
\paragraph{\textbf{Configurations}} To evaluate DynamicFL's performance, we define seven communication interval levels, $\{a, b, c, d, e, f, g\}$, each corresponding to 1, 4, 16, 32, 64, 128, 256 gradient updates between server aggregations during each communication round. Our experiments illustrate the dynamic adjustability of these intervals. For example, the configuration ’a-g’ denotes a specific combination of communication 
intervals, with high-frequency group clients following the ’a’ interval and low-frequency group clients 
adhering to the ’g’ interval. The notation ’a*’ represents training exclusively with high-frequency 
clients, while ’a’ encompasses all active clients. To address statistical heterogeneity, we consider two configurations: \textit{Fix}, leveraging client communication heterogeneity, such as disparate resources among devices like smartphones and tablets~\cite{bonawitz2019towards}, and \textit{Dynamic}, emphasizing server, like an edge server, strategically allocates limited FL resources according to bandwidth and network reliability~\cite{rodrigues2010survey,rahman2010survey, wireless,diao2019restricted,diao2020drasic}. We test three communication resource budget levels $\beta = \{0.3, 0.6, 0.9\}$ in both configurations. For example, in \textit{Fix-0.3}, 30\% of the total clients' budgets accommodate high-frequency communication, while the remaining 70\% support low-frequency communication. In the \textit{Dynamic} configuration, active clients are dynamically assigned communication resources in each round according to DynaComm. We standardize the communication cost, $\kappa_{g}^{t}$, relative to the cost of DynamicSGD for comparison.

\begin{table}[h!]
\caption{ResNet-18 accuracy results comparing DynamicSGD, DynamicAvg (our approach) with three communication interval combinations, and baselines across various configurations, with $K=2$, on CIFAR10 and CIFAR100 datasets.}
\centering
\begin{footnotesize}
\begin{tabular}{ccccccc}
\toprule
\multicolumn{3}{c}{\multirow{2}{*}{Method}}                      & \multicolumn{2}{c}{CIFAR10}   & \multicolumn{2}{c}{CIFAR100}  \\ \cline{4-7} 
\multicolumn{3}{c}{}                                             & \textit{Fix}           & \textit{Dynamic}       & \textit{Fix}            & \textit{Dynamic}      \\ \midrule
\multicolumn{3}{c}{DynamicSGD}                                   & \multicolumn{2}{c}{94.9(0.0)} & \multicolumn{2}{c}{71.1(0.1)} \\ \midrule
\multirow{9}{*}{DynAvg} & \multirow{3}{*}{a-g} & $\beta=0.3$ & 87.5(0.0)     & 89.5(0.0)     & 34.9(0.0)      & 36.1(0.0)    \\
                            &                      & $\beta=0.6$ & 91.2(0.0)     & 92.7(0.0)     & 53.3(0.0)      & 53.1(0.0)    \\
                            &                      & $\beta=0.9$ & 92.5(0.0)     & 93.0(0.0)     & 66.6(0.0)      & 63.2(0.0)    \\ \cmidrule{3-7}
                            & \multirow{3}{*}{b-g} & $\beta=0.3$ & 85.5(0.0)     & 86.5(0.2)     & 26.9(0.2)      & 26.0(0.5)    \\
                            &                      & $\beta=0.6$ & 88.9(0.0)     & 90.5(0.0)     & 42.2(0.2)      & 44.7(0.3)    \\
                            &                      & $\beta=0.9$ & 91.0(0.0)     & 91.0(0.0)     & 51.6(0.4)      & 50.1(0.5)    \\ \cmidrule{3-7}
                            & \multirow{3}{*}{c-g} & $\beta=0.3$ & 83.6(0.3)     & 84.2(0.0)     & 24.5(0.5)      & 24.2(0.0)    \\
                            &                      & $\beta=0.6$ & 85.7(0.0)     & 86.8(0.0)     & 31.5(1.5)      & 33.7(1.0)    \\
                            &                      & $\beta=0.9$ & 87.0(0.0)     & 86.7(0.5)     & 38.3(0.4)      & 38.8(1.5)    \\ \midrule
\multicolumn{3}{c}{FedProx}                                      & \multicolumn{2}{c}{81.7(0.8)} & \multicolumn{2}{c}{21.5(1.0)} \\
\multicolumn{3}{c}{FedEnsemble}                                  & \multicolumn{2}{c}{79.1(1.2)} & \multicolumn{2}{c}{18.2(0.1)} \\
\multicolumn{3}{c}{SCAFFOLD}                                     & \multicolumn{2}{c}{82.1(0.3)} & \multicolumn{2}{c}{15.7(3.7)} \\
\multicolumn{3}{c}{FedGen}                                       & \multicolumn{2}{c}{78.0(0.0)} & \multicolumn{2}{c}{19.6(0.4)} \\
\multicolumn{3}{c}{FedDyn}                                       & \multicolumn{2}{c}{79.3(0.0)} & \multicolumn{2}{c}{15.3(0.3)} \\
\multicolumn{3}{c}{FedNova}                                      & \multicolumn{2}{c}{81.2(0.1)} & \multicolumn{2}{c}{20.5(0.5)} \\
\multicolumn{3}{c}{FedAvg}                                       & \multicolumn{2}{c}{80.8(0.1)} & \multicolumn{2}{c}{20.4(0.4)} \\ \bottomrule
\end{tabular}
\label{tab:cifar10_cifar100_k2_resnet_main}
\end{footnotesize}
\end{table}

\begin{table*}[htb]
\caption{CNN accuracy results for three communication intervals across different configurations, with $K=2$, on CIFAR10 and CIFAR100 datasets, and $Dir(0.01)$ on FEMNIST dataset.}
\centering
\begin{small}
\begin{tabular}{@{}ccccccccc@{}}
\toprule
\multicolumn{3}{c}{\multirow{2}{*}{Method}}                      & \multicolumn{2}{c}{CIFAR10}                                       & \multicolumn{2}{c}{CIFAR100}                                      & \multicolumn{2}{c}{FEMNIST}                                       \\ \cmidrule(l){4-9} 
\multicolumn{3}{c}{}                                             & \textit{Fix} & \textit{Dynamic} & \textit{Fix} & \textit{Dynamic} & \textit{Fix} & \textit{Dynamic} \\ \midrule
\multicolumn{3}{c}{DynamicSGD}                                   & \multicolumn{2}{c}{77.2(0.3)}                                     & \multicolumn{2}{c}{27.0(0.5)}                                     & \multicolumn{2}{c}{75.4(1.8)}                                     \\ \midrule
\multirow{9}{*}{DynamicAvg} & \multirow{3}{*}{a-g} & $\beta=0.3$ & 70.1(0.2)                     & 70.9(0.0)                         & 25.5(0.0)                     & 24.9(0.1)                         & 60.3(1.7)                     & 64.4(0.3)                         \\
                            &                      & $\beta=0.6$ & 72.7(0.2)                     & 75.4(0.4)                         & 27.5(1.2)                     & 25.4(0.3)                         & 67.0(1.7)                     & 70.4(0.3)                         \\
                            &                      & $\beta=0.9$ & 74.6(0.1)                     & 75.7(0.0)                         & 28.9(0.2)                     & 29.4(0.1)                         & 70.3(0.0)                     & 73.3(0.2)                         \\ \cmidrule{3-9}
                            & \multirow{3}{*}{b-g} & $\beta=0.3$ & 66.9(1.1)                     & 67.4(0.0)                         & 21.8(0.2)                     & 21.7(0.2)                         & 60.6(2.5)                     & 63.3(0.2)                         \\
                            &                      & $\beta=0.6$ & 68.8(0.1)                     & 68.4(0.6)                         & 20.8(0.2)                     & 17.2(0.1)                         & 60.4(1.2)                     & 61.3(1.5)                         \\
                            &                      & $\beta=0.9$ & 68.3(0.1)                     & 68.6(1.0)                         & 16.5(1.1)                     & 14.0(0.1)                         & 59.6(1.3)                     & 59.9(2.5)                         \\ \cmidrule{3-9}
                            & \multirow{3}{*}{c-g} & $\beta=0.3$ & 66.9(1.6)                     & 66.2(0.8)                         & 20.6(0.2)                     & 19.7(0.1)                         & 60.9(0.5)                     & 60.7(0.5)                         \\
                            &                      & $\beta=0.6$ & 66.7(0.3)                     & 68.2(1.1)                         & 21.3(0.1)                     & 17.5(0.1)                         & 60.8(3.0)                     & 61.1(0.4)                         \\
                            &                      & $\beta=0.9$ & 66.3(1.1)                     & 68.2(0.7)                         & 15.9(0.3)                     & 14.3(0.2)                         & 57.3(1.3)                     & 59.2(1.3)                         \\ \midrule
\multicolumn{3}{c}{FedProx}                                      & \multicolumn{2}{c}{65.4(0.2)}                                     & \multicolumn{2}{c}{18.1(0.3)}                                     & \multicolumn{2}{c}{50.8(0.0)}                                     \\
\multicolumn{3}{c}{FedEnsemble}                                  & \multicolumn{2}{c}{66.0(0.9)}                                     & \multicolumn{2}{c}{18.0(0.3)}                                     & \multicolumn{2}{c}{31.1(0.0)}                                     \\
\multicolumn{3}{c}{SCAFFOLD}                                     & \multicolumn{2}{c}{66.2(0.0)}                                     & \multicolumn{2}{c}{23.8(0.1)}                                     & \multicolumn{2}{c}{N/A}                                      \\
\multicolumn{3}{c}{FedGen}                                       & \multicolumn{2}{c}{64.3(0.1)}                                     & \multicolumn{2}{c}{17.3(0.3)}                                     & \multicolumn{2}{c}{N/A}                                           \\
\multicolumn{3}{c}{FedDyn}                                       & \multicolumn{2}{c}{63.6(0.3)}                                     & \multicolumn{2}{c}{16.4(0.0)}                                     & \multicolumn{2}{c}{N/A}                                           \\
\multicolumn{3}{c}{FedNova}                                      & \multicolumn{2}{c}{66.5(0.6)}                                     & \multicolumn{2}{c}{17.7(0.6)}                                     & \multicolumn{2}{c}{45.5(0.0)}                                     \\
\multicolumn{3}{c}{FedAvg}                                       & \multicolumn{2}{c}{66.3(0.0)}                                     & \multicolumn{2}{c}{18.4(0.6)}                                     & \multicolumn{2}{c}{37.7(0.8)}                                     \\ \bottomrule
\end{tabular}
\end{small}
\label{tab:cifar10_cifar100_femnist_k2_d001_cnn_main}
\end{table*}

\begin{figure*}[h]
\centering
{\includegraphics[width=1\linewidth]{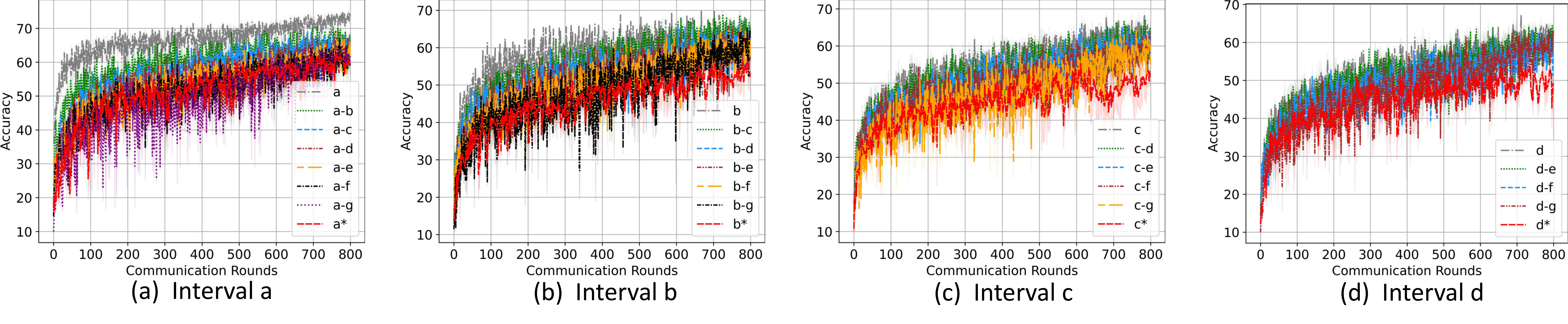}}
\vspace{-0.8cm}
\caption{Learning curves for all communication interval combinations in Table~\ref{tab:freq_ablation_cifar10_main}, $Dir(0.1)$ setting.}
\label{fig:learning curve}
\end{figure*}

\begin{figure*}[h]
\centering
\includegraphics[width=1\linewidth]{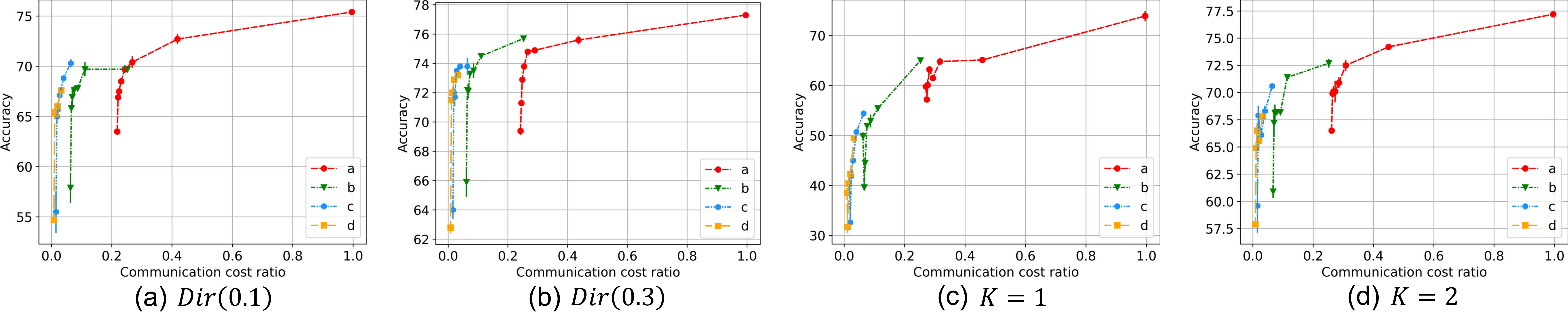}
\vspace{-0.5cm}
\caption{This figure illustrates the trend from Table~\ref{tab:freq_ablation_cifar10_main}, comparing the high-frequency group only, extended intervals, and DynamicSGD, using a CNN with CIFAR-10. For the interval 'a', we depict a, a-b, a-c, a-d, a-e, a-g, and a*. The same pattern is replicated for intervals b, c, and d.}
\label{fig:ablation_trend}
\end{figure*}

\subsection{Experimental results}
\paragraph{\textbf{Data statistical heterogeneity}} The effectiveness of DynamicFL is demonstrated across various experimental setups, as shown in Tables~\ref{tab:cifar10_cifar100_k2_resnet_main}--\ref{tab:freq_ablation_cifar10_main}. By effectively leveraging communication heterogeneity, DynamicFL yields substantial enhancements in global model performance, with improvements ranging from minor percentages to a striking 46.2\%. This performance range depends on the degree of statistical heterogeneity and the specific configurations applied. 
Our results reveal a notable performance gap between DynamicSGD (the upper bound) and the best baseline across all statistically heterogeneous settings. This gap, which equals 12.8\%, 49.6\%, and 24.6\% performance differences in CIFAR10 and CIFAR100's ResNet $K=2$ settings, and the CNN $Dir(0.01)$ setting for FEMNIST, respectively, underscores the necessity of DynamicAvg. When dealing with statistical heterogeneity, no advanced FL algorithm consistently outperforms FedAvg, emphasizing the challenges associated with such heterogeneity. This statistical heterogeneity considerably diminishes the accuracy of all FL algorithms across all experiments, particularly when each client's data is sampled from a single class, with CIFAR100 being the most complex among the datasets examined. Interestingly, due to its robust network structure, ResNet-18 shows a greater capacity for knowledge gain via DynamicFL than CNN, suggesting that more robust structures may derive greater benefits from DynamicAvg.

\begin{table*}[h!]
\caption{Accuracy results in comparing various communication interval combinations, and the standardized server total communication cost $\kappa_{g}^{t}$, using CNN under the CIFAR-10, \textit{Fix-0.3} configuration. 'a' signifies that all active clients train for 'a' intervals, while 'a*' indicates training exclusively involving high-frequency group clients. The same pattern is replicated for intervals b, c, and d.}
\centering
\begin{small}
\begin{tabular}{ccccccccc}
\toprule
\multirow{2}{*}{\begin{tabular}[c]{@{}c@{}}Communication   \\ interval\end{tabular}} & \multicolumn{2}{c}{$Dir(0.1)$} & \multicolumn{2}{c}{$Dir(0.3)$} & \multicolumn{2}{c}{$K=1$} & \multicolumn{2}{c}{$K=2$} \\ \cline{2-9} 
                                                                                     & Accuracy      & $\kappa_{g}^{t}$            & Accuracy      & $\kappa_{g}^{t}$              & Accuracy   & $\kappa_{g}^{t}$           & Accuracy   & $\kappa_{g}^{t}$            \\ \midrule
a                                                                                    & 75.4(0.1)     & 0.996   & 77.3(0.1)     & 0.996   & 73.9(1.0)  & 0.996 & 77.2(0.3)  & 0.996 \\
a-b                                                                                  & 72.7(0.5)     & 0.418   & 75.6(0.3)     & 0.436   & 65.1(0.0)  & 0.457 & 74.2(0.3)  & 0.450 \\
a-e                                                                                  & 68.5(0.1)     & 0.231   & 73.8(0.1)     & 0.254   & 63.2(0.7)  & 0.282 & 70.1(1.0)  & 0.273 \\
a*                                                                                   & 63.5(0.3)     & 0.218   & 69.4(0.3)     & 0.242   & 59.8(0.4)  & 0.270 & 66.5(0.1)  & 0.261 \\ \midrule
b                                                                                    & 69.7(0.0)     & 0.252   & 75.7(0.0)     & 0.252   & 65.0(0.0)  & 0.252 & 72.7(0.4)  & 0.252 \\
b-c                                                                                  & 69.7(0.7)     & 0.111   & 74.5(0.0)     & 0.110   & 55.4(0.3)  & 0.111 & 71.4(0.0)  & 0.114 \\
b-e                                                                                  & 67.6(0.4)     & 0.075   & 73.3(0.2)     & 0.073   & 51.9(0.6)  & 0.075 & 68.2(0.1)  & 0.079 \\
b*                                                                                   & 57.9(1.5)     & 0.063   & 65.9(1.0)     & 0.061   & 49.8(0.8)  & 0.062 & 60.9(0.6)  & 0.067 \\ \midrule
c                                                                                    & 70.3(0.4)     & 0.064   & 73.8(0.6)     & 0.064   & 54.4(0.1)  & 0.064 & 70.6(0.0)  & 0.064 \\
c-d                                                                                  & 68.8(0.1)     & 0.040   & 73.8(0.0)     & 0.040   & 50.7(0.1)  & 0.040 & 68.3(0.5)  & 0.040 \\
c-e                                                                                  & 67.1(0.4)     & 0.027   & 73.5(0.2)     & 0.028   & 45.0(0.4)  & 0.029 & 66.1(0.2)  & 0.028 \\
c*                                                                                   & 55.5(2.1)     & 0.015   & 64.0(0.6)     & 0.016   & 40.3(1.3)  & 0.017 & 59.6(2.5)  & 0.015 \\ \midrule
d                                                                                    & 67.6(0.3)     & 0.032   & 73.2(0.2)     & 0.032   & 49.5(0.5)  & 0.032 & 67.8(0.0)  & 0.032 \\
d-e                                                                                  & 66.0(0.4)     & 0.020   & 72.9(0.1)     & 0.020   & 42.3(0.4)  & 0.020 & 65.6(0.4)  & 0.020 \\
d*                                                                                   & 54.7(0.3)     & 0.008   & 62.8(0.4)     & 0.008   & 38.5(1.0)  & 0.009 & 57.9(0.3)  & 0.008 \\ \midrule
FedProx                                                                              & 64.8(0.8)    & 0.004      & 71.1(0.5)    & 0.004      & 20.9(0.3)  & 0.004   & 65.4(0.2)  & 0.004   \\
FedEnsemble                                                                          & 65.3(0.3)    & 0.004      & 71.7(0.2)    & 0.004      & 19.9(0.7)  & 0.004   & 66.0(0.9)  & 0.004   \\
SCAFFOLD                                                                             & 52.6(1.8)    & 0.008      & 71.0(1.0)    & 0.008      & 14.9(1.7)  & 0.008   & 66.2(0.0)  & 0.008   \\
FedGen                                                                               & 62.4(1.8)    & 0.009      & 70.4(0.7)    & 0.009      & 14.5(0.6)  & 0.009   & 64.3(0.1)  & 0.009   \\
FedDyn                                                                               & 66.2(0.0)    & 0.004      & 70.5(0.0)    & 0.004      & 17.3(2.2)  & 0.004   & 63.6(0.3)  & 0.004   \\
FedNova                                                                              & 63.9(0.6)    & 0.004      & 71.5(0.2)    & 0.004      & 20.2(0.8)  & 0.004   & 66.5(0.6)  & 0.004   \\
FedAvg                                                                               & 65.0(0.5)    & 0.004      & 69.6(0.5)    & 0.004      & 19.1(1.4)  & 0.004   & 66.3(0.0)  & 0.004   \\ \bottomrule
\end{tabular}
\end{small}
\label{tab:freq_ablation_cifar10_main}
\end{table*}

\paragraph{\textbf{Communication resources}} In a variety of communication heterogeneities, we demonstrate the robustness and flexibility of DynamicFL. Our findings highlight that, by fully capitalizing on communication heterogeneity, the high-frequency group resembling FedSGD can consistently bolster global model performance, thereby validating our design. Our evaluation reveals a consistent trend of enhanced performance when more communication resources are utilized, achieving results on par with DynamicSGD with less than 60\% communication cost, and in some instances, as little as 15\%. Remarkably, the degree of performance improvement appears to be closely linked to the level of statistical heterogeneity in the data. Specifically, settings marked by higher degrees of data heterogeneity typically yielded more substantial performance enhancements compared to the baselines. This highlights the effectiveness and flexibility of DynamicAvg in highly statistically heterogeneous settings. The robustness of our methodology is further validated by the consistency of our learning curves, as depicted in Figure~\ref{fig:learning curve}. However, in certain scenarios such as those presented in Table~\ref{tab:cifar10_cifar100_femnist_k2_d001_cnn_main}, minor improvements in performance are observed for smaller intervals 'b' and 'c' as we increase the communication resource budgets for clients and server. This might imply a non-linear correlation, suggesting that even with a small number of clients involved in high-frequency training, a performance comparable to that of the majority of clients can be achieved. Additionally, this might suggest that for simpler networks, a higher communication frequency is required to address statistical heterogeneity, serving as compensation for the disparities inherent in the network structure. When comparing \textit{Fix} and \textit{Dynamic} configurations, similar performance levels are achieved. A slightly lower accuracy in some \textit{Fix} cases may be due to bias, as only a subset of clients can maintain high-frequency communication. Although DynaComm alleviates bias, it still remains a factor.

\begin{figure*}[hbt]
\centering
{\includegraphics[width=1\linewidth]{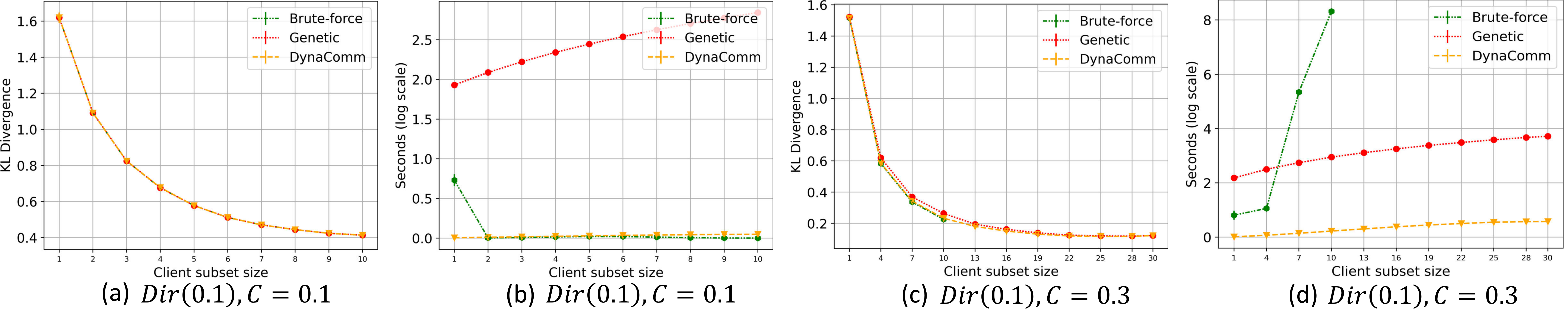}}
\caption{Comparison of KL-divergence and Runtime for CIFAR100 with $Dir(0.1)$.}
\label{fig-main:dp_cifar100_d01}
\end{figure*}
\begin{figure*}[hbt]
\centering
{\includegraphics[width=1\linewidth]{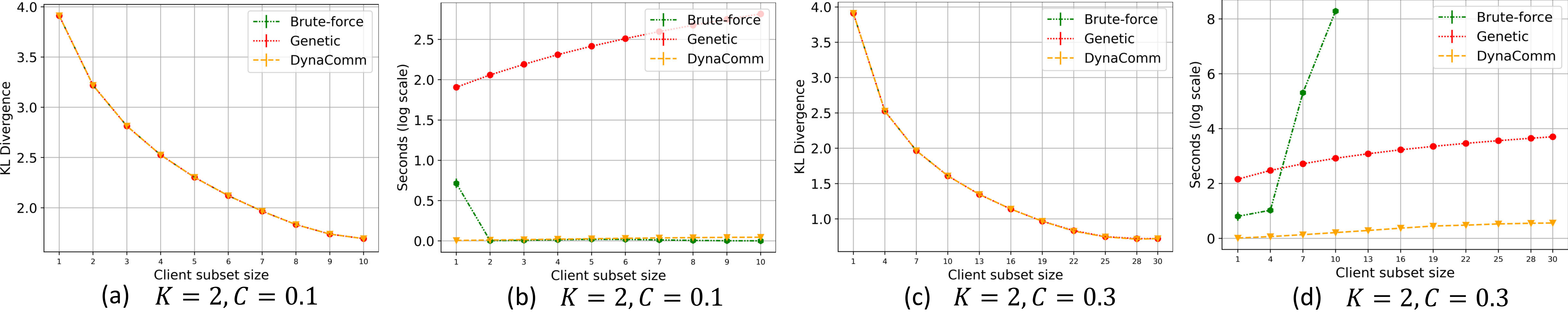}}
\caption{Comparison of KL-divergence and Runtime for CIFAR100 with $K=2$.}
\vspace{-0.3cm}
\label{fig-main:dp_cifar100_l2}
\end{figure*}
\begin{figure*}[hbt]
\centering
{\includegraphics[width=1\linewidth]{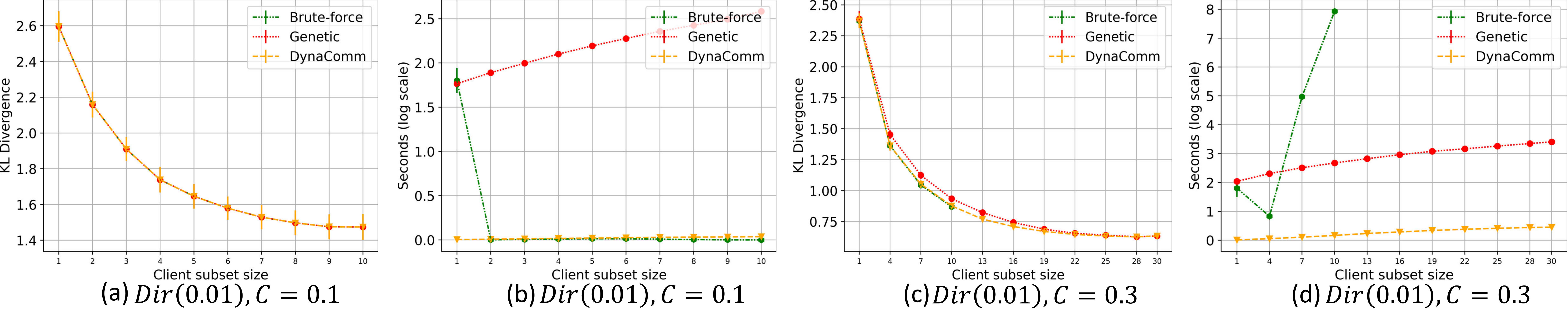}}
\vspace{-0.6cm}
\caption{Comparison of KL-divergence and Runtime for FEMNIST with $Dir(0.01)$.}
\vspace{-0.3cm}
\label{fig-main:dp_femnist_d001}
\end{figure*}
\begin{figure*}[hbt]
\centering
{\includegraphics[width=1\linewidth]{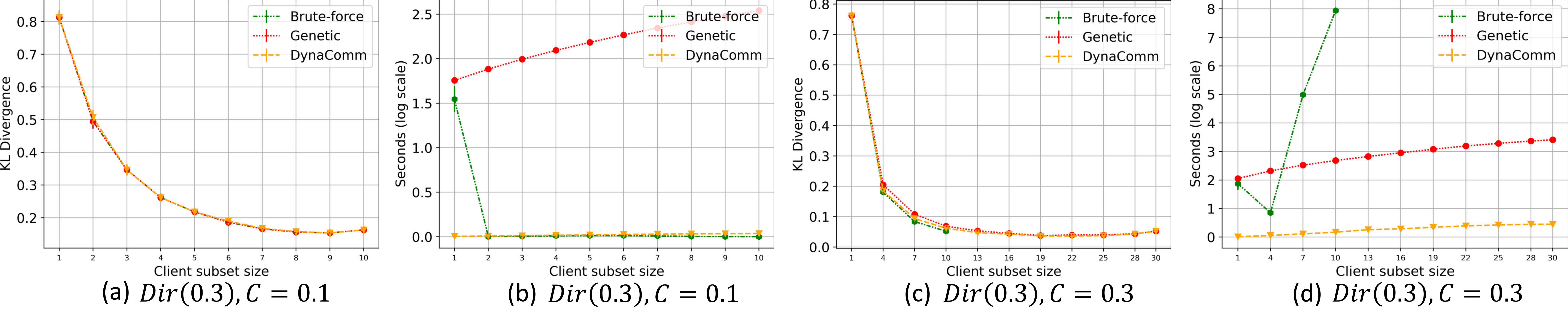}}
\vspace{-0.7cm}
\caption{Comparison of KL-divergence and Runtime for FEMNIST with $Dir(0.3)$.}
\vspace{-0.5cm}
\label{fig-main:dp_femnist_d03}
\end{figure*}

\paragraph{\textbf{Communication intervals}} We further demonstrate the robustness and flexibility of DynamicFL under various combinations of communication intervals. We emphasize that a low-frequency group, mirroring FedAvg's behavior for efficient dataset iteration, can indeed enhance model performance. Moreover, we observe a non-linear relationship between global model performance and communication cost. As depicted in Figure~\ref{fig:ablation_trend} and Table~\ref{tab:freq_ablation_cifar10_main}, we note that considerable performance enhancements can be achieved by reducing the long interval of the low-frequency group, while maintaining the short interval for the high-frequency group, with only a minimal increase in communication costs. Notably, utilizing the shortest interval isn't always necessary. Moderately longer intervals, akin to FedSGD, can also benefit the global model, yielding similar performance to shorter intervals. For instance, in Table~\ref{tab:freq_ablation_cifar10_main}, 'c-e' under $Dir(0.3)$ achieves 73.5\% accuracy with a 0.028 communication cost ratio, similar to 'a-e' with a 0.254 communication cost ratio. Furthermore, in the $K=1$ setting, using 'd-e' yields an accuracy boost of over 17.6\% against the top baseline by leveraging communication heterogeneity. Table~\ref{tab:freq_ablation_cifar10_main} shows that including a low-frequency group like FedAvg in training results in higher performance than using only high-frequency clients, as seen in four statistical heterogeneity instances (e.g., 'a*' to 'a'). Under $K = 1$, for CIFAR-10 with \textit{Fix-0.3}, 'd-g', 'd-f', and 'd-e' achieve accuracies of 31.7\%, 40.4\%, and 42.3\%, respectively. These results outperform baselines by 10.8-17.2\%, 19.5-25.9\%, and 21.4-27.8\% with 1.22-2.75, 1.55-3.5, and 2.22-5.0 times the communication cost, respectively. Similarly, under $K = 1$ for CIFAR-100 with \textit{Fix-0.3}, 'd-g', 'd-f', and 'd-e' report accuracies of 2.7\%, 2.9\%, and 3.4\%, respectively. Compared to the baselines, our results achieve improved accuracies by factors of 1.8-2.25, 1.93-2.42, and 2.27-2.83 and come with 1.18-3.25, 1.36-3.75, and 1.91-5.25 times the communication cost, respectively. Lastly, under $Dir(0.01)$ for FEMNIST with \textit{Fix-0.3}, 'd-g', 'd-f', and 'd-e' achieve accuracies of 58.2\%, 60.8\%, and 62.6\%, respectively. These results surpass baselines by 7.4-27.1\%, 10.0-29.7\%, and 11.8-31.5\% with 3, 3.75, and 5 times the communication costs, respectively. However, in the most heterogeneous $K=1$ scenario depicted in Figure~\ref{fig:ablation_trend} (c), training exclusively with high-frequency group clients outperforms adding 'g' interval low-frequency group, implying FedAvg's potential to degrade performance in statistical heterogeneity.

\subsection{Dynamic Communication Optimizer}
In this section, we highlight the effectiveness of DynaComm, as described in Section 3 (Method section) of the main text. In summary, DynaComm is capable of managing communication resource constraints while achieving results comparable to the brute-force method but with significantly lower time costs, and it surpasses the Genetic Algorithm in performance. Furthermore, we note that KL-divergence tends to plateau or even increase once the subset size surpasses a specific threshold. This observation suggests a non-linear relationship between the sampling and global label distributions, which further affirms the design correctness of DynamicFL.


We implement the Genetic Algorithm~\cite{genetic} using the scikit-opt library, adhering to the default parameter configuration as outlined in the library's documentation. This standard configuration guarantees the algorithm's wide-ranging applicability and preserves comparability across diverse problem settings. The hyperparameters encompass population size, maximum iterations, and mutation probability, all of which are pivotal elements impacting the algorithm's search process and convergence performance. We use population size of 50, max iterations as 200, and mutation probability of 0.01.

Next, we provide a comparison of the KL-divergence and Runtime for the brute-force method, Genetic Algorithm, and DynaComm when applied to CIFAR100, and FEMNIST, as shown in Figures~\ref{fig-main:dp_cifar100_d01}--\ref{fig-main:dp_femnist_d03}. The figures illustrate the KL-divergence and Runtime at 0.1 (a-b) and 0.3 (c-d) active client rates for four types of statistical heterogeneity: $Dir(0.1)$, $Dir(0.3)$, $K=1$, and $K=2$. 

We first explore the effect of the active rate. DynaComm offers more considerable benefits when the active rate $C$ is high, as the union of a small fraction of active clients' labels can achieve a strikingly low KL-divergence that decreases swiftly. For example, when the client subset size is 10 and $C=0.3$ for the CIFAR10 dataset, the KL-divergence is almost zero. Furthermore, when using DynamicFL, even with a low active rate, selecting only part of the active clients might yield results comparable to those obtained by engaging all active clients, due to the plateau or even increase effect.

Next, when evaluating the effect of target classes, we found that as the number of target classes increases, a larger client subset size is required to achieve the same KL-divergence as the dataset with fewer target classes. For example, when comparing CIFAR10 and CIFAR100 under $Dir(0.1)$ and with a client subset size of 5, CIFAR10 can achieve a KL-divergence of 0.25, while CIFAR100 reaches 0.55. This suggests that CIFAR100 requires a more substantial client subset size, implying that CIFAR100 is the more challenging dataset.

Lastly, we assessed the effect of statistical heterogeneity within a single dataset and found that higher statistical heterogeneity correlates with increased KL-divergence. For instance, in the FEMNIST dataset with $C=0.1$, we found KL-divergence values of 1.5 for $Dir(0.01)$, 0.4 for $Dir(0.1)$, and 0.15 for $Dir(0.3)$ when the client subset size is 10. This correlation suggests a predictable decline in performance with increased levels of statistical heterogeneity.

\vspace{-4pt}
\section{Conclusion}
We introduce DynamicFL, an innovative FL framework that addresses statistical heterogeneity by leveraging communication heterogeneity. Our method successfully bridges FedSGD and FedAvg, demonstrating its robustness and adaptability through extensive experiments across diverse settings. To the best of our knowledge, \textit{DynamicFL is the first FL method that flexibly bridges gap between FedSGD and FedAvg, adaptively accommodating varying client-server communication capabilities and data statistical heterogeneity.} 

Our evaluation shows up to 10\% accuracy boost compared to baseline methods.
Future research directions involve the development of advanced applications or methodologies designed to leverage communication heterogeneity more effectively.

\bibliography{arxiv}
\bibliographystyle{acm}

\clearpage
\appendix
\onecolumn

\centerline{\LARGE Appendix for ``Dynamic Federated Learning''}
The \textbf{Appendix} provides additional information on the DynamicFL framework, including details of the Dynamic Communication Optimizer in Section~\ref{sec: DynaComm}, experimental setup in Section~\ref{sec:experimental_setup}, supplementary experimental results in Section~\ref{sec:experimental_results}, and technical analysis in Section~\ref{sec:theory}.

\vspace{-0.3cm}
\section{Dynamic Communication Optimizer}
\label{sec: DynaComm}
In this section, we highlight the effectiveness of DynaComm, as described in Section 3 (Method section) of the main text. In summary, DynaComm is capable of managing communication resource constraints while achieving results comparable to the brute-force method but with significantly lower time costs, and it surpasses the Genetic Algorithm in performance. Furthermore, we note that KL-divergence tends to plateau or even increase once the subset size surpasses a specific threshold. This observation suggests a non-linear relationship between the sampling and global label distributions, which further affirms the design correctness of DynamicFL.

We provide a summary of the essential hyperparameters used for the Genetic Algorithm, an optimization method inspired by natural selection, as depicted in Table \ref{tab:genetic_algo}. We implement the Genetic Algorithm using the scikit-opt library, adhering to the default parameter configuration as outlined in the library's documentation. This standard configuration guarantees the algorithm's wide-ranging applicability and preserves comparability across diverse problem settings. The hyperparameters encompass population size, maximum iterations, and mutation probability, all of which are pivotal elements impacting the algorithm's search process and convergence performance.
\begin{table}[htbp]
\caption{Hyperparameters used for the Genetic Algorithm. The Genetic Algorithm is implemented by https://github.com/guofei9987/scikit-opt. We employ the default parameter configuration, which is available at: https://scikit-opt.github.io/scikit-opt/\#/en/args}
\centering
\begin{tabular}{@{}ccccc@{}}
\toprule
The Genetic Algorithm      \\ \midrule
Population   Size    & 50     \\ \midrule
Maximum iteration    & 200    \\ \midrule
Mutation Probability & 0.01   \\ \bottomrule
\end{tabular}
\label{tab:genetic_algo}
\end{table}

Next, we provide a thorough comparison of the KL-divergence and Runtime for the brute-force method, Genetic Algorithm, and DynaComm when applied to CIFAR10, CIFAR100, and FEMNIST, as shown in Figures~\ref{fig:dp_cifar10_d01}--\ref{fig:dp_femnist_d03}. The figures illustrate the KL-divergence and Runtime at 0.1 (a-b) and 0.3 (c-d) active client rates for four types of statistical heterogeneity: $Dir(0.1)$, $Dir(0.3)$, $K=1$, and $K=2$. 

We first explore the effect of the active rate. our DynaComm offers more considerable benefits when the active rate $C$ is high, as the union of a small fraction of active clients' labels can achieve a strikingly low KL-divergence that decreases swiftly. For example, when the client subset size is 10 and $C=0.3$ for the CIFAR10 dataset, the KL-divergence is almost zero. Furthermore, when using DynamicFL, even with a low active rate, selecting only part of the active clients might yield results comparable to those obtained by engaging all active clients, due to the plateau or even increase effect.

Next, when evaluating the effect of target classes, we found that as the number of target classes increases, a larger client subset size is required to achieve the same KL-divergence as the dataset with fewer target classes. For example, when comparing CIFAR10 and CIFAR100 under $Dir(0.1)$ and with a client subset size of 5, CIFAR10 can achieve a KL-divergence of 0.25, while CIFAR100 reaches 0.55. This suggests that CIFAR100 requires a more substantial client subset size, implying that CIFAR100 is the more challenging dataset.

Lastly, we assessed the effect of statistical heterogeneity within a single dataset and found that higher statistical heterogeneity correlates with increased KL-divergence. For instance, in the FEMNIST dataset with $C=0.1$, we found KL-divergence values of 1.5 for $Dir(0.01)$, 0.4 for $Dir(0.1)$, and 0.15 for $Dir(0.3)$ when the client subset size is 10. This correlation suggests a predictable decline in performance with increased levels of statistical heterogeneity.
\begin{figure}[h!]
\vspace{-0.5cm}
\centering
{\includegraphics[width=0.95\linewidth]{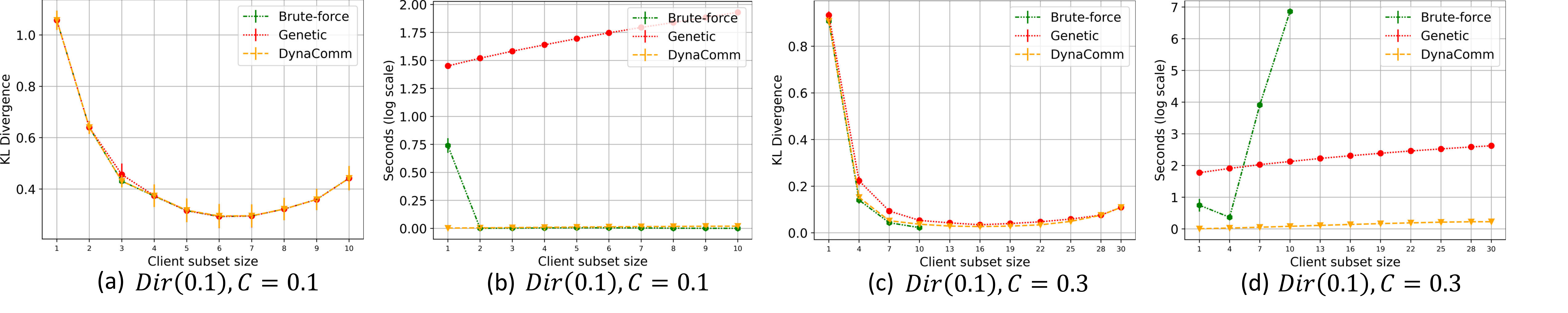}}
\vspace{-0.35cm}
\caption{Comparison of KL-divergence and Runtime for CIFAR10 with $Dir(0.1)$.}
\label{fig:dp_cifar10_d01}
\end{figure}
\begin{figure}[h!]
\centering
{\includegraphics[width=0.95\linewidth]{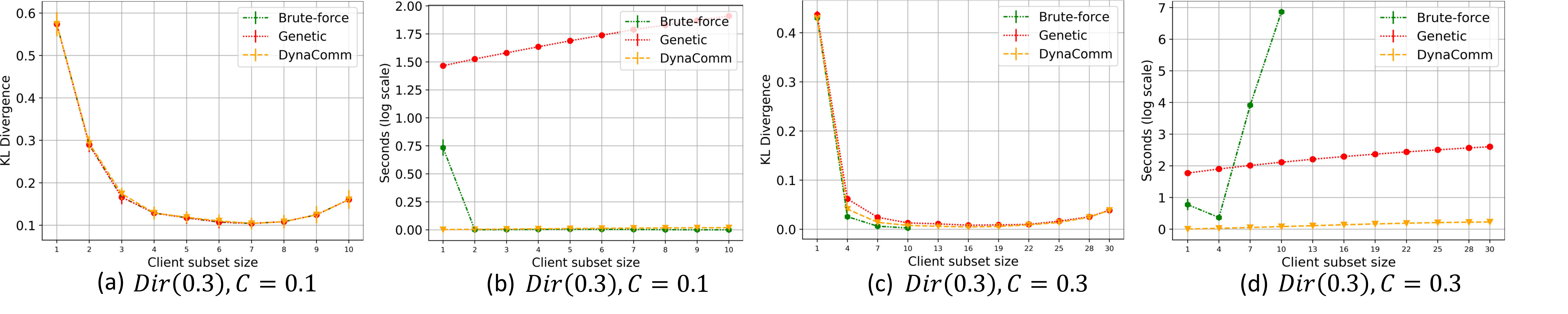}}
\vspace{-0.35cm}
\caption{Comparison of KL-divergence and Runtime for CIFAR10 with $Dir(0.3)$.}
\label{fig:dp_cifar10_d03}
\end{figure}
\begin{figure}[h!]
\centering
{\includegraphics[width=0.95\linewidth]{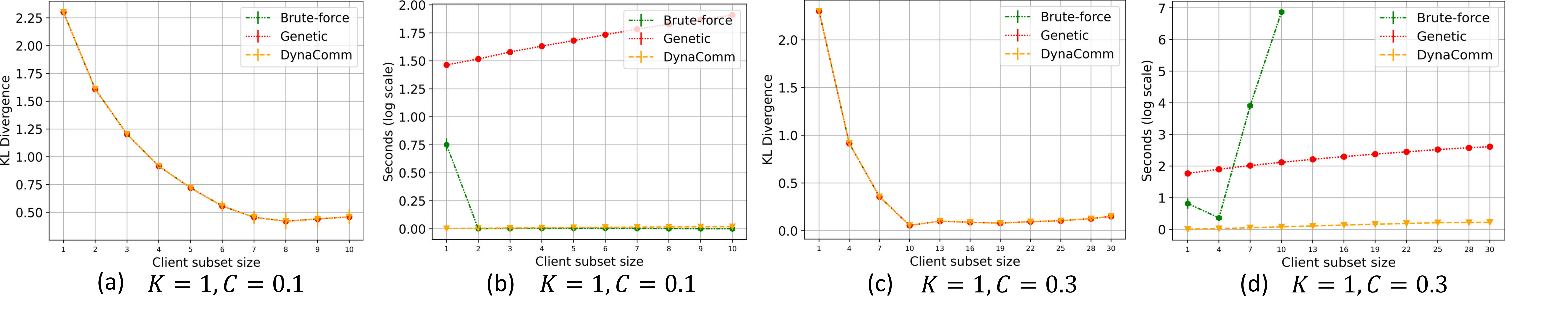}}
\vspace{-0.35cm}
\caption{Comparison of KL-divergence and Runtime for CIFAR10 with $K=1$.}
\label{fig:dp_cifar10_l1}
\end{figure}
\begin{figure}[h!]
\centering
{\includegraphics[width=0.95\linewidth]{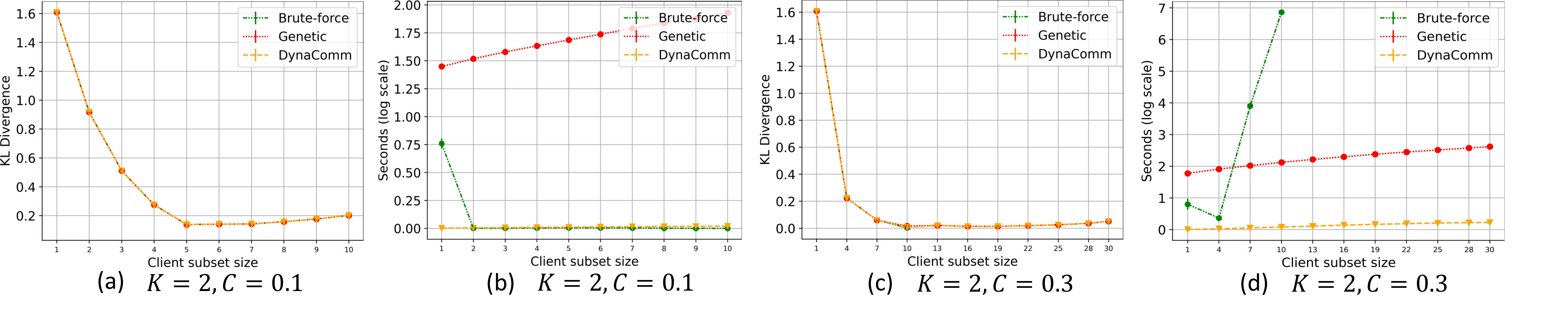}}
\vspace{-0.35cm}
\caption{Comparison of KL-divergence and Runtime for CIFAR10 with $K=2$.}
\label{fig:dp_cifar10_l2}
\end{figure}
\begin{figure}[h!]
\centering
{\includegraphics[width=0.95\linewidth]{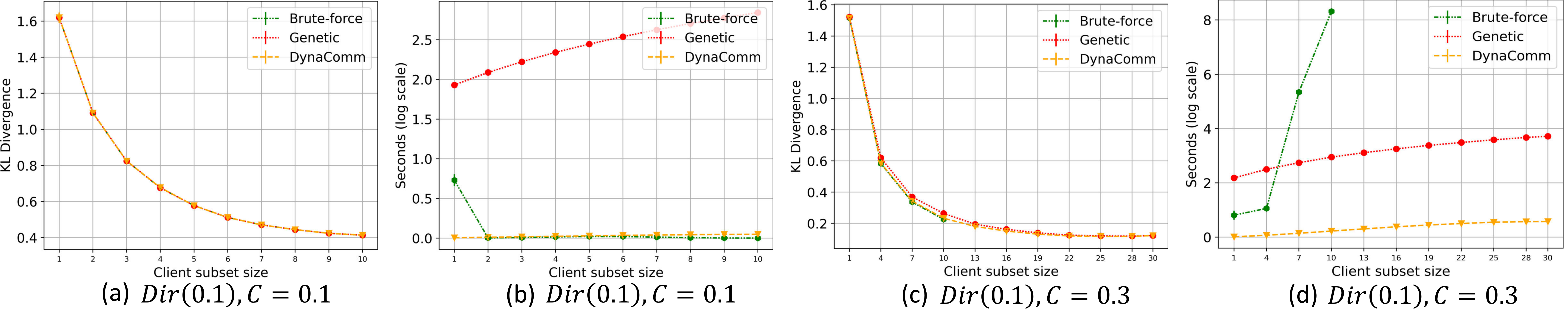}}
\vspace{-0.35cm}
\caption{Comparison of KL-divergence and Runtime for CIFAR100 with $Dir(0.1)$.}
\label{fig:dp_cifar100_d01}
\end{figure}
\begin{figure}[h!]
\centering
{\includegraphics[width=0.95\linewidth]{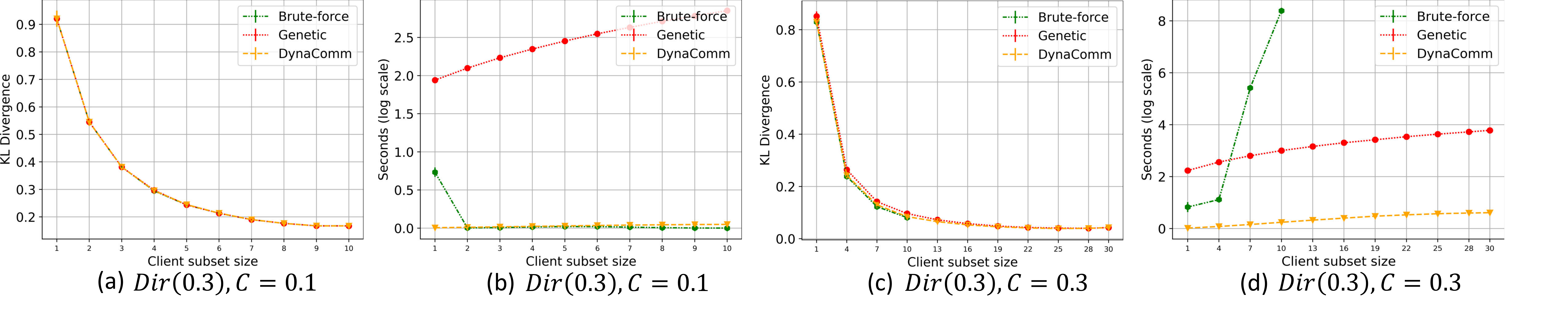}}
\vspace{-0.35cm}
\caption{Comparison of KL-divergence and Runtime for CIFAR100 with $Dir(0.3)$.}
\label{fig:dp_cifar100_d03}
\end{figure}
\begin{figure}[h!]
\centering
{\includegraphics[width=0.95\linewidth]{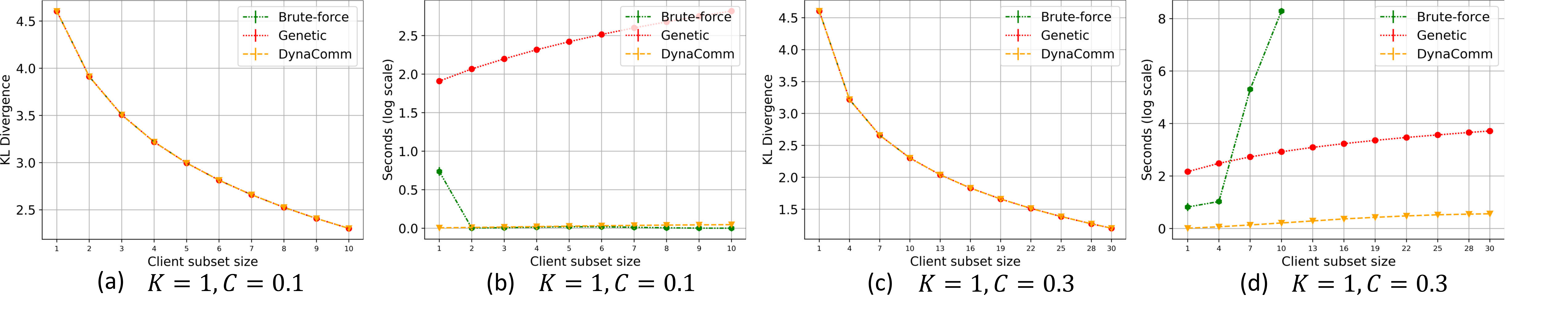}}
\vspace{-0.35cm}
\caption{Comparison of KL-divergence and Runtime for CIFAR100 with $K=1$.}
\label{fig:dp_cifar100_l1}
\end{figure}
\begin{figure}[h!]
\centering
{\includegraphics[width=0.95\linewidth]{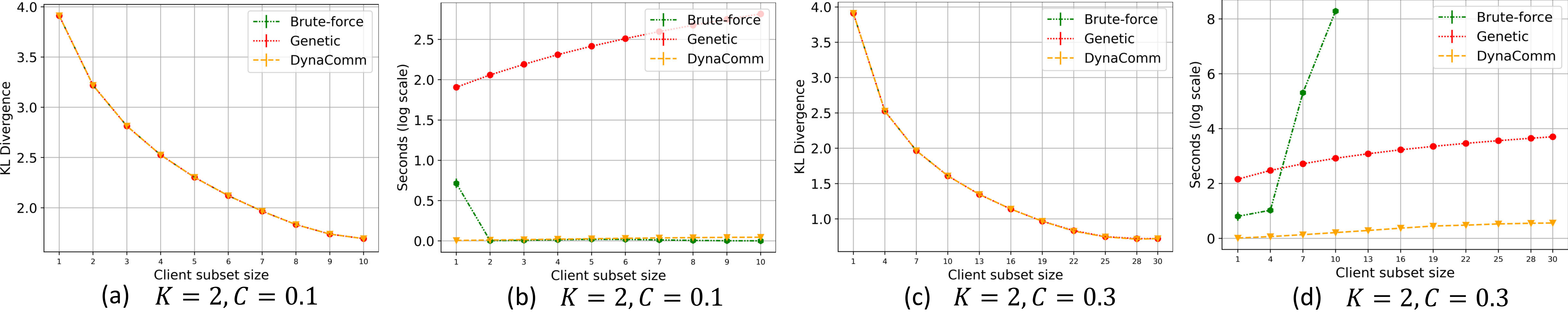}}
\vspace{-0.35cm}
\caption{Comparison of KL-divergence and Runtime for CIFAR100 with $K=2$.}
\label{fig:dp_cifar100_l2}
\end{figure}
\begin{figure}[h!]
\centering
{\includegraphics[width=0.95\linewidth]{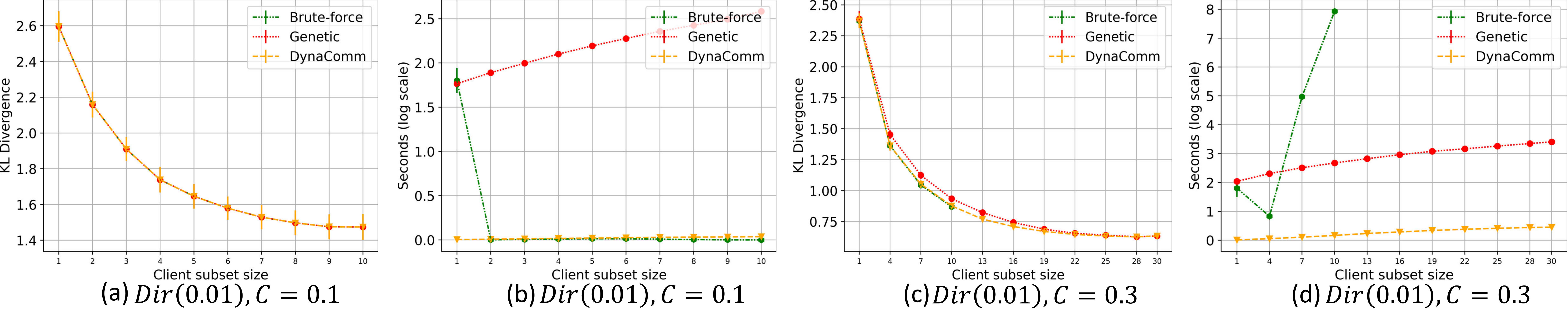}}
\vspace{-0.35cm}
\caption{Comparison of KL-divergence and Runtime for FEMNIST with $Dir(0.01)$.}
\label{fig:dp_femnist_d001}
\end{figure}
\begin{figure}[h!]
\centering
{\includegraphics[width=0.95\linewidth]{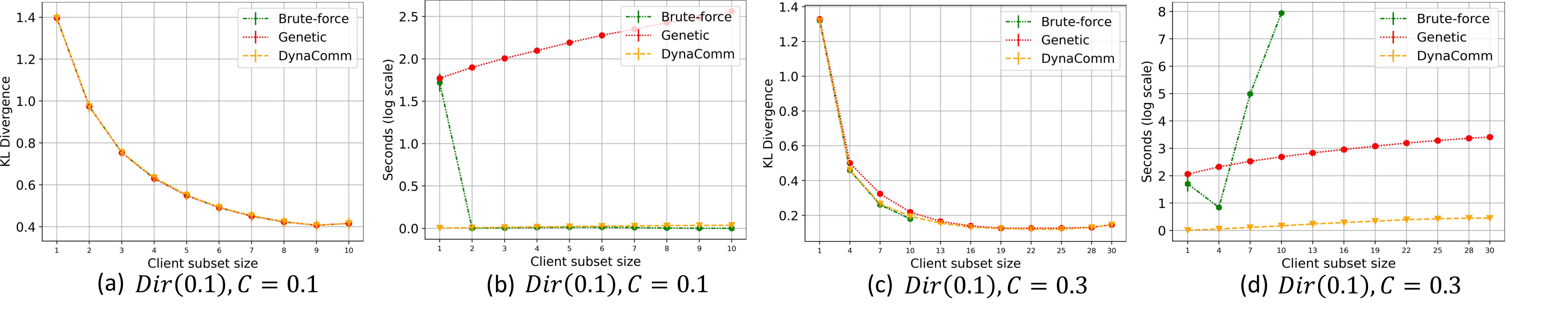}}
\vspace{-0.35cm}
\caption{Comparison of KL-divergence and Runtime for FEMNIST with $Dir(0.1)$.}
\label{fig:dp_femnist_d01}
\end{figure}
\begin{figure}[h!]
\centering
{\includegraphics[width=0.95\linewidth]{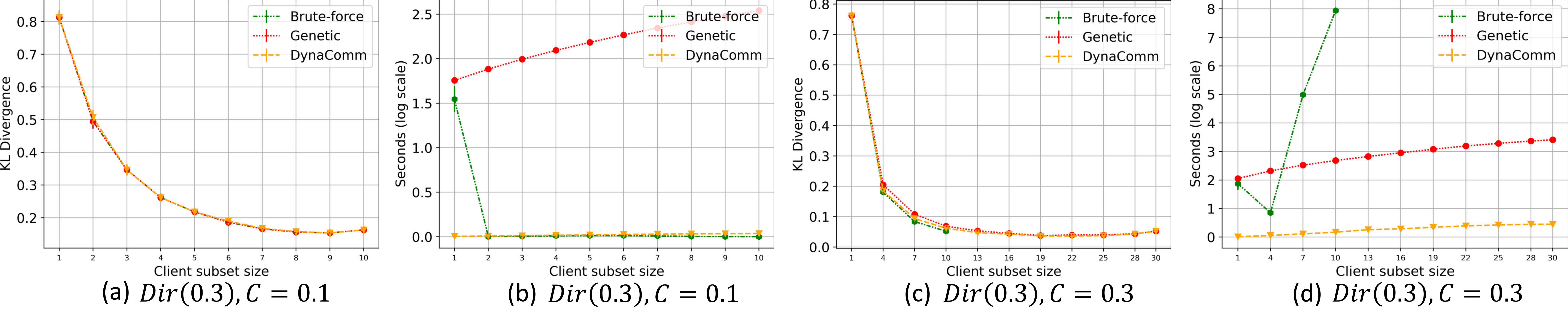}}
\vspace{-0.35cm}
\caption{Comparison of KL-divergence and Runtime for FEMNIST with $Dir(0.3)$.}
\vspace{-0.35cm}
\label{fig:dp_femnist_d03}
\end{figure}

\clearpage
\section{Experimental setup}
\label{sec:experimental_setup}
In this section, we introduce the experimental setup. We employ the top-1 accuracy on the test dataset as a metric to compare the studied algorithms. For a fair comparison, we execute all the examined algorithms for an equal total number of gradient updates. The number of rounds is set to 800 to allow the model to converge fully. For all FedAvg-based methods, we set the local epoch $E$ to 5.

In Table~\ref{tab:network_arch}, we provide the configurations for hidden sizes across different models and datasets. 
\begin{table}[h!]
\caption{Network architecture for all methods}
\centering
\begin{tabular}{@{}cccc@{}}
\toprule
Data                          & CIFAR10              & CIFAR100             & FEMNIST  \\ \midrule
CNN Hidden Size        & \multicolumn{3}{c}{{[}120, 84{]}}                      \\ \midrule
ResNet-18 Hidden Size & \multicolumn{2}{c}{{[}64, 128, 256, 512{]}} & N/A      \\ \bottomrule
\end{tabular}
\label{tab:network_arch}
\end{table}

In Table~\ref{tab:dataset_info}, we present the key characteristics of the CIFAR10, CIFAR100~\cite{netzer2011reading, li2022federated}, and FEMNIST~\cite{he2016deep} datasets used in our experiments. These diverse datasets, with varying numbers of features and classes, allow for a comprehensive evaluation of DynamicFL's performance and robustness.
\begin{table}[h!]
\caption{Statistics of datasets in our experiments}
\centering
\begin{tabular}{@{}cccc@{}}
\toprule
Datasets           & CIFAR10 & CIFAR100 & FEMNIST \\ \midrule
Training instances & 50,000  & 50,000   & 644,210 \\ \midrule
Test instances     & 10,000  & 10,000   & 161,053 \\ \midrule
Features           & 1,024   & 1,024    & 784     \\ \midrule
Classes            & 10      & 100      & 62      \\ \bottomrule
\end{tabular}
\label{tab:dataset_info}
\end{table}

In Table \ref{tab:hyperparameter}, we present the shared hyperparameters used across our experiments for both CNN and ResNet-18 models and different methods. Moreover, in the experiments, the hyperparameters are not tuned for optimal performance in the federated setting. Thus, the experimental results could potentially be further enhanced.
\begin{table}[htbp]
\centering
\caption{Hyperparameters}
\begin{tabular}{@{}ccc@{}}
\toprule
\multirow{9}{*}{Shared parameter} & Batch size                   & 10                        \\
                                  & Optimizer                    & SGD                       \\
                                  & Learning rate for CNN        & 1.00E-02                  \\
                                  & Learning rate for Resnet-18  & 3.00E-02                  \\
                                  & Weight decay                 & 5.00E-04                  \\
                                  & Momentum                     & 9.00E-01                  \\
                                  & Nesterov                     & \checkmark \\
                                  & Communication round          & 800                       \\
                                  & Scheduler                    & Cosine Annealing          \\ \midrule
\multirow{2}{*}{FedGen}           & Generator optimizer          & adam                      \\
                                  & Generator learning rate      & 1.00E-04                  \\
                                  \midrule
FedDyn                            & Regularization   parameter   & 0.1                       \\ \midrule
FedProx                           & Regularization parameter     & 0.01                      \\ \bottomrule
\end{tabular}
\label{tab:hyperparameter}
\end{table}

\clearpage
\section{Experimental results}
\label{sec:experimental_results}
In this section, we present the experimental results obtained from extensive experimental settings. These substantial results further affirm the robustness and flexibility of DynamicFL, as discussed in Section 4.3 (Experimental Results). We demonstrate that statistical heterogeneity can be effectively managed by fully capitalizing on communication heterogeneity.

In Subsection~\ref{appendix:Communication_resources}, we illustrate the superiority of DynamicFL in terms of communication resources, as depicted in Tables~\ref{tab:cifar10d01communication_resource}--\ref{tab:femnistd03communication_resource}. We present the performance of CNN and ResNet-18, along with the standardized total server communication cost $\kappa_{g}^{t}$ for three communication intervals across five degrees of statistical heterogeneity for three datasets (CIFAR10, CIFAR100, FEMNIST). We underscore three \textbf{key findings}: 

\textbf{(1) The high-frequency group, similar to FedSGD, can consistently enhance global model performance by fully exploiting communication heterogeneity.}

\textbf{(2) There is a consistent pattern of performance boost when additional communication resources are deployed}, matching DynamicSGD's results with less than 60\% communication cost, and in some cases, only 15\%. In the CIFAR10 dataset, under the conditions of $Dir(0.1)$, \textit{Fix-0.9}, and 'a-g', the CNN architecture in DynamicAvg achieves an accuracy of 73.1\% compared to DynamicSGD's accuracy of 75.4\%, using only 58.9\% of the DynamicSGD's communication cost. Notably, this surpasses the baseline performance, which ranges in accuracy from 52.6\% to 66.2\% 
 (Table~\ref{tab:cifar10d01communication_resource}). 
 Similarly, under the conditions of $Dir(0.1)$, \textit{Dynamic-0.6}, and 'a-g', the ResNet-18 architecture in DynamicAvg achieves an accuracy of 91.6\%, compared to DynamicSGD's accuracy of 94.3\%, using only 56.2\% of the DynamicSGD's communication cost. This performance exceeds the baselines, which range from 73.9\% to 82.4\% in accuracy (Table~\ref{tab:cifar10d01communication_resource}).
 In the CIFAR100 dataset, under the conditions of $Dir(0.3)$, \textit{Fix-0.9}, and 'b-g', the CNN architecture in DynamicAvg achieves an accuracy of 40.9\%, compared to DynamicSGD's accuracy of 43.9\%, requiring only 22.3\% of the DynamicSGD's communication cost. This performance surpasses the baselines, which range in accuracy from 32.0\% to 38.4\% (Table~\ref{tab:cifar100d03communication_resource}).
 Finally, under the conditions of $K=2$, \textit{Fix-0.6}, and 'a-g', employing the CNN architecture, DynamicAvg achieves an accuracy of 27.5\%, compared to DynamicSGD's accuracy of 27\%, while requiring only 59.4\% of the DynamicSGD's communication cost. This result surpasses the baselines, which range in accuracy from 16.4\% to 23.8\% (Table~\ref{tab:cifar100k2communication_resource}).

\textbf{(3) The magnitude of performance improvement correlates closely with the degree of statistical heterogeneity in the data}, with settings of higher data heterogeneity typically generating more pronounced performance boosts. 
In the CIFAR10 dataset, under the conditions of $K=1$, \textit{Fix-0.9}, and 'b-g', the ResNet-18 architecture in DynamicAvg achieves an accuracy of 82.4\%, compared to DynamicSGD's accuracy of 94.5\%, while consuming only 16\% of the DynamicSGD's communication cost. This result surpasses the baselines, which range in accuracy from 11.8\% to 15.8\% (Table~\ref{tab:cifar10k1communication_resource}).
In the FEMNIST dataset, under the conditions of $Dir(0.01)$, \textit{Dynamic-0.6}, and 'a-g', the CNN architecture in DynamicAvg achieves an accuracy of 70.4\%, compared to DynamicSGD's accuracy of 75.4\%, while utilizing only 59\% of the DynamicSGD's communication cost. This performance outperforms the baselines, which range in accuracy from 31.1\% to 50.8\% (Table~\ref{tab:femnistd001communication_resource}).

In a comparison of Fix and Dynamic configurations, we see that they attain comparable performance. A trivial decrease in accuracy in a few instances of Fix might be resultant of bias, as it's only possible for a limited group of clients to uphold high-frequency communication. While DynaComm mitigates the bias, it continues to be an element to consider.

In Subsection~\ref{appendix:Communication_intervals}, \textbf{we further validate the robustness and flexibility of DynamicFL under a range of communication interval combinations}, as illustrated in Tables~\ref{tab:freq_ablation_cifar10_appendix}--\ref{tab:freq_ablation_femnist_appendix} and Figures~\ref{fig:freq_ablation_cifar10_d01}--\ref{fig:freq_ablation_femnist_d03}. We provide results contrasting various communication interval combinations, including the high-frequency group only, elongated intervals, and DynamicSGD, as well as the standardized server total communication cost $\kappa_{g}^{t}$, employing a CNN under the \textit{Fix-0.3} setup for the CIFAR10, CIFAR100, and FEMNIST datasets. We extract three central points: 

(1) A low-frequency group, mirroring FedAvg's strategy for efficient dataset iteration, can indeed augment model performance.

(2) A non-linear association between global model performance and communication cost is observed, suggesting that substantial performance improvements can be secured with marginal increases in communication costs. For example, in the CIFAR10 dataset with the $K=1$ setting, the 'd-e' condition yields an accuracy boost of over 17.6\% against the top baseline, using only 2\% of DynamicSGD's communication cost (Table~\ref{tab:freq_ablation_cifar10_appendix}). Similarly, in the FEMNIST dataset with the $Dir(0.01)$ setting, the 'd-e' condition results in an accuracy boost of over 11.8\% against the top baseline, using only 2\% of DynamicSGD's communication cost (Table~\ref{tab:freq_ablation_femnist_appendix}).

(3) The utilization of the shortest interval isn't invariably required. Moderately longer intervals, reminiscent of FedSGD, can also contribute positively to the global model, delivering performance on par with shorter intervals. For instance, in the CIFAR10 dataset with the $Dir(0.3)$ setting, 'c-e' achieves 73.5\% accuracy with a communication cost ratio of 0.028, similar to 'a-e' which has a communication cost ratio of 0.254 (Table~\ref{tab:freq_ablation_cifar10_appendix}). In the CIFAR100 dataset with the $K=2$ setting, 'c-g' achieves 20.4\% accuracy with a communication cost ratio of 0.022, comparable to 'a-e' which has a communication cost ratio of 0.316 (Table~\ref{tab:freq_ablation_cifar100_appendix}). In the CIFAR100 dataset with the $Dir(0.1)$ setting, 'b-e' achieves 38.2\% accuracy with a communication cost ratio of 0.09, comparable to 'a-b' which has a communication cost ratio of 0.482 (Table~\ref{tab:freq_ablation_cifar100_appendix}).

\vspace{-0.2cm}
\subsection{Communication resources}
\label{appendix:Communication_resources}
\vspace{-0.3cm}
\begin{table}[htbp]
\centering
\begin{small}
\caption{CNN and ResNet-18 accuracy results for three communication intervals across different configurations, $Dir(0.1)$ on CIFAR10 dataset.}
\vspace{-0.3cm}
\begin{tabular}{@{}ccccccccc@{}}
\toprule
\multicolumn{3}{c}{\multirow{2}{*}{Method}}                      & \multicolumn{2}{c}{CNN}        & \multicolumn{2}{c}{ResNet-18} & \multicolumn{2}{c}{$\kappa_{g}^{t}$}                                                           \\ \cmidrule(l){4-9} 
\multicolumn{3}{c}{}                                             & \textit{Fix} & \textit{Dynamic} & \textit{Fix} & \textit{Dynamic} & \textit{Fix}              & \textit{Dynamic}             \\ \midrule
\multicolumn{3}{c}{DynamicSGD}                                   & \multicolumn{2}{c}{75.4 (0.1)} & \multicolumn{2}{c}{94.3(0.1)} & \multicolumn{2}{c}{0.996}                                                                      \\ \midrule
\multirow{9}{*}{DynamicAvg} & \multirow{3}{*}{a-g} & $\beta=0.3$ & 66.2 (0.2)     & 68.9 (0.9)    & 84.0(0.0)     & 86.8(0.0)     & 0.221                                         & 0.300                                          \\
                            &                      & $\beta=0.6$ & 70.0 (0.0)     & 72.7 (0.9)    & 89.3(0.0)     & 91.6(0.0)     & 0.458                                         & 0.562                                          \\
                            &                      & $\beta=0.9$ & 73.1 (0.0)     & 73.1 (0.4)    & 92.3(0.0)     & 91.9(0.0)     & 0.589                                         & 0.655                                          \\ \cmidrule(l){2-9} 
                            & \multirow{3}{*}{b-g} & $\beta=0.3$ & 66.2 (0.2)     & 66.0 (1.1)    & 84.6(0.2)     & 86.1(0.0)     & 0.066                                         & 0.078                                          \\
                            &                      & $\beta=0.6$ & 69.0 (0.3)     & 70.5 (0.2)    & 87.4(0.3)     & 89.4(0.0)     & 0.116                                         & 0.144                                          \\
                            &                      & $\beta=0.9$ & 70.0 (0.7)     & 70.3 (0.9)    & 89.4(0.1)     & 89.9(0.1)     & 0.157                                         & 0.167                                          \\ \cmidrule(l){2-9} 
                            & \multirow{3}{*}{c-g} & $\beta=0.3$ & 64.4 (0.6)     & 66.7 (0.3)    & 82.7(1.1)     & 83.6(0.9)     & 0.018                                         & 0.022                                          \\
                            &                      & $\beta=0.6$ & 65.9 (0.3)     & 66.7 (0.4)    & 84.7(0.7)     & 86.6(0.2)     & 0.032                                         & 0.038                                          \\
                            &                      & $\beta=0.9$ & 67.1 (0.1)     & 66.9 (0.4)    & 86.2(0.5)     & 87.3(0.3)     & 0.042                                         & 0.044                                          \\ \midrule
\multicolumn{3}{c}{FedProx}                                      & \multicolumn{2}{c}{64.8(0.8)}  & \multicolumn{2}{c}{82.4(0.1)} & \multicolumn{2}{c}{0.004}                                                                      \\
\multicolumn{3}{c}{FedEnsemble}                                  & \multicolumn{2}{c}{65.3(0.3)}  & \multicolumn{2}{c}{81.1(1.3)} & \multicolumn{2}{c}{0.004}                                                                      \\
\multicolumn{3}{c}{SCAFFOLD}                                     & \multicolumn{2}{c}{52.6(1.8)}  & \multicolumn{2}{c}{73.9(3.6)} & \multicolumn{2}{c}{0.008}                                                                      \\
\multicolumn{3}{c}{FedGen}                                       & \multicolumn{2}{c}{62.4(1.8)}  & \multicolumn{2}{c}{78.4(0.0)} & \multicolumn{2}{c}{\begin{tabular}[c]{@{}c@{}}0.009(CNN)\\      0.004(ResNet-18)\end{tabular}} \\
\multicolumn{3}{c}{FedDyn}                                       & \multicolumn{2}{c}{66.2(0.0)}  & \multicolumn{2}{c}{80.2(0.9)} & \multicolumn{2}{c}{0.004}                                                                      \\
\multicolumn{3}{c}{FedNova}                                      & \multicolumn{2}{c}{63.9(0.6)}  & \multicolumn{2}{c}{80.1(1.5)} & \multicolumn{2}{c}{0.004}                                                                      \\
\multicolumn{3}{c}{FedAvg}                                       & \multicolumn{2}{c}{65.0(0.5)}  & \multicolumn{2}{c}{79.6(1.3)} & \multicolumn{2}{c}{0.004}                                                                      \\ \bottomrule
\end{tabular}
\label{tab:cifar10d01communication_resource}
\end{small}
\end{table}
\vspace{-0.5cm}
\begin{table}[htbp]
\centering
\caption{CNN and ResNet-18 accuracy results for three communication intervals across different configurations, $Dir(0.3)$ on CIFAR10 dataset.}
\begin{small}
\begin{tabular}{@{}ccccccccc@{}}
\toprule
\multicolumn{3}{c}{\multirow{2}{*}{Method}}                      & \multicolumn{2}{c}{CNN}        & \multicolumn{2}{c}{ResNet-18} & \multicolumn{2}{c}{$\kappa_{g}^{t}$}                                                           \\ \cmidrule(l){4-9} 
\multicolumn{3}{c}{}                                             & \textit{Fix} & \textit{Dynamic} & \textit{Fix} & \textit{Dynamic} & \textit{Fix}              & \textit{Dynamic}             \\ \midrule
\multicolumn{3}{c}{DynamicSGD}                                   & \multicolumn{2}{c}{77.3 (0.1)} & \multicolumn{2}{c}{94.6(0.1)} & \multicolumn{2}{c}{0.996}                                                                      \\ \midrule
\multirow{9}{*}{DynamicAvg} & \multirow{3}{*}{a-g} & $\beta=0.3$ & 71.2 (0.2)     & 74.3 (0.1)    & 89.9(0.0)     & 90.3(0.0)     & 0.245                                         & 0.299                                          \\
                            &                      & $\beta=0.6$ & 74.5 (0.0)     & 75.5 (0.4)    & 91.5(0.0)     & 92.2(0.0)     & 0.452                                         & 0.539                                          \\
                            &                      & $\beta=0.9$ & 75.4 (0.2)     & 75.4 (0.1)    & 92.0(0.0)     & 92.2(0.0)     & 0.596                                         & 0.636                                          \\ \cmidrule(l){2-9} 
                            & \multirow{3}{*}{b-g} & $\beta=0.3$ & 72.3 (0.2)     & 72.6 (0.1)    & 89.7(0.1)     & 89.9(0.1)     & 0.064                                         & 0.078                                          \\
                            &                      & $\beta=0.6$ & 73.7 (0.5)     & 74.4 (0.3)    & 90.5(0.0)     & 91.8(0.2)     & 0.109                                         & 0.137                                          \\
                            &                      & $\beta=0.9$ & 74.3 (0.4)     & 74.9 (0.3)    & 91.8(0.1)     & 92.0(0.1)     & 0.153                                         & 0.162                                          \\ \cmidrule(l){2-9} 
                            & \multirow{3}{*}{c-g} & $\beta=0.3$ & 71.4 (0.7)     & 72.2 (0.7)    & 89.4(0.3)     & 89.4(0.0)     & 0.019                                         & 0.022                                          \\
                            &                      & $\beta=0.6$ & 72.4 (0.5)     & 72.7 (0.8)    & 90.1(0.1)     & 90.7(0.2)     & 0.031                                         & 0.036                                          \\
                            &                      & $\beta=0.9$ & 73.6 (0.0)     & 73.5 (0.4)    & 90.6(0.1)     & 90.8(0.1)     & 0.039                                         & 0.042                                          \\ \midrule
\multicolumn{3}{c}{FedProx}                                      & \multicolumn{2}{c}{71.1(0.5)}  & \multicolumn{2}{c}{88.5(0.0)} & \multicolumn{2}{c}{0.004}                                                                      \\
\multicolumn{3}{c}{FedEnsemble}                                  & \multicolumn{2}{c}{71.7(0.2)}  & \multicolumn{2}{c}{88.6(0.0)} & \multicolumn{2}{c}{0.004}                                                                      \\
\multicolumn{3}{c}{SCAFFOLD}                                     & \multicolumn{2}{c}{71.0(1.0)}  & \multicolumn{2}{c}{89.0(0.8)} & \multicolumn{2}{c}{0.008}                                                                      \\
\multicolumn{3}{c}{FedGen}                                       & \multicolumn{2}{c}{70.4(0.7)}  & \multicolumn{2}{c}{83.6(0.0)} & \multicolumn{2}{c}{\begin{tabular}[c]{@{}c@{}}0.009(CNN)\\      0.004(ResNet-18)\end{tabular}} \\
\multicolumn{3}{c}{FedDyn}                                       & \multicolumn{2}{c}{70.5(0.0)}  & \multicolumn{2}{c}{88.9(0.0)} & \multicolumn{2}{c}{0.004}                                                                      \\
\multicolumn{3}{c}{FedNova}                                      & \multicolumn{2}{c}{71.5(0.2)}  & \multicolumn{2}{c}{86.4(0.2)} & \multicolumn{2}{c}{0.004}                                                                      \\
\multicolumn{3}{c}{FedAvg}                                       & \multicolumn{2}{c}{69.6(0.5)}  & \multicolumn{2}{c}{88.3(0.3)} & \multicolumn{2}{c}{0.004}                                                                      \\ \bottomrule
\end{tabular}
\end{small}
\label{tab:cifar10d03communication_resource}
\end{table}

\begin{table}[htbp]
\centering
\caption{CNN and ResNet-18 accuracy results for three communication intervals across different configurations, $K=1$ on CIFAR10 dataset.}
\begin{small}
\begin{tabular}{@{}ccccccccc@{}}
\toprule
\multicolumn{3}{c}{\multirow{2}{*}{Method}}                      & \multicolumn{2}{c}{CNN}       & \multicolumn{2}{c}{ResNet-18} & \multicolumn{2}{c}{$\kappa_{g}^{t}$}                                                           \\ \cmidrule(l){4-9} 
\multicolumn{3}{c}{}                                             & \textit{Fix} & \textit{Dynamic} & \textit{Fix} & \textit{Dynamic} & \textit{Fix}              & \textit{Dynamic}             \\ \midrule
\multicolumn{3}{c}{DynamicSGD}                                   & \multicolumn{2}{c}{73.9(1.0)} & \multicolumn{2}{c}{94.5(0.0)} & \multicolumn{2}{c}{0.996}                                                                      \\ \midrule
\multirow{9}{*}{DynamicAvg} & \multirow{3}{*}{a-g} & $\beta=0.3$ & 57.1 (0.3)    & 59.3 (0.3)    & 29.1(0.0)     & 32.4(0.0)     & 0.273                                         & 0.302                                          \\
                            &                      & $\beta=0.6$ & 66.3 (1.2)    & 66.6 (0.2)    & 42.3(0.0)     & 28.5(0.0)     & 0.484                                         & 0.585                                          \\
                            &                      & $\beta=0.9$ & 68.7 (0.3)    & 70.3 (0.4)    & 52.2(0.0)     & 53.3(0.0)     & 0.621                                         & 0.669                                          \\ \cmidrule(l){2-9} 
                            & \multirow{3}{*}{b-g} & $\beta=0.3$ & 40.3 (0.9)    & 42.3 (0.1)    & 22.7(1.5)     & 22.0(0.7)     & 0.066                                         & 0.078                                          \\
                            &                      & $\beta=0.6$ & 44.1 (1.3)    & 43.9 (0.6)    & 53.5(4.5)     & 77.7(2.7)     & 0.121                                         & 0.149                                          \\
                            &                      & $\beta=0.9$ & 59.1 (3.6)    & 46.9 (1.8)    & 82.4(0.1)     & 84.0(1.0)     & 0.160                                         & 0.170                                          \\ \cmidrule(l){2-9} 
                            & \multirow{3}{*}{c-g} & $\beta=0.3$ & 33.9 (0.6)    & 35.2 (0.6)    & 17.0(0.3)     & 17.1(0.2)     & 0.020                                         & 0.022                                          \\
                            &                      & $\beta=0.6$ & 38.1 (0.3)    & 38.9 (1.2)    & 17.9(1.0)     & 19.7(0.4)     & 0.032                                         & 0.039                                          \\
                            &                      & $\beta=0.9$ & 37.9 (1.3)    & 40.2 (0.5)    & 24.0(0.3)     & 24.9(0.6)     & 0.042                                         & 0.044                                          \\ \midrule
\multicolumn{3}{c}{FedProx}                                      & \multicolumn{2}{c}{20.9(0.3)} & \multicolumn{2}{c}{15.8(0.4)} & \multicolumn{2}{c}{0.004}                                                                      \\
\multicolumn{3}{c}{FedEnsemble}                                  & \multicolumn{2}{c}{19.9(0.7)} & \multicolumn{2}{c}{12.4(1.0)} & \multicolumn{2}{c}{0.004}                                                                      \\
\multicolumn{3}{c}{SCAFFOLD}                                     & \multicolumn{2}{c}{14.9(1.7)} & \multicolumn{2}{c}{12.7(0.0)} & \multicolumn{2}{c}{0.008}                                                                      \\
\multicolumn{3}{c}{FedGen}                                       & \multicolumn{2}{c}{14.5(0.6)} & \multicolumn{2}{c}{14.8(0.0)} & \multicolumn{2}{c}{\begin{tabular}[c]{@{}c@{}}0.009(CNN)\\      0.004(ResNet-18)\end{tabular}} \\
\multicolumn{3}{c}{FedDyn}                                       & \multicolumn{2}{c}{17.3(2.2)} & \multicolumn{2}{c}{15.2(1.2)} & \multicolumn{2}{c}{0.004}                                                                      \\
\multicolumn{3}{c}{FedNova}                                      & \multicolumn{2}{c}{20.2(0.8)} & \multicolumn{2}{c}{13.3(1.3)} & \multicolumn{2}{c}{0.004}                                                                      \\
\multicolumn{3}{c}{FedAvg}                                       & \multicolumn{2}{c}{19.1(1.4)} & \multicolumn{2}{c}{11.8(0.7)} & \multicolumn{2}{c}{0.004}                                                                      \\ \bottomrule
\end{tabular}
\end{small}
\label{tab:cifar10k1communication_resource}
\end{table}

\begin{table}[htbp]
\centering
\caption{CNN and ResNet-18 accuracy results for three communication intervals across different configurations, $K=2$ on CIFAR10 dataset.}
\begin{small}
\begin{tabular}{@{}ccccccccc@{}}
\toprule
\multicolumn{3}{c}{\multirow{2}{*}{Method}}                      & \multicolumn{2}{c}{CNN}                                           & \multicolumn{2}{c}{ResNet-18}                                     & \multicolumn{2}{c}{$\kappa_{g}^{t}$}                                                       \\ \cmidrule(l){4-9} 
\multicolumn{3}{c}{}                                             & \textit{Fix} & \textit{Dynamic} & \textit{Fix} & \textit{Dynamic} & \textit{Fix}              & \textit{Dynamic}             \\ \midrule
\multicolumn{3}{c}{DynamicSGD}                                   & \multicolumn{2}{c}{77.2(0.3)}                                     & \multicolumn{2}{c}{94.9(0.0)}                                     & \multicolumn{2}{c}{0.996}                                                                  \\ \midrule
\multirow{9}{*}{DynamicAvg} & \multirow{3}{*}{a-g} & $\beta=0.3$ & 70.1(0.2)                     & 70.9(0.0)                         & 87.5(0.0)                     & 89.5(0.0)                         & 0.264                                      & 0.302                                         \\
                            &                      & $\beta=0.6$ & 72.7(0.2)                     & 75.4(0.4)                         & 91.2(0.0)                     & 92.7(0.0)                         & 0.438                                      & 0.512                                         \\
                            &                      & $\beta=0.9$ & 74.6(0.1)                     & 75.7(0.0)                         & 92.5(0.0)                     & 93.0(0.0)                         & 0.539                                      & 0.567                                         \\ \cmidrule(l){2-9} 
                            & \multirow{3}{*}{b-g} & $\beta=0.3$ & 66.9(1.1)                     & 67.4(0.0)                         & 85.5(0.0)                     & 86.5(0.2)                         & 0.070                                      & 0.078                                         \\
                            &                      & $\beta=0.6$ & 68.8(0.1)                     & 68.4(0.6)                         & 88.9(0.0)                     & 90.5(0.0)                         & 0.112                                      & 0.131                                         \\
                            &                      & $\beta=0.9$ & 68.3(0.1)                     & 68.6(1.0)                         & 91.0(0.0)                     & 91.0(0.0)                         & 0.137                                      & 0.144                                         \\ \cmidrule(l){2-9} 
                            & \multirow{3}{*}{c-g} & $\beta=0.3$ & 66.9(1.6)                     & 66.2(0.8)                         & 83.6(0.3)                     & 84.2(0.0)                         & 0.018                                      & 0.022                                         \\
                            &                      & $\beta=0.6$ & 66.7(0.3)                     & 68.2(1.1)                         & 85.7(0.0)                     & 86.8(0.0)                         & 0.030                                      & 0.035                                         \\
                            &                      & $\beta=0.9$ & 66.3(1.1)                     & 68.2(0.7)                         & 87.0(0.0)                     & 86.7(0.5)                         & 0.036                                      & 0.038                                         \\ \midrule
\multicolumn{3}{c}{FedProx}                                      & \multicolumn{2}{c}{65.4(0.2)}                                     & \multicolumn{2}{c}{81.7(0.8)}                                     & \multicolumn{2}{c}{0.004}                                                                  \\
\multicolumn{3}{c}{FedEnsemble}                                  & \multicolumn{2}{c}{66.0(0.9)}                                     & \multicolumn{2}{c}{79.1(1.2)}                                     & \multicolumn{2}{c}{0.004}                                                                  \\
\multicolumn{3}{c}{SCAFFOLD}                                     & \multicolumn{2}{c}{66.2(0.0)}                                     & \multicolumn{2}{c}{82.1(0.3)}                                     & \multicolumn{2}{c}{0.008}                                                                  \\
\multicolumn{3}{c}{FedGen}                                       & \multicolumn{2}{c}{64.3(0.1)}                                     & \multicolumn{2}{c}{78.0(0.0)}                                     & \multicolumn{2}{c}{\begin{tabular}[c]{@{}c@{}}0.009(CNN) \\ 0.004(ResNet-18)\end{tabular}} \\
\multicolumn{3}{c}{FedDyn}                                       & \multicolumn{2}{c}{63.6(0.3)}                                     & \multicolumn{2}{c}{79.3(0.0)}                                     & \multicolumn{2}{c}{0.004}                                                                  \\
\multicolumn{3}{c}{FedNova}                                      & \multicolumn{2}{c}{66.5(0.6)}                                     & \multicolumn{2}{c}{81.2(0.1)}                                     & \multicolumn{2}{c}{0.004}                                                                  \\
\multicolumn{3}{c}{FedAvg}                                       & \multicolumn{2}{c}{66.3(0.0)}                                     & \multicolumn{2}{c}{80.8(0.1)}                                     & \multicolumn{2}{c}{0.004}                                                                  \\ \bottomrule
\end{tabular}
\end{small}
\label{tab:cifar10k2communication_resource}
\end{table}
\newpage
\vspace{+7cm}
\begin{table}[htbp]
\centering
\caption{CNN and ResNet-18 accuracy results for three communication intervals across different configurations, $Dir(0.1)$ on CIFAR100 dataset.}
\begin{small}
\begin{tabular}{@{}ccccccccc@{}}
\toprule
\multicolumn{3}{c}{\multirow{2}{*}{Method}}                      & \multicolumn{2}{c}{CNN}        & \multicolumn{2}{c}{ResNet-18} & \multicolumn{2}{c}{$\kappa_{g}^{t}$}                                                           \\ \cmidrule(l){4-9} 
\multicolumn{3}{c}{}                                             & \textit{Fix} & \textit{Dynamic} & \textit{Fix} & \textit{Dynamic} & \textit{Fix}              & \textit{Dynamic}             \\ \midrule
\multicolumn{3}{c}{DynamicSGD}                                   & \multicolumn{2}{c}{42.0 (0.4)} & \multicolumn{2}{c}{75.3(0.1)} & \multicolumn{2}{c}{0.996}                                                                      \\ \midrule
\multirow{9}{*}{DynamicAvg} & \multirow{3}{*}{a-g} & $\beta=0.3$ & 35.8 (0.2)     & 35.7 (0.4)    & 60.0(0.0)     & 59.8(0.0)     & 0.307                                         & 0.302                                          \\
                            &                      & $\beta=0.6$ & 39.8 (0.6)     & 39.4 (0.0)    & 67.0(0.0)     & 69.5(0.0)     & 0.591                                         & 0.599                                          \\
                            &                      & $\beta=0.9$ & 42.5 (0.4)     & 41.5 (0.2)    & 72.4(0.0)     & 72.3(0.0)     & 0.851                                         & 0.891                                          \\ \cmidrule(l){2-9} 
                            & \multirow{3}{*}{b-g} & $\beta=0.3$ & 33.7 (0.4)     & 35.4 (0.6)    & 57.9(1.0)     & 58.9(0.2)     & 0.082                                         & 0.078                                          \\
                            &                      & $\beta=0.6$ & 37.0 (0.1)     & 37.2 (0.7)    & 63.2(0.5)     & 64.9(0.4)     & 0.148                                         & 0.153                                          \\
                            &                      & $\beta=0.9$ & 38.0 (0.0)     & 38.0 (0.1)    & 69.0(0.1)     & 69.7(0.1)     & 0.224                                         & 0.226                                          \\ \cmidrule(l){2-9} 
                            & \multirow{3}{*}{c-g} & $\beta=0.3$ & 35.2 (0.2)     & 35.3 (0.0)    & 57.8(0.9)     & 57.3(0.6)     & 0.021                                         & 0.022                                          \\
                            &                      & $\beta=0.6$ & 34.9 (0.3)     & 35.3 (0.8)    & 62.9(0.1)     & 62.5(0.3)     & 0.040                                         & 0.040                                          \\
                            &                      & $\beta=0.9$ & 35.6 (0.5)     & 35.8 (0.5)    & 66.3(0.1)     & 67.1(0.2)     & 0.056                                         & 0.058                                          \\ \midrule
\multicolumn{3}{c}{FedProx}                                      & \multicolumn{2}{c}{33.5(0.3)}  & \multicolumn{2}{c}{55.6(0.3)} & \multicolumn{2}{c}{0.004}                                                                      \\
\multicolumn{3}{c}{FedEnsemble}                                  & \multicolumn{2}{c}{33.4(0.0)}  & \multicolumn{2}{c}{54.9(0.2)} & \multicolumn{2}{c}{0.004}                                                                      \\
\multicolumn{3}{c}{SCAFFOLD}                                     & \multicolumn{2}{c}{35.8(0.5)}  & \multicolumn{2}{c}{63.5(0.6)} & \multicolumn{2}{c}{0.008}                                                                      \\
\multicolumn{3}{c}{FedGen}                                       & \multicolumn{2}{c}{33.5(0.1)}  & \multicolumn{2}{c}{46.2(0.0)} & \multicolumn{2}{c}{\begin{tabular}[c]{@{}c@{}}0.011(CNN)\\      0.004(ResNet-18)\end{tabular}} \\
\multicolumn{3}{c}{FedDyn}                                       & \multicolumn{2}{c}{31.5(0.4)}  & \multicolumn{2}{c}{55.1(0.2)} & \multicolumn{2}{c}{0.004}                                                                      \\
\multicolumn{3}{c}{FedNova}                                      & \multicolumn{2}{c}{34.0(0.1)}  & \multicolumn{2}{c}{54.0(0.1)} & \multicolumn{2}{c}{0.004}                                                                      \\
\multicolumn{3}{c}{FedAvg}                                       & \multicolumn{2}{c}{33.0(0.2)}  & \multicolumn{2}{c}{54.0(0.3)} & \multicolumn{2}{c}{0.004}                                                                      \\ \bottomrule
\end{tabular}
\end{small}
\label{tab:cifar100d01communication_resource}
\end{table}

\begin{table}[htbp]
\centering
\caption{CNN and ResNet-18 accuracy results for three communication intervals across different configurations, $Dir(0.3)$ on CIFAR100 dataset.}
\begin{small}
\begin{tabular}{@{}ccccccccc@{}}
\toprule
\multicolumn{3}{c}{\multirow{2}{*}{Method}}                      & \multicolumn{2}{c}{CNN}        & \multicolumn{2}{c}{ResNet-18} & \multicolumn{2}{c}{$\kappa_{g}^{t}$}                                                           \\ \cmidrule(l){4-9} 
\multicolumn{3}{c}{}                                             & \textit{Fix} & \textit{Dynamic} & \textit{Fix} & \textit{Dynamic} & \textit{Fix}              & \textit{Dynamic}             \\ \midrule
\multicolumn{3}{c}{DynamicSGD}                                   & \multicolumn{2}{c}{43.9 (0.0)} & \multicolumn{2}{c}{75.4(0.1)} & \multicolumn{2}{c}{0.996}                                                                      \\ \midrule
\multirow{9}{*}{DynamicAvg} & \multirow{3}{*}{a-g} & $\beta=0.3$ & 37.2 (0.1)     & 37.0 (0.1)    & 62.6(0.0)     & 63.1(0.0)     & 0.325                                         & 0.302                                          \\
                            &                      & $\beta=0.6$ & 40.2 (0.4)     & 40.4 (0.2)    & 68.4(0.0)     & 68.8(0.0)     & 0.615                                         & 0.599                                          \\
                            &                      & $\beta=0.9$ & 42.8 (0.1)     & 44.2 (0.2)    & 71.9(0.0)     & 72.9(0.0)     & 0.865                                         & 0.893                                          \\ \cmidrule(l){2-9} 
                            & \multirow{3}{*}{b-g} & $\beta=0.3$ & 35.6 (0.2)     & 36.5 (0.5)    & 61.6(0.8)     & 62.5(0.4)     & 0.083                                         & 0.078                                          \\
                            &                      & $\beta=0.6$ & 39.5 (0.1)     & 40.8 (0.1)    & 64.5(0.0)     & 67.1(0.1)     & 0.157                                         & 0.153                                          \\
                            &                      & $\beta=0.9$ & 40.9 (0.6)     & 41.4 (0.0)    & 69.9(0.2)     & 70.5(0.0)     & 0.223                                         & 0.226                                          \\ \cmidrule(l){2-9} 
                            & \multirow{3}{*}{c-g} & $\beta=0.3$ & 35.9 (0.5)     & 36.6 (0.2)    & 61.4(0.6)     & 63.1(0.0)     & 0.021                                         & 0.022                                          \\
                            &                      & $\beta=0.6$ & 36.8 (0.1)     & 38.2 (0.4)    & 63.8(0.3)     & 65.1(0.5)     & 0.039                                         & 0.040                                          \\
                            &                      & $\beta=0.9$ & 39.4 (0.2)     & 39.4 (0.1)    & 69.4(0.0)     & 68.9(0.2)     & 0.057                                         & 0.058                                          \\ \midrule
\multicolumn{3}{c}{FedProx}                                      & \multicolumn{2}{c}{34.7(0.5)}  & \multicolumn{2}{c}{58.7(0.6)} & \multicolumn{2}{c}{0.004}                                                                      \\
\multicolumn{3}{c}{FedEnsemble}                                  & \multicolumn{2}{c}{33.0(0.3)}  & \multicolumn{2}{c}{58.9(0.5)} & \multicolumn{2}{c}{0.004}                                                                      \\
\multicolumn{3}{c}{SCAFFOLD}                                     & \multicolumn{2}{c}{38.4(0.7)}  & \multicolumn{2}{c}{66.0(0.0)} & \multicolumn{2}{c}{0.008}                                                                      \\
\multicolumn{3}{c}{FedGen}                                       & \multicolumn{2}{c}{34.6(0.8)}  & \multicolumn{2}{c}{49.3(0.0)} & \multicolumn{2}{c}{\begin{tabular}[c]{@{}c@{}}0.011(CNN)\\      0.004(ResNet-18)\end{tabular}} \\
\multicolumn{3}{c}{FedDyn}                                       & \multicolumn{2}{c}{32.0(0.0)}  & \multicolumn{2}{c}{60.7(0.0)} & \multicolumn{2}{c}{0.004}                                                                      \\
\multicolumn{3}{c}{FedNova}                                      & \multicolumn{2}{c}{34.7(0.1)}  & \multicolumn{2}{c}{58.9(0.5)} & \multicolumn{2}{c}{0.004}                                                                      \\
\multicolumn{3}{c}{FedAvg}                                       & \multicolumn{2}{c}{33.6(0.7)}  & \multicolumn{2}{c}{58.3(0.2)} & \multicolumn{2}{c}{0.004}                                                                      \\ \bottomrule
\end{tabular}
\end{small}
\label{tab:cifar100d03communication_resource}
\end{table}

\begin{table}[htbp]
\centering
\caption{CNN and ResNet-18 accuracy results for three communication intervals across different configurations, $K=1$ on CIFAR100 dataset.}
\begin{small}
\begin{tabular}{@{}ccccccccc@{}}
\toprule
\multicolumn{3}{c}{\multirow{2}{*}{Method}}                      & \multicolumn{2}{c}{CNN}      & \multicolumn{2}{c}{ResNet-18} & \multicolumn{2}{c}{$\kappa_{g}^{t}$}                                                           \\ \cmidrule(l){4-9} 
\multicolumn{3}{c}{}                                             & \textit{Fix} & \textit{Dynamic} & \textit{Fix} & \textit{Dynamic} & \textit{Fix}              & \textit{Dynamic}             \\ \midrule
\multicolumn{3}{c}{DynamicSGD}                                   & \multicolumn{2}{c}{6.9(0.2)} & \multicolumn{2}{c}{7.1(0.1)}  & \multicolumn{2}{c}{0.996}                                                                      \\ \midrule
\multirow{9}{*}{DynamicAvg} & \multirow{3}{*}{a-g} & $\beta=0.3$ & 9.6 (0.4)     & 8.0 (0.4)    & 3.1(0.0)      & 2.9(0.0)      & 0.311                                         & 0.302                                          \\
                            &                      & $\beta=0.6$ & 15.1 (0.8)    & 12.1 (0.4)   & 4.8(0.0)      & 5.1(0.0)      & 0.583                                         & 0.599                                          \\
                            &                      & $\beta=0.9$ & 13.1 (0.4)    & 10.6 (0.6)   & 5.2(0.0)      & 5.4(0.0)      & 0.899                                         & 0.897                                          \\ \cmidrule(l){2-9} 
                            & \multirow{3}{*}{b-g} & $\beta=0.3$ & 4.8 (0.3)     & 4.4 (0.0)    & 3.8(0.3)      & 2.9(0.1)      & 0.079                                         & 0.078                                          \\
                            &                      & $\beta=0.6$ & 6.2 (0.1)     & 6.4 (0.3)    & 6.5(0.1)      & 3.9(0.0)      & 0.154                                         & 0.153                                          \\
                            &                      & $\beta=0.9$ & 7.3 (0.0)     & 7.2 (0.2)    & 7.2(0.6)      & 5.0(0.0)      & 0.228                                         & 0.227                                          \\ \cmidrule(l){2-9} 
                            & \multirow{3}{*}{c-g} & $\beta=0.3$ & 3.6 (0.2)     & 4.2 (0.0)    & 1.6(0.1)      & 1.7(0.1)      & 0.023                                         & 0.022                                          \\
                            &                      & $\beta=0.6$ & 4.8 (0.2)     & 4.5 (0.0)    & 3.5(0.7)      & 3.4(0.5)      & 0.039                                         & 0.040                                          \\
                            &                      & $\beta=0.9$ & 4.7 (0.2)     & 4.9 (0.3)    & 3.9(0.6)      & 4.7(0.3)      & 0.057                                         & 0.058                                          \\ \midrule
\multicolumn{3}{c}{FedProx}                                      & \multicolumn{2}{c}{1.4(0.1)} & \multicolumn{2}{c}{1.2(0.1)}  & \multicolumn{2}{c}{0.004}                                                                      \\
\multicolumn{3}{c}{FedEnsemble}                                  & \multicolumn{2}{c}{1.5(0.1)} & \multicolumn{2}{c}{1.3(0.1)}  & \multicolumn{2}{c}{0.004}                                                                      \\
\multicolumn{3}{c}{SCAFFOLD}                                     & \multicolumn{2}{c}{1.3(0.1)} & \multicolumn{2}{c}{1.2(0.1)}  & \multicolumn{2}{c}{0.008}                                                                      \\
\multicolumn{3}{c}{FedGen}                                       & \multicolumn{2}{c}{1.2(0.0)} & \multicolumn{2}{c}{1.4(0.0)}  & \multicolumn{2}{c}{\begin{tabular}[c]{@{}c@{}}0.011(CNN)\\      0.004(ResNet-18)\end{tabular}} \\
\multicolumn{3}{c}{FedDyn}                                       & \multicolumn{2}{c}{1.3(0.1)} & \multicolumn{2}{c}{1.4(0.0)}  & \multicolumn{2}{c}{0.004}                                                                      \\
\multicolumn{3}{c}{FedNova}                                      & \multicolumn{2}{c}{1.3(0.1)} & \multicolumn{2}{c}{1.4(0.1)}  & \multicolumn{2}{c}{0.004}                                                                      \\
\multicolumn{3}{c}{FedAvg}                                       & \multicolumn{2}{c}{1.3(0.1)} & \multicolumn{2}{c}{1.2(0.0)}  & \multicolumn{2}{c}{0.004}                                                                      \\ \bottomrule
\end{tabular}
\end{small}
\label{tab:cifar100k1communication_resource}
\end{table}

\begin{table}[htbp]
\centering
\caption{CNN and ResNet-18 accuracy results for three communication intervals across different configurations, $K=2$ on CIFAR100 dataset.}
\begin{small}
\begin{tabular}{@{}ccccccccc@{}}
\toprule
\multicolumn{3}{c}{\multirow{2}{*}{Method}}                      & \multicolumn{2}{c}{CNN}       & \multicolumn{2}{c}{ResNet-18} & \multicolumn{2}{c}{$\kappa_{g}^{t}$}                                                           \\ \cmidrule(l){4-9} 
\multicolumn{3}{c}{}                                             & \textit{Fix} & \textit{Dynamic} & \textit{Fix} & \textit{Dynamic} & \textit{Fix}              & \textit{Dynamic}             \\ \midrule
\multicolumn{3}{c}{DynamicSGD}                                   & \multicolumn{2}{c}{27.0(0.5)} & \multicolumn{2}{c}{71.1(0.1)} & \multicolumn{2}{c}{0.996}                                                                      \\ \midrule
\multirow{9}{*}{DynamicAvg} & \multirow{3}{*}{a-g} & $\beta=0.3$ & 25.5 (0.0)    & 24.9 (0.1)    & 34.9(0.0)     & 36.1(0.0)     & 0.308                                         & 0.302                                          \\
                            &                      & $\beta=0.6$ & 27.5 (1.2)    & 25.4 (0.3)    & 53.3(0.0)     & 53.1(0.0)     & 0.594                                         & 0.599                                          \\
                            &                      & $\beta=0.9$ & 28.9 (0.2)    & 29.4 (0.1)    & 66.6(0.0)     & 63.2(0.0)     & 0.905                                         & 0.897                                          \\ \cmidrule(l){2-9} 
                            & \multirow{3}{*}{b-g} & $\beta=0.3$ & 21.8 (0.2)    & 21.7 (0.2)    & 26.9(0.2)     & 26.0(0.5)     & 0.083                                         & 0.078                                          \\
                            &                      & $\beta=0.6$ & 20.8 (0.2)    & 17.2 (0.1)    & 42.2(0.2)     & 44.7(0.3)     & 0.150                                         & 0.153                                          \\
                            &                      & $\beta=0.9$ & 16.5 (1.1)    & 14.0 (0.1)    & 51.6(0.4)     & 50.1(0.5)     & 0.226                                         & 0.227                                          \\ \cmidrule(l){2-9} 
                            & \multirow{3}{*}{c-g} & $\beta=0.3$ & 20.6 (0.2)    & 19.7 (0.1)    & 24.5(0.5)     & 24.2(0.0)     & 0.022                                         & 0.022                                          \\
                            &                      & $\beta=0.6$ & 21.3 (0.1)    & 17.5 (0.1)    & 31.5(1.5)     & 33.7(1.0)     & 0.041                                         & 0.040                                          \\
                            &                      & $\beta=0.9$ & 15.9 (0.3)    & 14.3 (0.2)    & 38.3(0.4)     & 38.8(1.5)     & 0.058                                         & 0.058                                          \\ \midrule
\multicolumn{3}{c}{FedProx}                                      & \multicolumn{2}{c}{18.1(0.3)} & \multicolumn{2}{c}{21.5(1.0)} & \multicolumn{2}{c}{0.004}                                                                      \\
\multicolumn{3}{c}{FedEnsemble}                                  & \multicolumn{2}{c}{18.0(0.3)} & \multicolumn{2}{c}{18.2(0.1)} & \multicolumn{2}{c}{0.004}                                                                      \\
\multicolumn{3}{c}{SCAFFOLD}                                     & \multicolumn{2}{c}{23.8(0.1)} & \multicolumn{2}{c}{15.7(3.7)} & \multicolumn{2}{c}{0.008}                                                                      \\
\multicolumn{3}{c}{FedGen}                                       & \multicolumn{2}{c}{17.3(0.3)} & \multicolumn{2}{c}{19.6(0.4)} & \multicolumn{2}{c}{\begin{tabular}[c]{@{}c@{}}0.011(CNN)\\      0.004(ResNet-18)\end{tabular}} \\
\multicolumn{3}{c}{FedDyn}                                       & \multicolumn{2}{c}{16.4(0.0)} & \multicolumn{2}{c}{15.3(0.3)} & \multicolumn{2}{c}{0.004}                                                                      \\
\multicolumn{3}{c}{FedNova}                                      & \multicolumn{2}{c}{17.7(0.6)} & \multicolumn{2}{c}{20.5(0.5)} & \multicolumn{2}{c}{0.004}                                                                      \\
\multicolumn{3}{c}{FedAvg}                                       & \multicolumn{2}{c}{18.4(0.6)} & \multicolumn{2}{c}{20.4(0.4)} & \multicolumn{2}{c}{0.004}                                                                      \\ \bottomrule
\end{tabular}
\end{small}
\label{tab:cifar100k2communication_resource}
\vspace{-0.4cm}
\end{table}
\vspace{+10cm}
\begin{table}[htbp]
\centering
\caption{CNN accuracy results for three communication intervals across different configurations, $Dir(0.01)$ on FEMNIST dataset.}
\begin{small}
\begin{tabular}{@{}ccccccc@{}}
\toprule
\multicolumn{3}{c}{\multirow{2}{*}{Method}}                      & \multicolumn{2}{c}{CNN}        & \multicolumn{2}{c}{$\kappa_{g}^{t}$}                                                           \\ \cmidrule(l){4-7} 
\multicolumn{3}{c}{}                                             & \textit{Fix}            & \textit{Dynamic}       & \textit{Fix}                               & \textit{Dynamic}                             \\ \midrule
\multicolumn{3}{c}{DynamicSGD}                                   & \multicolumn{2}{c}{75.4 (1.8)} & \multicolumn{2}{c}{0.996}                                                                      \\ \midrule
\multirow{9}{*}{DynamicAvg} & \multirow{3}{*}{a-g} & $\beta=0.3$ & 60.3(1.7)      & 64.4(0.3)     & 0.281                               & 0.299                                                    \\
                            &                      & $\beta=0.6$ & 67.0(1.7)      & 70.4(0.3)     & 0.571                               & 0.590                                                    \\
                            &                      & $\beta=0.9$ & 70.3(0.0)      & 73.3(0.2)     & 0.817                               & 0.865                                                    \\ \cmidrule(l){2-7} 
                            & \multirow{3}{*}{b-g} & $\beta=0.3$ & 60.6(2.5)      & 63.3(0.2)     & 0.077                               & 0.077                                                    \\
                            &                      & $\beta=0.6$ & 60.4(1.2)      & 61.3(1.5)     & 0.142                               & 0.149                                                    \\
                            &                      & $\beta=0.9$ & 59.6(1.3)      & 59.9(2.5)     & 0.201                               & 0.217                                                    \\ \cmidrule(l){2-7} 
                            & \multirow{3}{*}{c-g} & $\beta=0.3$ & 60.9(0.5)      & 60.7(0.5)     & 0.020                               & 0.021                                                    \\
                            &                      & $\beta=0.6$ & 60.8(3.0)      & 61.1(0.4)     & 0.035                               & 0.039                                                    \\
                            &                      & $\beta=0.9$ & 57.3(1.3)      & 59.2(1.3)     & 0.053                               & 0.055                                                    \\ \midrule
\multicolumn{3}{c}{FedProx}                                      & \multicolumn{2}{c}{50.8(0.0)}  & \multicolumn{2}{c}{0.004}                                                                      \\
\multicolumn{3}{c}{FedEnsemble}                                  & \multicolumn{2}{c}{31.1(0.0)}  & \multicolumn{2}{c}{0.004}                                                                      \\
\multicolumn{3}{c}{SCAFFOLD}                                     & \multicolumn{2}{c}{N/A}   & \multicolumn{2}{c}{0.008}                                                                      \\
\multicolumn{3}{c}{FedGen}                                       & \multicolumn{2}{c}{N/A}        & \multicolumn{2}{c}{0.010} \\
\multicolumn{3}{c}{FedDyn}                                       & \multicolumn{2}{c}{N/A}        & \multicolumn{2}{c}{0.004}                                                                      \\
\multicolumn{3}{c}{FedNova}                                      & \multicolumn{2}{c}{45.5(0.0)}  & \multicolumn{2}{c}{0.004}                                                                      \\
\multicolumn{3}{c}{FedAvg}                                       & \multicolumn{2}{c}{37.7(0.8)}  & \multicolumn{2}{c}{0.004}                                                                      \\ \bottomrule
\end{tabular}
\end{small}
\label{tab:femnistd001communication_resource}
\end{table}

\begin{table}[htbp]
\centering
\caption{CNN accuracy results for three communication intervals across different configurations, $Dir(0.1)$ on FEMNIST dataset.}
\begin{small}
\begin{tabular}{@{}ccccccc@{}}
\toprule
\multicolumn{3}{c}{\multirow{2}{*}{Method}}                      & \multicolumn{2}{c}{CNN}        & \multicolumn{2}{c}{$\kappa_{g}^{t}$}                                                           \\ \cmidrule(l){4-7} 
\multicolumn{3}{c}{}                                             & \textit{Fix}            & \textit{Dynamic}       & \textit{Fix}                               & \textit{Dynamic}                             \\ \midrule
\multicolumn{3}{c}{DynamicSGD}                                   & \multicolumn{2}{c}{85.7 (0.0)} & \multicolumn{2}{c}{0.996}                                                                      \\ \midrule
\multirow{9}{*}{DynamicAvg} & \multirow{3}{*}{a-g} & $\beta=0.3$ & 84.8 (0.0)     & 85.4 (0.1)    & 0.268                               & 0.303                                                    \\
                            &                      & $\beta=0.6$ & 85.0 (0.1)     & 85.4 (0.2)    & 0.538                               & 0.599                                                    \\
                            &                      & $\beta=0.9$ & 86.0 (0.1)     & 86.2 (0.0)    & 0.773                               & 0.819                                                    \\ \cmidrule(l){2-7} 
                            & \multirow{3}{*}{b-g} & $\beta=0.3$ & 84.9 (0.1)     & 84.8 (0.1)    & 0.074                               & 0.078                                                    \\
                            &                      & $\beta=0.6$ & 84.5 (0.2)     & 84.7 (0.1)    & 0.137                               & 0.151                                                    \\
                            &                      & $\beta=0.9$ & 85.3 (0.0)     & 85.1 (0.2)    & 0.195                               & 0.205                                                    \\ \cmidrule(l){2-7} 
                            & \multirow{3}{*}{c-g} & $\beta=0.3$ & 84.5 (0.1)     & 84.6 (0.0)    & 0.020                               & 0.022                                                    \\
                            &                      & $\beta=0.6$ & 85.1 (0.3)     & 84.9 (0.0)    & 0.035                               & 0.039                                                    \\
                            &                      & $\beta=0.9$ & 84.9 (0.0)     & 85.3 (0.1)    & 0.051                               & 0.052                                                    \\ \midrule
\multicolumn{3}{c}{FedProx}                                      & \multicolumn{2}{c}{84.1(0.5)}  & \multicolumn{2}{c}{0.004}                                                                      \\
\multicolumn{3}{c}{FedEnsemble}                                  & \multicolumn{2}{c}{83.6(0.2)}  & \multicolumn{2}{c}{0.004}                                                                      \\
\multicolumn{3}{c}{SCAFFOLD}                                     & \multicolumn{2}{c}{N/A}        & \multicolumn{2}{c}{0.008}                                                                      \\
\multicolumn{3}{c}{FedGen}                                       & \multicolumn{2}{c}{N/A}        & \multicolumn{2}{c}{0.010} \\
\multicolumn{3}{c}{FedDyn}                                       & \multicolumn{2}{c}{N/A}        & \multicolumn{2}{c}{0.004}                                                                      \\
\multicolumn{3}{c}{FedNova}                                      & \multicolumn{2}{c}{83.9(0.0)}  & \multicolumn{2}{c}{0.004}                                                                      \\
\multicolumn{3}{c}{FedAvg}                                       & \multicolumn{2}{c}{83.8(0.3)}  & \multicolumn{2}{c}{0.004}                                                                      \\ \bottomrule
\end{tabular}
\end{small}
\label{tab:femnistd01communication_resource}
\end{table}

\begin{table}[htbp]
\centering
\caption{CNN accuracy results for three communication intervals across different configurations, $Dir(0.3)$ on FEMNIST dataset.}
\begin{small}
\begin{tabular}{@{}ccccccc@{}}
\toprule
\multicolumn{3}{c}{\multirow{2}{*}{Method}}                      & \multicolumn{2}{c}{CNN}        & \multicolumn{2}{c}{$\kappa_{g}^{t}$}                                                           \\ \cmidrule(l){4-7} 
\multicolumn{3}{c}{}                                             & \textit{Fix}            & \textit{Dynamic}       & \textit{Fix}                               & \textit{Dynamic}                             \\ \midrule
\multicolumn{3}{c}{DynamicSGD}                                   & \multicolumn{2}{c}{87.6 (0.0)} & \multicolumn{2}{c}{0.996}                                                                      \\ \midrule
\multirow{9}{*}{DynamicAvg} & \multirow{3}{*}{a-g} & $\beta=0.3$ & 86.7 (0.0)     & 86.7 (0.0)    & 0.297                                         & 0.303                                          \\
                            &                      & $\beta=0.6$ & 87.1 (0.0)     & 87.2 (0.1)    & 0.530                                         & 0.599                                          \\
                            &                      & $\beta=0.9$ & 87.3 (0.0)     & 87.4 (0.0)    & 0.762                                         & 0.813                                          \\ \cmidrule(l){2-7} 
                            & \multirow{3}{*}{b-g} & $\beta=0.3$ & 86.7 (0.0)     & 86.7 (0.0)    & 0.077                                         & 0.078                                          \\
                            &                      & $\beta=0.6$ & 86.8 (0.0)     & 87.0 (0.0)    & 0.144                                         & 0.151                                          \\
                            &                      & $\beta=0.9$ & 87.0 (0.0)     & 86.9 (0.0)    & 0.194                                         & 0.204                                          \\ \cmidrule(l){2-7} 
                            & \multirow{3}{*}{c-g} & $\beta=0.3$ & 86.5 (0.1)     & 86.5 (0.0)    & 0.020                                         & 0.022                                          \\
                            &                      & $\beta=0.6$ & 86.7 (0.1)     & 86.7 (0.0)    & 0.038                                         & 0.039                                          \\
                            &                      & $\beta=0.9$ & 86.9 (0.1)     & 86.7 (0.0)    & 0.049                                         & 0.052                                          \\ \midrule
\multicolumn{3}{c}{FedProx}                                      & \multicolumn{2}{c}{85.1(0.1)}  & \multicolumn{2}{c}{0.004}                                                                      \\
\multicolumn{3}{c}{FedEnsemble}                                  & \multicolumn{2}{c}{83.8(0.0)}  & \multicolumn{2}{c}{0.004}                                                                      \\
\multicolumn{3}{c}{SCAFFOLD}                                     & \multicolumn{2}{c}{N/A}        & \multicolumn{2}{c}{0.008}                                                                      \\
\multicolumn{3}{c}{FedGen}                                       & \multicolumn{2}{c}{N/A}        & \multicolumn{2}{c}{0.010} \\
\multicolumn{3}{c}{FedDyn}                                       & \multicolumn{2}{c}{N/A}        & \multicolumn{2}{c}{0.004}                                                                      \\
\multicolumn{3}{c}{FedNova}                                      & \multicolumn{2}{c}{85.7(0.0)}  & \multicolumn{2}{c}{0.004}                                                                      \\
\multicolumn{3}{c}{FedAvg}                                       & \multicolumn{2}{c}{85.5(0.3)}  & \multicolumn{2}{c}{0.004}                                                                      \\ \bottomrule
\end{tabular}
\end{small}
\label{tab:femnistd03communication_resource}
\end{table}

\clearpage
\subsection{Communication intervals}
\label{appendix:Communication_intervals}


\begin{table}[htbp]
\caption{Accuracy results comparing various communication interval combinations, and the standardized server total communication cost $\kappa_{g}^{t}$, using CNN under the CIFAR-10, \textit{Fix-0.3} configuration. 'a' signifies that all active clients train for 'a' intervals, while 'a*' indicates training exclusively involving high-frequency group clients. The same pattern is replicated for intervals b, c, and d.}
\vspace{-0.2cm}
\centering
\begin{small}
\begin{tabular}{@{}ccccccccc@{}}
\toprule
\multirow{2}{*}{\begin{tabular}[c]{@{}c@{}}Communication   \\ interval\end{tabular}} & \multicolumn{2}{c}{$Dir(0.1)$} & \multicolumn{2}{c}{$Dir(0.3)$} & \multicolumn{2}{c}{$K=1$} & \multicolumn{2}{c}{$K=2$} \\ \cmidrule(l){2-9} 
                                                                                     & Accuracy     & $\kappa_{g}$    & Accuracy     & $\kappa_{g}$    & Accuracy   & $\kappa_{g}$ & Accuracy   & $\kappa_{g}$ \\ \midrule
a                                                                                    & 75.4(0.1)    & 0.996      & 77.3(0.1)    & 0.996      & 73.9(1.0)  & 0.996   & 77.2(0.3)  & 0.996   \\
a-b                                                                                  & 72.7(0.5)    & 0.418      & 75.6(0.3)    & 0.436      & 65.1(0.0)  & 0.457   & 74.2(0.3)  & 0.450   \\
a-c                                                                                  & 70.4(0.6)    & 0.268      & 74.9(0.1)    & 0.290      & 64.8(0.7)  & 0.317   & 72.5(0.5)  & 0.308   \\
a-d                                                                                  & 69.7(0.4)    & 0.243      & 74.8(0.1)    & 0.266      & 61.5(0.5)  & 0.293   & 70.9(0.5)  & 0.285   \\
a-e                                                                                  & 68.5(0.1)    & 0.231      & 73.8(0.1)    & 0.254      & 63.2(0.7)  & 0.282   & 70.1(1.0)  & 0.273   \\
a-f                                                                                  & 67.5(0.1)    & 0.224      & 72.9(0.0)    & 0.248      & 60.1(1.1)  & 0.276   & 70.1(0.5)  & 0.267   \\
a-g                                                                                  & 66.9(0.3)    & 0.221      & 71.3(0.2)    & 0.245      & 57.2(0.2)  & 0.273   & 69.9(0.1)  & 0.264   \\
a*                                                                                   & 63.5(0.3)    & 0.218      & 69.4(0.3)    & 0.242      & 59.8(0.4)  & 0.270   & 66.5(0.1)  & 0.261   \\ \midrule
b                                                                                    & 69.7(0.0)    & 0.252      & 75.7(0.0)    & 0.252      & 65.0(0.0)  & 0.252   & 72.7(0.4)  & 0.252   \\
b-c                                                                                  & 69.7(0.7)    & 0.111      & 74.5(0.0)    & 0.110      & 55.4(0.3)  & 0.111   & 71.4(0.0)  & 0.114   \\
b-d                                                                                  & 67.8(0.1)    & 0.087      & 73.5(0.5)    & 0.085      & 52.9(1.3)  & 0.087   & 68.2(0.2)  & 0.091   \\
b-e                                                                                  & 67.6(0.4)    & 0.075      & 73.3(0.2)    & 0.073      & 51.9(0.6)  & 0.075   & 68.2(0.1)  & 0.079   \\
b-f                                                                                  & 66.9(0.8)    & 0.069      & 72.1(0.5)    & 0.067      & 44.5(1.1)  & 0.069   & 68.1(0.8)  & 0.073   \\
b-g                                                                                  & 65.8(0.2)    & 0.066      & 72.2(0.2)    & 0.064      & 39.6(0.4)  & 0.066   & 67.2(1.3)  & 0.070   \\
b*                                                                                   & 57.9(1.5)    & 0.063      & 65.9(1.0)    & 0.061      & 49.8(0.8)  & 0.062   & 60.9(0.6)  & 0.067   \\ \midrule
c                                                                                    & 70.3(0.4)    & 0.064      & 73.8(0.6)    & 0.064      & 54.4(0.1)  & 0.064   & 70.6(0.0)  & 0.064   \\
c-d                                                                                  & 68.8(0.1)    & 0.040      & 73.8(0.0)    & 0.040      & 50.7(0.1)  & 0.040   & 68.3(0.5)  & 0.040   \\
c-e                                                                                  & 67.1(0.4)    & 0.027      & 73.5(0.2)    & 0.028      & 45.0(0.4)  & 0.029   & 66.1(0.2)  & 0.028   \\
c-f                                                                                  & 65.7(0.1)    & 0.021      & 71.7(0.6)    & 0.022      & 41.9(0.0)  & 0.023   & 66.6(1.0)  & 0.021   \\
c-g                                                                                  & 65.0(0.1)    & 0.018      & 72.0(0.3)    & 0.019      & 32.6(0.3)  & 0.020   & 67.9(0.9)  & 0.018   \\
c*                                                                                   & 55.5(2.1)    & 0.015      & 64.0(0.6)    & 0.016      & 40.3(1.3)  & 0.017   & 59.6(2.5)  & 0.015   \\ \midrule
d                                                                                    & 67.6(0.3)    & 0.032      & 73.2(0.2)    & 0.032      & 49.5(0.5)  & 0.032   & 67.8(0.0)  & 0.032   \\
d-e                                                                                  & 66.0(0.4)    & 0.020      & 72.9(0.1)    & 0.020      & 42.3(0.4)  & 0.020   & 65.6(0.4)  & 0.020   \\
d-f                                                                                  & 65.3(0.0)    & 0.014      & 72.0(0.4)    & 0.014      & 40.4(1.9)  & 0.014   & 66.5(0.7)  & 0.014   \\
d-g                                                                                  & 65.4(0.3)    & 0.011      & 71.5(0.1)    & 0.011      & 31.7(1.2)  & 0.011   & 64.9(0.6)  & 0.011   \\
d*                                                                                   & 54.7(0.3)    & 0.008      & 62.8(0.4)    & 0.008      & 38.5(1.0)  & 0.009   & 57.9(0.3)  & 0.008   \\ \midrule
FedProx                                                                              & 64.8(0.8)    & 0.004      & 71.1(0.5)    & 0.004      & 20.9(0.3)  & 0.004   & 65.4(0.2)  & 0.004   \\
FedEnsemble                                                                          & 65.3(0.3)    & 0.004      & 71.7(0.2)    & 0.004      & 19.9(0.7)  & 0.004   & 66.0(0.9)  & 0.004   \\
SCAFFOLD                                                                             & 52.6(1.8)    & 0.008      & 71.0(1.0)    & 0.008      & 14.9(1.7)  & 0.008   & 66.2(0.0)  & 0.008   \\
FedGen                                                                               & 62.4(1.8)    & 0.009      & 70.4(0.7)    & 0.009      & 14.5(0.6)  & 0.009   & 64.3(0.1)  & 0.009   \\
FedDyn                                                                               & 66.2(0.0)    & 0.004      & 70.5(0.0)    & 0.004      & 17.3(2.2)  & 0.004   & 63.6(0.3)  & 0.004   \\
FedNova                                                                              & 63.9(0.6)    & 0.004      & 71.5(0.2)    & 0.004      & 20.2(0.8)  & 0.004   & 66.5(0.6)  & 0.004   \\
FedAvg                                                                               & 65.0(0.5)    & 0.004      & 69.6(0.5)    & 0.004      & 19.1(1.4)  & 0.004   & 66.3(0.0)  & 0.004   \\ \bottomrule
\end{tabular}
\end{small}
\label{tab:freq_ablation_cifar10_appendix}
\end{table}

\begin{figure}[h]
\centering
\includegraphics[width=1\linewidth]{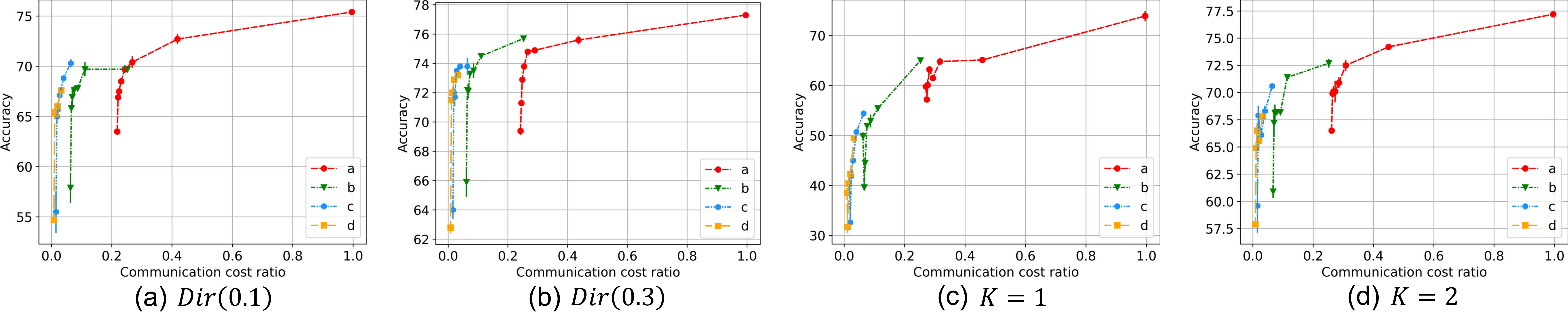}
\vspace{-0.8cm}
\caption{This figure illustrates the trend from Table~\ref{tab:freq_ablation_cifar10_appendix}, comparing the high-frequency group only, extended intervals, and DynamicSGD, using a CNN with CIFAR-10. For the interval 'a', we depict a, a-b, a-c, a-d, a-e, a-g, and a*. The same pattern is replicated for intervals b, c, and d.}
\label{fig:cirfar10_ablation_trend_appendix}
\end{figure}
\begin{figure}[htbp]
\centering
{\includegraphics[width=1\linewidth]{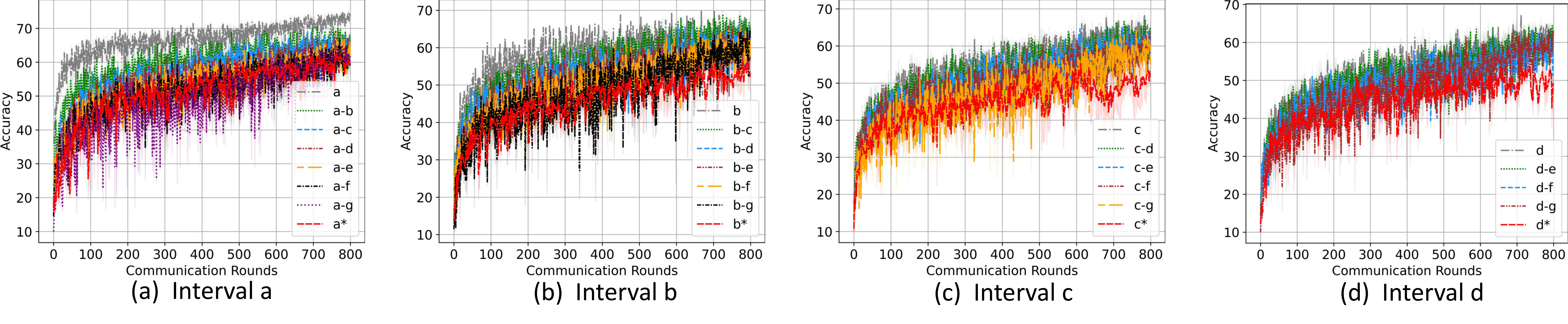}}
\vspace{-0.8cm}
\caption{Learning curves for all communication interval combinations in Table~\ref{tab:freq_ablation_cifar10_appendix}, $Dir(0.1)$ setting.}
\label{fig:freq_ablation_cifar10_d01}
\end{figure}
\begin{figure}[htbp]
\centering
{\includegraphics[width=1\linewidth]{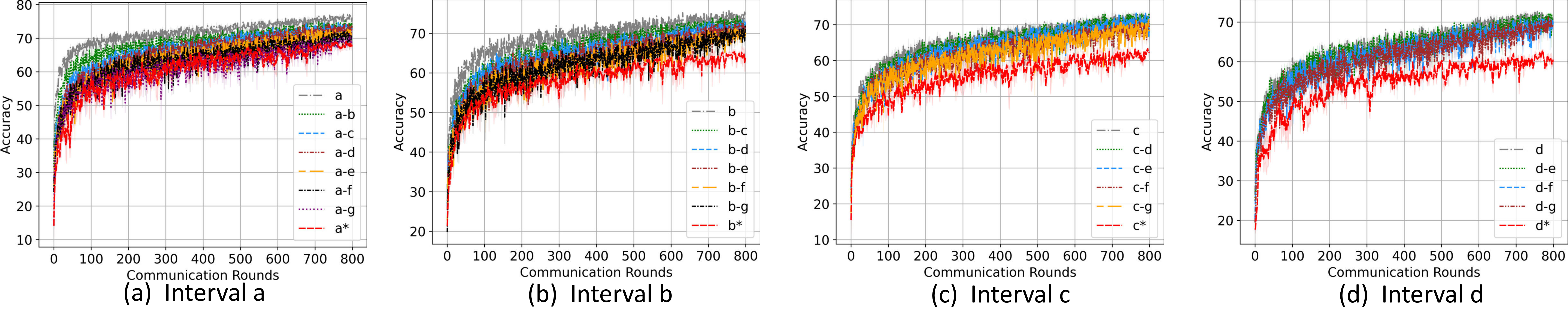}}
\vspace{-0.8cm}
\caption{Learning curves for all communication interval combinations in Table~\ref{tab:freq_ablation_cifar10_appendix}, $Dir(0.3)$ setting.}
\label{fig:freq_ablation_cifar10_d03}
\end{figure}
\begin{figure}[htbp]
\centering
{\includegraphics[width=1\linewidth]{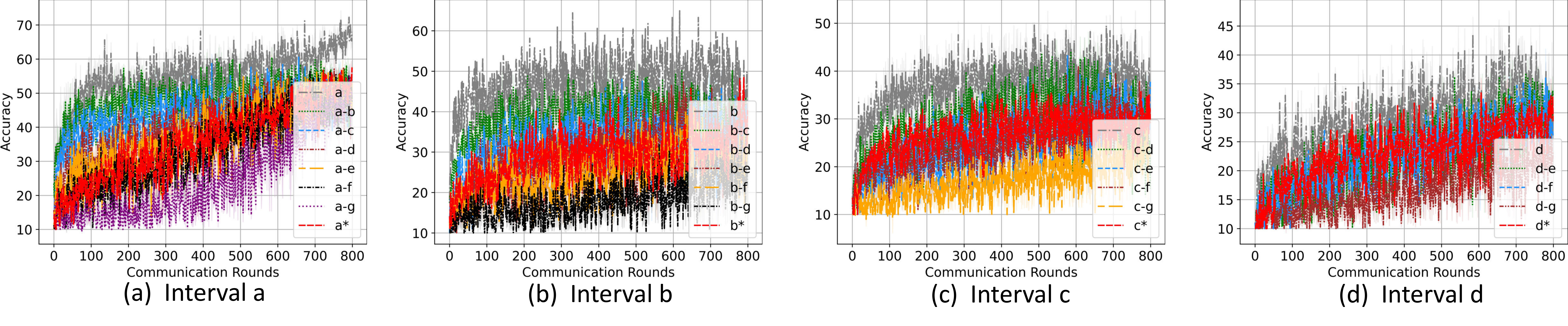}}
\vspace{-0.8cm}
\caption{Learning curves for all communication interval combinations in Table~\ref{tab:freq_ablation_cifar10_appendix}, $K=1$ setting.}
\label{fig:freq_ablation_cifar10_l1}
\end{figure}
\begin{figure}[htbp]
\centering
{\includegraphics[width=1\linewidth]{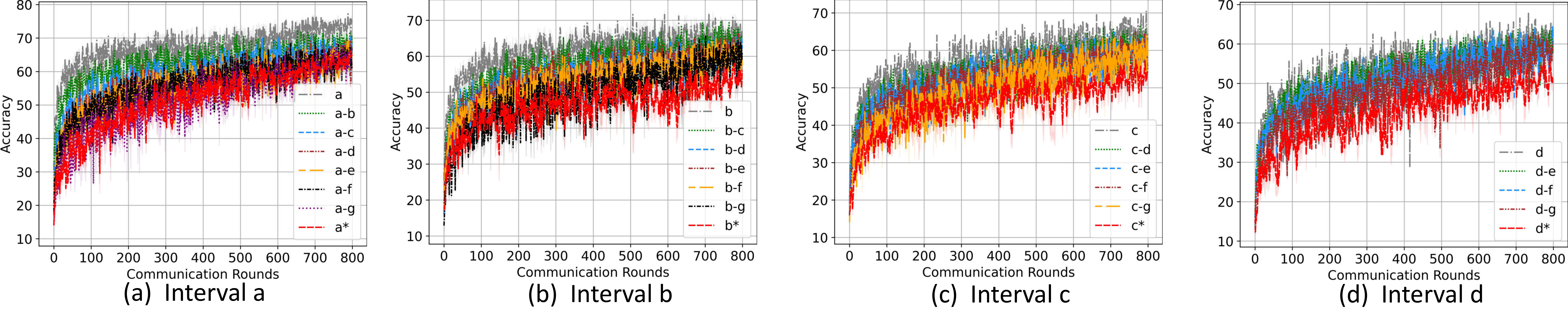}}
\vspace{-0.8cm}
\caption{Learning curves for all communication interval combinations in Table~\ref{tab:freq_ablation_cifar10_appendix}, $K=2$ setting.}
\label{fig:freq_ablation_cifar10_l2}
\end{figure}

\clearpage
\begin{table}[htbp]
\caption{Accuracy results comparing various communication interval combinations, and the standardized server total communication cost $\kappa_{g}^{t}$, using CNN under the CIFAR-100, \textit{Fix-0.3} configuration. 'a' signifies that all active clients train for 'a' intervals, while 'a*' indicates training exclusively involving high-frequency group clients. The same pattern is replicated for intervals b, c, and d.}
\centering
\begin{small}
\begin{tabular}{@{}ccccccccc@{}}
\toprule
\multirow{2}{*}{\begin{tabular}[c]{@{}c@{}}Communication   \\ interval\end{tabular}} & \multicolumn{2}{c}{$Dir(0.1)$} & \multicolumn{2}{c}{$Dir(0.3)$} & \multicolumn{2}{c}{$K=1$}    & \multicolumn{2}{c}{$K=2$}    \\ \cmidrule(l){2-9} 
                                                                                     & Accuracy   & $\kappa_{g}^{t}$  & Accuracy   & $\kappa_{g}^{t}$  & Accuracy  & $\kappa_{g}^{t}$ & Accuracy  & $\kappa_{g}^{t}$ \\ \midrule
a                                                                                    & 42.6  & 0.996             & 44.0  & 0.996             & 6.7  & 0.996            & 27.6 & 0.996            \\
a-b                                                                                  & 38.3  & 0.482             & 43.3  & 0.496             & 14.5 & 0.485            & 22.9 & 0.482            \\
a-c                                                                                  & 38.1  & 0.348             & 41.8  & 0.366             & 7.2  & 0.352            & 17.1 & 0.349            \\
a-d                                                                                  & 36.7  & 0.326             & 40.9  & 0.344             & 8.0  & 0.330            & 17.0 & 0.327            \\
a-e                                                                                  & 37.4  & 0.315             & 39.6  & 0.334             & 8.2  & 0.319            & 19.8 & 0.316            \\
a-f                                                                                  & 36.4  & 0.309             & 39.0  & 0.328             & 9.6  & 0.313            & 23.5 & 0.310            \\
a-g                                                                                  & 35.5  & 0.307             & 37.1  & 0.325             & 9.0  & 0.311            & 25.5 & 0.308            \\
a*                                                                                   & 29.8  & 0.304             & 31.5  & 0.323             & 13.5 & 0.308            & 20.3 & 0.305            \\ \midrule
b                                                                                    & 37.9  & 0.252             & 43.1  & 0.252             & 10.5 & 0.252            & 22.3 & 0.252            \\
b-c                                                                                  & 36.7  & 0.123             & 40.6  & 0.124             & 6.4  & 0.121            & 14.7 & 0.124            \\
b-d                                                                                  & 37.4  & 0.101             & 40.3  & 0.102             & 5.5  & 0.099            & 14.9 & 0.102            \\
b-e                                                                                  & 38.2  & 0.090             & 38.7  & 0.091             & 5.5  & 0.088            & 17.5 & 0.091            \\
b-f                                                                                  & 35.5  & 0.085             & 37.7  & 0.086             & 5.5  & 0.082            & 21.2 & 0.086            \\
b-g                                                                                  & 34.2  & 0.082             & 36.6  & 0.083             & 5.1  & 0.079            & 22.1 & 0.083            \\
b*                                                                                   & 29.3  & 0.079             & 30.1  & 0.080             & 9.8  & 0.076            & 18.2 & 0.080            \\ \midrule
c                                                                                    & 37.2  & 0.064             & 39.6  & 0.064             & 5.2  & 0.064            & 14.7 & 0.064            \\
c-d                                                                                  & 37.1  & 0.041             & 40.9  & 0.041             & 4.2  & 0.042            & 14.9 & 0.042            \\
c-e                                                                                  & 35.4  & 0.029             & 38.7  & 0.030             & 3.9  & 0.031            & 17.5 & 0.030            \\
c-f                                                                                  & 34.0  & 0.023             & 36.3  & 0.024             & 3.8  & 0.026            & 19.3 & 0.025            \\
c-g                                                                                  & 35.0  & 0.021             & 37.3  & 0.021             & 4.1  & 0.023            & 20.4 & 0.022            \\
c*                                                                                   & 34.1  & 0.018             & 36.1  & 0.018             & 4.1  & 0.020            & 21.8 & 0.019            \\ \midrule
d                                                                                    & 35.5  & 0.032             & 39.9  & 0.032             & 3.8  & 0.032            & 15.8 & 0.032            \\
d-e                                                                                  & 36.4  & 0.021             & 38.2  & 0.021             & 3.4  & 0.021            & 17.3 & 0.021            \\
d-f                                                                                  & 35.0  & 0.015             & 37.4  & 0.015             & 2.9  & 0.015            & 19.2 & 0.015            \\
d-g                                                                                  & 33.6  & 0.012             & 35.7  & 0.013             & 2.7  & 0.013            & 19.8 & 0.013            \\
d*                                                                                   & 34.2  & 0.009             & 33.3  & 0.010             & 2.6  & 0.010            & 21.1 & 0.010            \\ \midrule
FedProx                                                                              & 33.5(0.3)  & 0.004             & 34.7(0.5)  & 0.004             & 1.4(0.1)  & 0.004            & 18.1(0.3) & 0.004            \\
FedEnsemble                                                                          & 33.4(0.0)  & 0.004             & 33.0(0.3)  & 0.004             & 1.5(0.1)  & 0.004            & 18.0(0.3) & 0.004            \\
SCAFFOLD                                                                             & 35.8(0.5)  & 0.008             & 38.4(0.7)  & 0.008             & 1.3(0.1)  & 0.008            & 23.8(0.1) & 0.008            \\
FedGen                                                                               & 33.5(0.1)  & 0.011             & 34.6(0.8)  & 0.011             & 1.2(0.0)  & 0.011            & 17.3(0.3) & 0.011            \\
FedDyn                                                                               & 31.5(0.4)  & 0.004             & 32.0(0.0)  & 0.004             & 1.3(0.1)  & 0.004            & 16.4(0.0) & 0.004            \\
FedNova                                                                              & 34.0(0.1)  & 0.004             & 34.7(0.1)  & 0.004             & 1.3(0.1)  & 0.004            & 17.7(0.6) & 0.004            \\
FedAvg                                                                               & 33.0(0.2)  & 0.004             & 33.6(0.7)  & 0.004             & 1.3(0.1)  & 0.004            & 18.4(0.6) & 0.004            \\ \bottomrule
\end{tabular}
\end{small}
\label{tab:freq_ablation_cifar100_appendix}
\end{table}

\begin{figure}[h]
\centering
\includegraphics[width=1\linewidth]{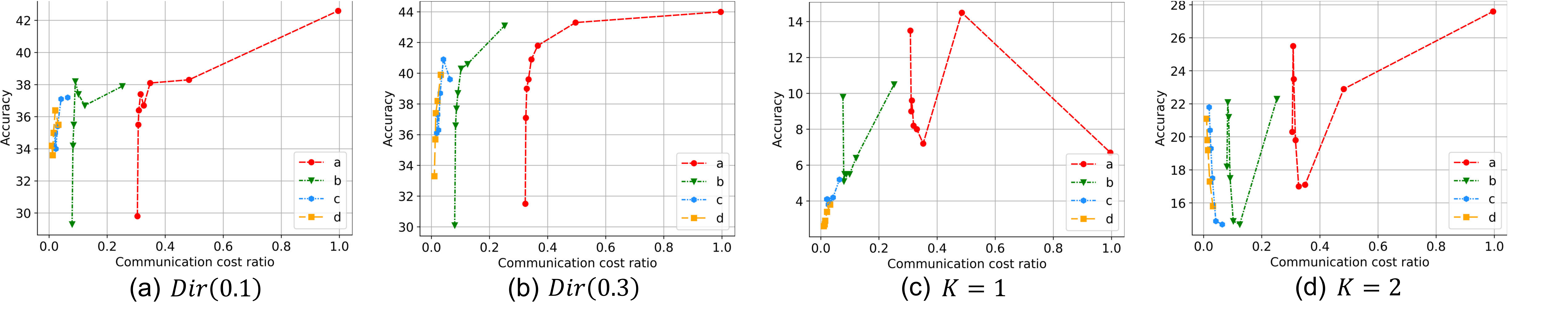}
\vspace{-0.85cm}
\caption{This figure illustrates the trend from Table~\ref{tab:freq_ablation_cifar100_appendix}, comparing the high-frequency group only, extended intervals, and DynamicSGD, using a CNN with CIFAR-100. For the interval 'a', we depict a, a-b, a-c, a-d, a-e, a-g, and a*. The same pattern is replicated for intervals b, c, and d.}
\label{fig:cirfar100_ablation_trend_appendix}
\end{figure}
\begin{figure}[htbp]
\centering
{\includegraphics[width=1\linewidth]{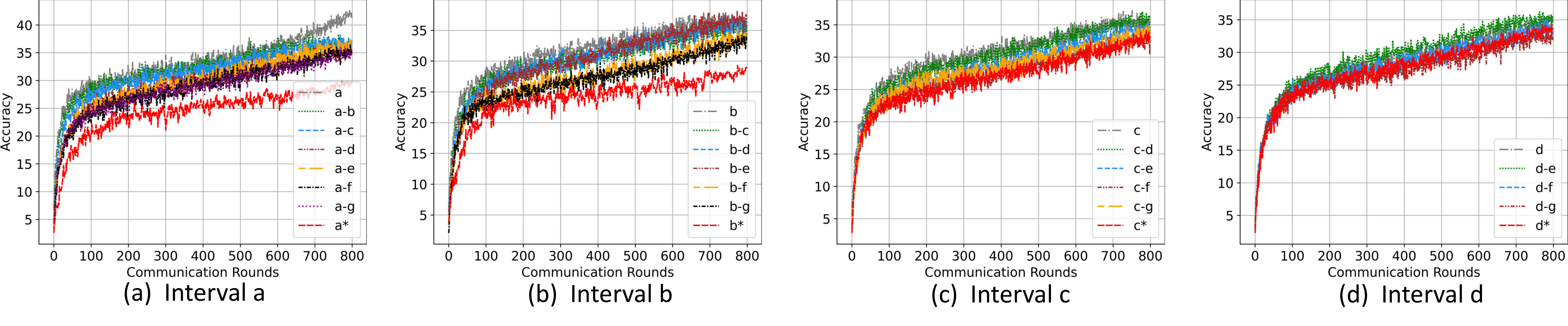}}
\vspace{-1cm}
\caption{Learning curves for all communication interval combinations in Table~\ref{tab:freq_ablation_cifar100_appendix}, $Dir(0.1)$ setting.}
\label{fig:freq_ablation_cifar100_d01}
\end{figure}
\begin{figure}[htbp]
\centering
{\includegraphics[width=1\linewidth]{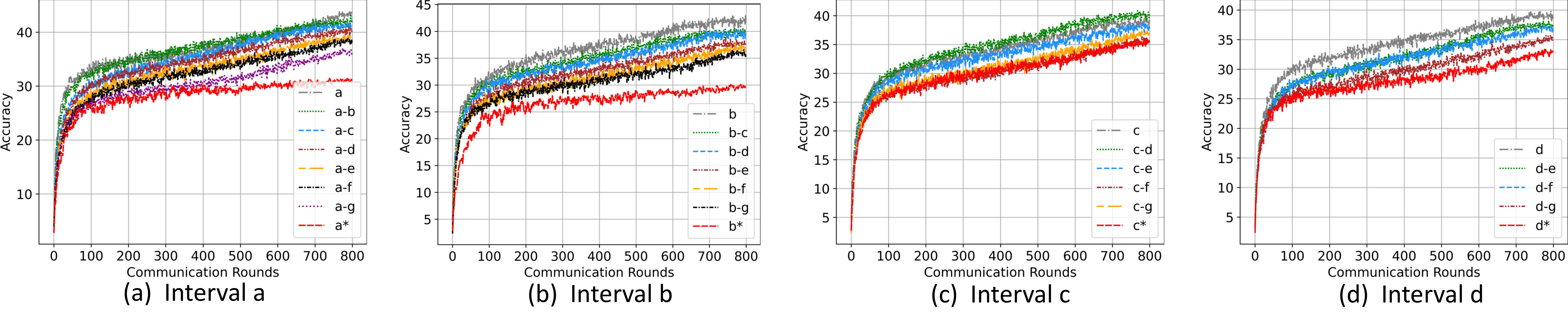}}
\vspace{-1cm}
\caption{Learning curves for all communication interval combinations in Table~\ref{tab:freq_ablation_cifar100_appendix}, $Dir(0.3)$ setting.}
\label{fig:freq_ablation_cifar100_d03}
\end{figure}
\begin{figure}[htbp]
\centering
{\includegraphics[width=1\linewidth]{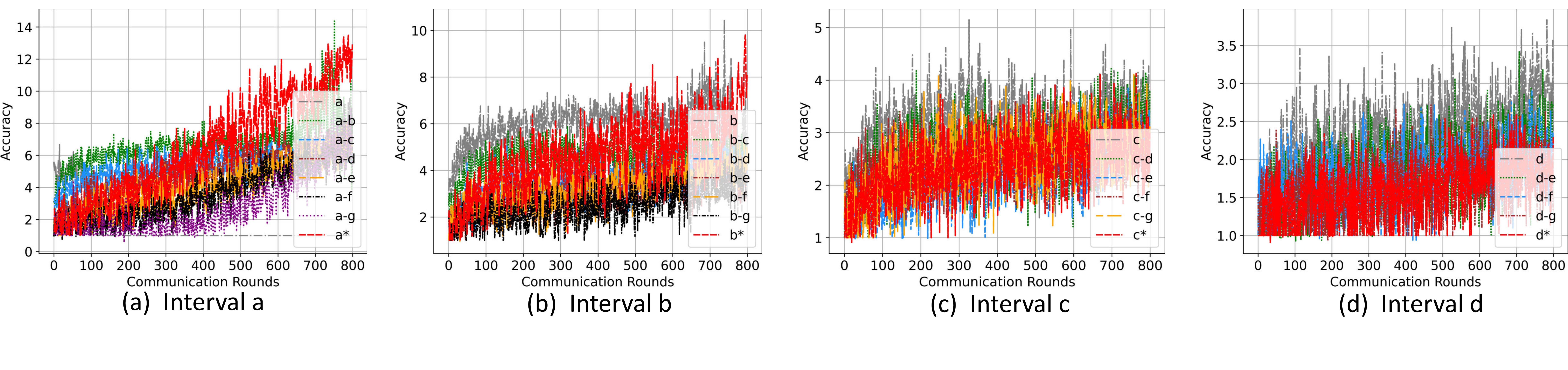}}
\vspace{-1cm}
\caption{Learning curves for all communication interval combinations in Table~\ref{tab:freq_ablation_cifar100_appendix}, $K=1$ setting.}
\label{fig:freq_ablation_cifar100_l1}
\end{figure}
\begin{figure}[htbp]
\centering
{\includegraphics[width=1\linewidth]{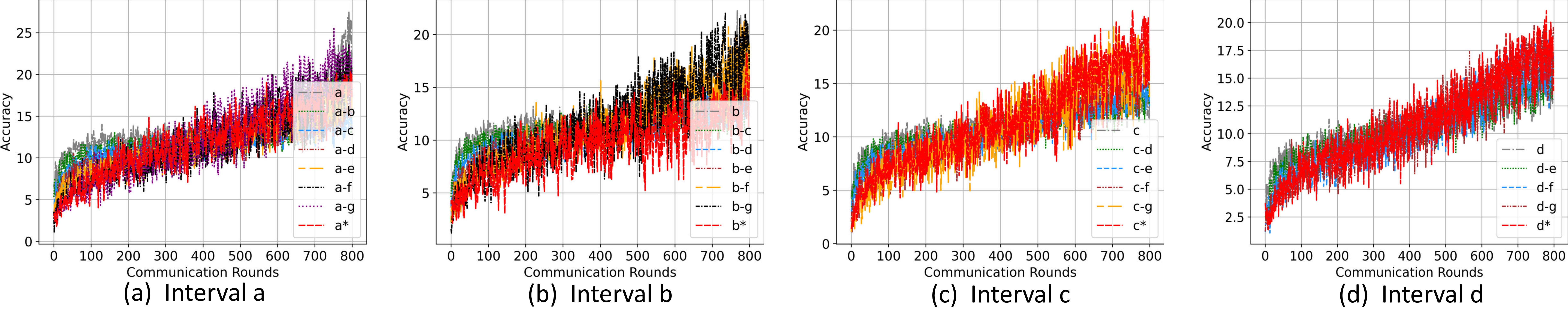}}
\vspace{-1cm}
\caption{Learning curves for all communication interval combinations in Table~\ref{tab:freq_ablation_cifar100_appendix}, $K=2$ setting.}
\label{fig:freq_ablation_cifar100_l2}
\end{figure}

\clearpage

\begin{table}[ht]
\caption{Accuracy results comparing various communication interval combinations, and the standardized server total communication cost $\kappa_{g}^{t}$, using CNN under the FEMNIST, \textit{Fix-0.3} configuration. 'a' signifies that all active clients train for 'a' intervals, while 'a*' indicates training exclusively involving high-frequency group clients. The same pattern is replicated for intervals b, c, and d.}
\centering
\begin{small}
\begin{tabular}{@{}ccccccc@{}}
\toprule
\multirow{2}{*}{Communication   interval} & \multicolumn{2}{c}{$Dir(0.01)$} & \multicolumn{2}{c}{$Dir(0.1)$} & \multicolumn{2}{c}{$Dir(0.3)$} \\ \cmidrule(l){2-7} 
                                          & Accuracy    & $\kappa_{g}^{t}$  & Accuracy   & $\kappa_{g}^{t}$  & Accuracy   & $\kappa_{g}^{t}$  \\ \midrule
a                                         & 77.9   & 1.000             & 85.6  & 1.000             & 87.6  & 1.000             \\
a-b                                       & 61.1   & 0.462             & 85.0  & 0.452             & 87.1  & 0.473             \\
a-c                                       & 62.6   & 0.324             & 84.7  & 0.312             & 86.8  & 0.339             \\
a-d                                       & 61.3   & 0.301             & 85.1  & 0.288             & 86.9  & 0.316             \\
a-e                                       & 61.5   & 0.289             & 85.2  & 0.277             & 87.0  & 0.305             \\
a-f                                       & 67.8   & 0.284             & 84.8  & 0.271             & 86.8  & 0.300             \\
a-g                                       & 62.7   & 0.281             & 84.7  & 0.268             & 86.7  & 0.297             \\
a*                                        & 68.7   & 0.278             & 84.3  & 0.265             & 86.2  & 0.294             \\ \midrule
b                                         & 58.8   & 0.250             & 85.1  & 0.250             & 87.1  & 0.250             \\
b-c                                       & 61.4   & 0.119             & 84.7  & 0.117             & 86.9  & 0.119             \\
b-d                                       & 60.6   & 0.096             & 84.9  & 0.094             & 86.9  & 0.097             \\
b-e                                       & 66.3   & 0.085             & 84.9  & 0.082             & 86.9  & 0.086             \\
b-f                                       & 61.1   & 0.080             & 84.9  & 0.077             & 86.8  & 0.080             \\
b-g                                       & 64.2   & 0.077             & 84.8  & 0.074             & 86.7  & 0.077             \\
b*                                        & 64.2   & 0.074             & 83.2  & 0.071             & 85.9  & 0.075             \\ \midrule
c                                         & 62.0   & 0.063             & 85.0  & 0.063             & 86.7  & 0.063             \\
c-d                                       & 61.0   & 0.040             & 85.1  & 0.040             & 86.9  & 0.040             \\
c-e                                       & 63.0   & 0.028             & 85.2  & 0.028             & 86.7  & 0.029             \\
c-f                                       & 62.9   & 0.022             & 84.7  & 0.023             & 86.8  & 0.023             \\
c-g                                       & 61.5   & 0.020             & 84.7  & 0.020             & 86.7  & 0.020             \\
c*                                        & 67.1   & 0.017             & 84.6  & 0.017             & 85.7  & 0.017             \\ \midrule
d                                         & 62.3   & 0.031             & 84.9  & 0.031             & 86.9  & 0.031             \\
d-e                                       & 62.6   & 0.020             & 84.6  & 0.020             & 86.9  & 0.021             \\
d-f                                       & 60.8   & 0.015             & 84.4  & 0.015             & 86.7  & 0.015             \\
d-g                                       & 58.2   & 0.012             & 84.2  & 0.012             & 86.7  & 0.012             \\
d*                                        & 61.5   & 0.009             & 83.4  & 0.009             & 85.6  & 0.010             \\ \midrule
FedProx                                   & 50.8   & 0.004             & 84.1(0.5)  & 0.004             & 85.1(0.1)  & 0.004             \\
FedEnsemble                               & 31.1   & 0.004             & 83.6(0.2)  & 0.004             & 83.8  & 0.004             \\
SCAFFOLD                                  & N/A         & 0.008             & N/A        & 0.008             & N/A        & 0.008             \\
FedGen                                    & N/A         & 0.01              & N/A        & 0.01              & N/A        & 0.01              \\
FedDyn                                    & N/A         & 0.004             & N/A        & 0.004             & N/A        & 0.004             \\
FedNova                                   & 45.5   & 0.004             & 83.9  & 0.004             & 85.7  & 0.004             \\
FedAvg                                    & 37.7(0.8)   & 0.004             & 83.8(0.3)  & 0.004             & 85.5(0.3)  & 0.004             \\ \bottomrule
\end{tabular}
\end{small}
\label{tab:freq_ablation_femnist_appendix}
\end{table}

\begin{figure}[h]
\centering
\includegraphics[width=1\linewidth]{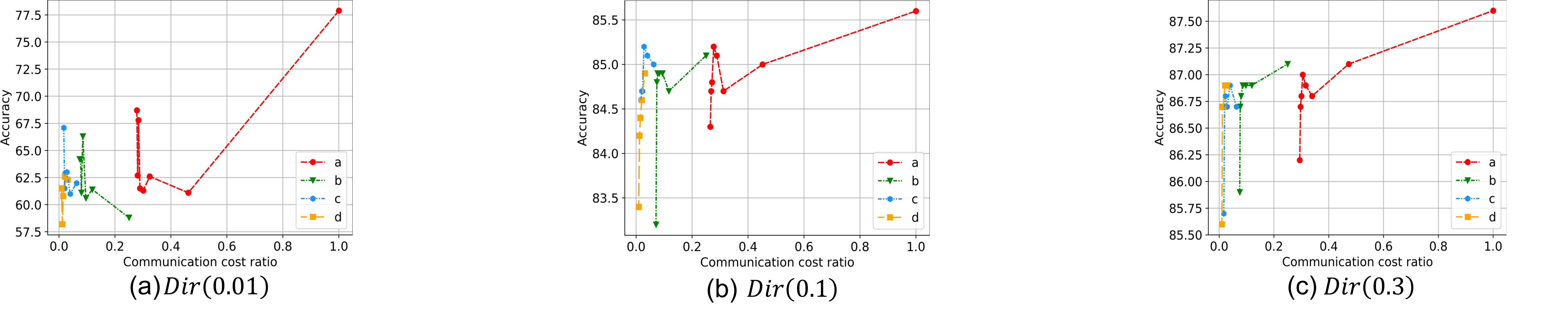}
\vspace{-0.6cm}
\caption{This figure illustrates the trend from Table~\ref{tab:freq_ablation_femnist_appendix}, comparing the high-frequency group only, extended intervals, and DynamicSGD, using a CNN with FEMNIST. For the interval 'a', we depict a, a-b, a-c, a-d, a-e, a-g, and a*. The same pattern is replicated for intervals b, c, and d.}
\label{fig:femnist_ablation_trend_appendix}
\end{figure}
\begin{figure}[htbp]
\centering
{\includegraphics[width=1\linewidth]{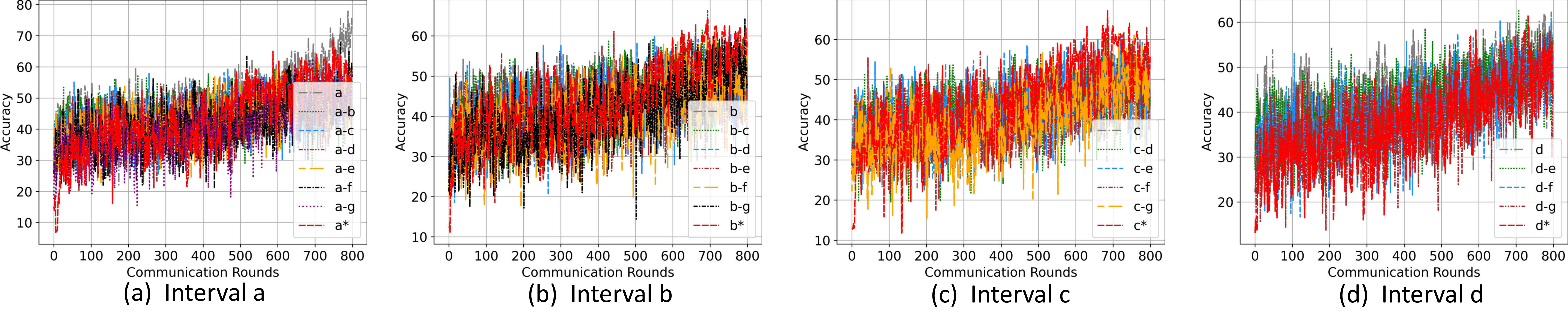}}
\vspace{-0.6cm}
\caption{Learning curves for all communication interval combinations in Table~\ref{tab:freq_ablation_femnist_appendix}, $Dir(0.01)$ setting.}
\label{fig:freq_ablation_femnist_d001}
\end{figure}
\begin{figure}[htbp]
\centering
{\includegraphics[width=1\linewidth]{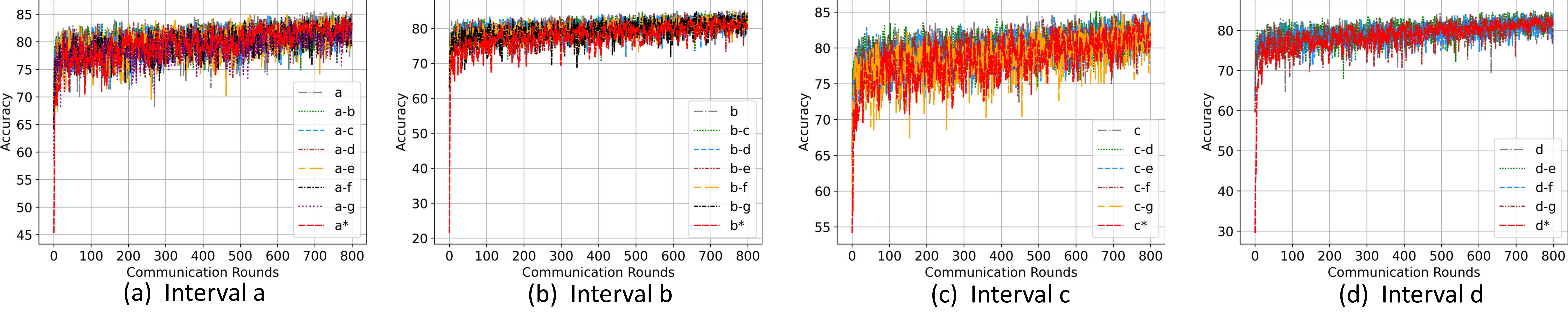}}
\vspace{-0.6cm}
\caption{Learning curves for all communication interval combinations in Table~\ref{tab:freq_ablation_femnist_appendix}, $Dir(0.1)$ setting.}
\label{fig:freq_ablation_femnist_d01}
\end{figure}
\begin{figure}[htbp]
\centering
{\includegraphics[width=1\linewidth]{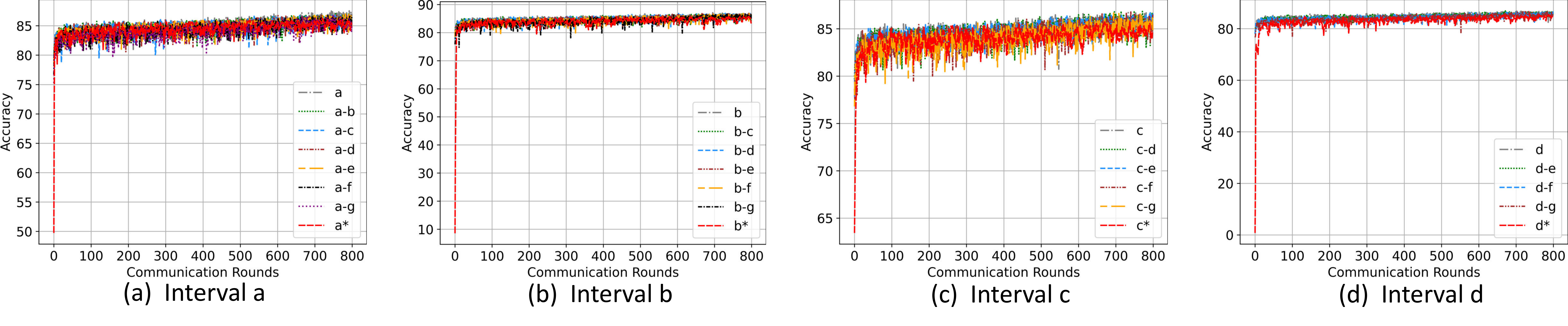}}
\vspace{-0.6cm}
\caption{Learning curves for all communication interval combinations in Table~\ref{tab:freq_ablation_femnist_appendix}, $Dir(0.3)$ setting.}
\label{fig:freq_ablation_femnist_d03}
\end{figure}

\clearpage

\section{Theoretical analysis}
\label{sec:theory}
In this section, we present the theoretical analysis of DynamicFL. Instead of pursuing a comprehensive analysis of the developed algorithm, we consider a simplified case to derive theoretical insights. 
Consider a federated learning problem that involves only three clients, denoted by $i=1,2,3$. Suppose client $i$ has $n_i$ data observations $z_{i,j}, j=1,\ldots,n_i$ and aims to learn their unknown mean based on a quadratic loss 
\begin{align}
    L_i (\theta) = \frac{1}{n_i} \sum_{j=1}^{n_i} \frac{1}{2} (\theta - z_{i,j})^2 . \label{eq1}
\end{align}
Suppose the global objective is to solve
\begin{align}
    \min_{\theta \in \R} L (\theta) = \sum_{i=1}^3 \lambda_i L_i (\theta)   \label{eq2}
\end{align}
where $\lambda_i = n_i/n$ and $n=\sum_{i=1}^3 n_i$.
We consider the following simplified version of the dynamicFL algorithm. 

\textbf{DynamicFL}: There are $r$ rounds. Within each round, 1) clients $1$ and $2$ will run a single local step and aggregate (namely, an FedSGD-type update) for $k$ times, 2) client $3$ will run $k$ local steps, and finally, 3) three clients' models are aggregated on the server-side (FedAvg-type update). 

We hope to develop the insight that even if the local minima of clients $1$ and $2$ are far from the global minimum, the performance (as measured by the convergence rate) of the dynamicFL algorithm is close to the following FedSGD baseline. 

\textbf{FedSGD}: all three clients perform a single local step and aggregate for $k \cdot r$ rounds. Note that the computational complexity as measured by the system-wide local steps is $K \cdot r$ for both the dynamicFL and FedSGD methods.

\begin{theorem}
    Assume the objective as defined in (\ref{eq2}). With the same learning rate $\eta \in (0,1)$ at each step for each client, the three-client DynamicFL and FedSGD algorithms will have the same rate of convergence at the same number of system-wide steps. More specifically, at $k \cdot r$ steps, where $k$ is the number of high-frequency local steps in DynamicFL and $r$ is the number of rounds, the updated model parameter satisfies
    \begin{align}
        \theta_r = \theta^* + (1- \eta)^{k \cdot r} (\theta_{0}  - \theta^*) \nonumber 
    \end{align}
    where $\theta^*$ is the global minimum to (\ref{eq2}) and $\theta_{0}$ is the initial model parameter.
\end{theorem}
\begin{proof}
    The local loss in (\ref{eq1}) can be rewritten as 
    \begin{align}
        L_i (\theta) 
        &= \frac{1}{n_i} \sum_{j=1}^{n_i} \frac{1}{2} (\theta - z_{i,j})^2 \nonumber \\
        &= \frac{1}{2} (\theta - \bar{z}_i)^2 + \frac{1}{n_i} \sum_{j=1}^{n_i} \frac{1}{2} (\bar{z}_i - z_{i,j})^2 ,  \label{eq3}
    \end{align}
    where $\bar{z}_i = n_i^{-1} \sum_{j=1}^{n_i} z_{i,j}$. 
    It can be calculated that the global minimum of the objective function $L$ in (\ref{eq1}) is 
    \begin{align}
        \theta^* = \sum_{i=1}^3 \lambda_i \bar{z}_i. \label{eq6}
    \end{align}

    For client $3$, suppose it runs $k $ local steps of gradient descent with learning rate $\eta > 0$, starting from the server-sent model $\theta_{i,0} = \theta_{r-1}$ at the beginning of round $r$, $r \geq 1$. The local updated parameter at step $k$, denoted by $\theta_{i,k}$, follows
    \begin{align}
        \theta_{i,k} 
        &= \theta_{i,k-1} - \eta \nabla_{\theta}  L_i (\theta_{i,k-1})   \nonumber \\
        &=  \theta_{i,k-1} - \eta (\theta_{i,k-1} - \bar{z}_i). \label{eq12}
    \end{align}
    Rearranging it, we have
    \begin{align}
        \theta_{i,k} - \bar{z}_i
        &= \theta_{i,k-1} - \eta \nabla_{\theta}  L_i (\theta_{i,k-1})   \nonumber \\
        &=  (1- \eta)  (\theta_{i,k-1} - \bar{z}_i) \nonumber \\
        &= \cdots = (1- \eta)^k (\theta_{i,0} - \bar{z}_i) \nonumber\\
        & = (1- \eta)^k (\theta_{r-1} - \bar{z}_i) \nonumber.
    \end{align}
    Thus, we have
    \begin{align}
       & \theta_{i,k} =  \bar{z}_i + (1- \eta)^k (\theta_{r-1} - \bar{z}_i) ,   \label{eq4} \\
       & \theta_{i,1} =  \bar{z}_i + (1- \eta) (\theta_{r-1}- \bar{z}_i) 
        \label{eq5}
    \end{align}

    Next, we bring the equations (\ref{eq4}) and (\ref{eq5}) into the dynamicFL algorithm. Within each round $r$, at each iteration $\ell=1,\ldots ,k$, clients $1$ and $2$ will run a local step and aggregation to obtain
    \begin{align}
        \theta_{1\&2, \ell} 
        &\de  \sum_{i=1}^2 \frac{n_i}{n_1+n_2} \biggl(\bar{z}_i + (1- \eta) (\theta_{1\&2, \ell-1} - \bar{z}_i) \biggr), \nonumber \\
        &\textrm{with }\theta_{1\&2, 0} \de \theta_{r-1}.\nonumber
    \end{align}
    Subtracting both sides by $\bar{z}_{1,2} \de \sum_{i=1}^2 \frac{n_i}{n_1+n_2} \bar{z}_i$, we obtain
    \begin{align}
        &\theta_{1\&2, \ell} - \bar{z}_{1,2}
        =  (1- \eta) (\theta_{1\&2, \ell-1} - \bar{z}_{1,2}) ,
    \end{align}
    which implies that 
    \begin{align}
        \theta_{1\&2, k} = \bar{z}_{1,2} + (1- \eta)^k (\theta_{r-1} - \bar{z}_{1,2}) \label{eq8}.
    \end{align}
    Meanwhile, the client 3 will obtain $\theta_{3,k}$ after $k$ steps of updates alone, as calculated in (\ref{eq4}).
    Combining (\ref{eq4}) and (\ref{eq8}), DynamicFL will obtain the aggregated model parameter at the end of round $r$ as
    \begin{align}
        \theta_r &= \frac{n_1+n_2}{n} \theta_{1\&2, k} + \frac{n_3}{n} \theta_{3,k} \nonumber \\
        &= \frac{n_1+n_2}{n} \bar{z}_{1,2} + \frac{n_1+n_2}{n} (1- \eta)^k (\theta_{r-1} - \bar{z}_{1,2}) \nonumber \\
        &\quad + \frac{n_3}{n} (\bar{z}_3 + (1- \eta)^k (\theta_{r-1} - \bar{z}_3)) \nonumber \\
        &= \theta^* + (1- \eta)^k (\theta_{r-1}  - \theta^*) \label{eq9}.
    \end{align}
    where the last equality is from (\ref{eq6}).
    By applying recursion to (\ref{eq9}), we obtain 
    \begin{align}
        \theta_r = \theta^* + (1- \eta)^{k \cdot r} (\theta_{0}  - \theta^*) \label{eq10}
    \end{align}
    where $\theta_{0}$ is the initial model from the server (at round $0$). 
    
    On the other hand, if we run FedSGD $k \cdot r$ rounds, where in each round, all three clients perform a single local step (with the same learning rate $\eta$) and and then aggregate, we will have
    \begin{align}
        \theta^{\textrm{FedSGD}}_{\ell} 
        &= \theta^{\textrm{FedSGD}}_{\ell-1} - \eta \nabla_{\theta}  L(\theta^{\textrm{FedSGD}}_{\ell-1})   \nonumber \\
        &= \theta^{\textrm{FedSGD}}_{\ell-1} - \eta (\theta^{\textrm{FedSGD}}_{\ell-1}- \theta^*),  \label{eq13}
    \end{align}
    for $\ell = 1, \ldots , k \cdot r$. Similarly to how Equation (\ref{eq12}) implies (\ref{eq4}), Equation (\ref{eq13}) implies that $\theta^{\textrm{FedSGD}}_{k \cdot r} = \theta^* + (1- \eta)^{k \cdot r} (\theta_{0}  - \theta^*)$, which is exactly (\ref{eq10}). 
\end{proof}

\clearpage
\end{document}